\documentclass{article} % For LaTeX2e
\usepackage{enumerate}
\usepackage{amsmath, amsthm}
\usepackage{amsfonts}
\usepackage{amssymb, stmaryrd}
\usepackage{bbm}
\usepackage{graphicx,caption, natbib}
\usepackage{subfig}
\usepackage{multirow}
\usepackage{color}
\usepackage{textcomp}
\usepackage{fullpage}

\newtheorem{theorem}{Theorem}
\newtheorem{lemma}{Lemma}

\newtheorem{proposition}{Proposition}
\newtheorem{corollary}{Corollary}
\newtheorem{definition}{Definition}

\def\fro{{\scriptscriptstyle \mathrm{F}}}
\def\triple{{|\!|\!|}}
\def\mb{{\mathbf m}}

\def\xb{{\mathbf x}}

\def\zb{{\mathbf z}}
\def\bb{{\mathbf b}}

\def\wb{{\mathbf w}}

\def\db{{\mathbf d}}
\def\sb{{\mathbf s}}
\def\ub{{\mathbf u}}
\def\vb{{\mathbf v}}

\def\alphab{{\boldsymbol\alpha}}

\def\varepsilonb{{\boldsymbol\varepsilon}}
\def\oneb{{\mathbf 1}}
\def\zerob{{\mathbf 0}}

\def\Ub{{\mathbf U}}
\def\Xb{{\mathbf X}}

\def\Hb{{\mathbf H}}
\def\Wb{{\mathbf W}}
\def\Vb{{\mathbf V}}
\def\Ib{{\mathbf I}}
\def\Ab{{\mathbf A}}
\def\Bb{{\mathbf B}}
\def\Cb{{\mathbf C}}

\def\Db{{\mathbf D}}
\def\Rb{{\mathbf R}}
\def\Sb{{\mathbf S}}
\def\Mb{{\mathbf M}}

\def\Pb{{\mathbf P}}
\def\Tb{{\mathbf T}}
\def\PJb{\Pb_{\J}}
\def\PJib{\Pb_{\Ji}}
\def\RJb{\Rb_{\J}}

\def\thetab{{\boldsymbol\theta}}
\def\Thetab{{\boldsymbol\Theta}}
\def\ThetaJb{\Thetab_{\J}}
\def\ThetaJib{\Thetab_{\Ji}}

\def\Real{{\mathbb{R}}}

\newcommand{\SET}[1]{\llbracket 1; #1\rrbracket}

\def\Bcal{\mathcal{B}}

\def\Wcal{\mathcal{W}}

\def\Ncal{\mathcal{N}}
\def\Kcal{\mathcal{K}}
\def\Lcal{\mathcal{L}}

\def\Dcal{\mathcal{D}}
\def\Ecal{\mathcal{E}}

\def\Zcal{\mathcal{Z}}

\def\Scal{\mathcal{S}}
\def\Qcal{\mathcal{Q}}
\def\diag{{\mathrm{diag}}}
\def\Diag{{\mathrm{Diag}}}

\def\support{{\mathrm{support}}}
\def\sign{{\mathrm{sign}}}
\def\trace{{\mathrm{Tr}}}
\def\sym{{\mathrm{sym}}}

\font\dsrom=dsrom10 scaled 1200
\newcommand{\indicator}[1]{\textrm{\dsrom{1}}_{#1}}
\def\Exp{{\mathbb{E}}}
\def\Pr{{\mathrm{Pr}}}

\def\defin{\triangleq}

\def\alphabo{{\boldsymbol\alpha}_0}
\def\dbo{{\mathbf d}_0}
\def\sbo{{\mathbf s}_0}
\def\Abo{{\mathbf A}_0}
\def\Dbo{{\mathbf D}_0}

\def\Ji{{\J^i}}
\def\J{{\mathrm{J}}}

\def\loweralpha{\underline{\alpha}}
\def\upperalpha{\overline{\alpha}}

\title{Local stability and robustness of sparse dictionary learning in the presence of noise}

\author{
Rodolphe Jenatton$^{*,\star}$ \texttt{jenatton@cmap.polytechnique.fr}
\\
R\'emi Gribonval$^\dag$ \texttt{remi.gribonval@inria.fr}
\\
Francis Bach$^\circ$ \texttt{francis.bach@inria.fr}
}

\date{}

\begin{document}

\maketitle

\begin{abstract}
A popular approach within the signal processing and machine learning communities consists in modelling
signals as sparse linear combinations of atoms selected from a \emph{learned} dictionary.
While this paradigm has led to numerous empirical successes in various fields ranging from image to audio processing, 
there have only been a few theoretical arguments supporting these evidences.
In particular, sparse coding, or sparse dictionary learning, relies on a non-convex procedure whose local minima have not been fully analyzed yet.
In this paper, we consider a probabilistic model of sparse signals, and show that, with high probability, 
sparse coding admits a local minimum around the reference dictionary generating the signals. 
Our study takes into account the case of over-complete dictionaries and noisy signals, thus extending previous work limited to noiseless settings and/or under-complete dictionaries.
The analysis we conduct is non-asymptotic and makes it possible to understand how the key quantities of the problem, 
such as the coherence or the level of noise, can scale with respect to the dimension of the signals, the number of atoms, the sparsity and the number of observations.
\end{abstract}

\newcommand{\blfootnote}[1]{{\let\thefootnote\relax\footnotetext{#1}}}

\blfootnote{$^*$CMAP, Ecole Polytechnique (UMR CNRS 7641), 91128 Palaiseau, France.}
\blfootnote{$^\dag$INRIA Rennes, Campus de Beaulieu, 35042 Rennes, France.}
\blfootnote{$^\circ$INRIA - SIERRA project, LIENS (INRIA/ENS/CNRS UMR 8548), 23, avenue d'Italie 75214 Paris, France.}
\blfootnote{$^\star$Most of the work was done while affiliated with$^\circ$.}
\section{Introduction}
Modelling signals as sparse linear combinations of atoms selected from a dictionary has become
a popular paradigm in many fields, including signal processing, statistics, and machine learning.
This line of research has witnessed the development of several well-founded
theoretical frameworks~(see, e.g.,~\cite{Wainwright2009, Zhang2009}) 
and efficient algorithmic tools~(see, e.g.,~\cite{Bach2011} and references therein).

However, the performance of such approaches hinges on the representation of the signals, which makes the question of designing ``good'' dictionaries prominent.
A great deal of effort has been dedicated to come up with efficient \emph{predefined} dictionaries, e.g., the various types of wavelets~\citep{Mallat:2008aa}.
These representations have notably contributed to 
many successful image processing applications such as compression, denoising and deblurring.
More recently, the idea of simultaneously \emph{learning} the dictionary and the sparse decompositions of the signals---also known as \emph{sparse dictionary learning}, or simply, \emph{sparse coding}---has emerged as a powerful framework, 
with state-of-the-art performance in many tasks, 
including inpainting and image classification~(see, e.g.,~\cite{Mairal2010} and references therein).

Although  sparse dictionary learning can sometimes be formulated as convex~\citep{Bach2008c, Bradley2009a}, 
non-parametric Bayesian~\citep{Zhou2009} and submodular~\citep{Krause2010} problems,
the most popular and widely used definition of sparse coding brings into play a non-convex optimization problem.
Despite its empirical and practical success, there has only been little theoretical analysis of the properties 
of sparse dictionary learning.
For instance, \citet{Maurer2010,Vainsencher2010,Mehta2012} derive generalization bounds which quantify how much the \emph{expected} signal-reconstruction error differs from the \emph{empirical} one,
computed from a random and finite-size sample of signals. 
In particular, the bounds obtained by \cite{Maurer2010,Vainsencher2010} are non-asymptotic and uniform with respect to the whole class of dictionaries considered (e.g., those with normalized atoms).
As discussed later, the questions raised in this paper explore a different and complementary direction.

Another theoretical aspect of interest consists in characterizing the local minima of the optimization problem associated to sparse coding, in spite of the non-convexity of its formulation.
This problem is closely related to the question of \emph{identifiability}, that is, 
whether it is possible to recover a reference dictionary that is assumed to generate the observed signals.
Identifying such a dictionary is important when the interpretation of the learned atoms matters, e.g., 
in source localization~\citep{ComonJutten2010} or in topic modelling~\citep{Jenatton2010b}.
The authors of \cite{Gribonval2010} pioneered research in this direction by considering noiseless sparse signals, possibly corrupted by some outliers, in the case where the reference dictionary forms a basis.
Still in a noiseless setting, and without outliers, \cite{Geng2011} extended the analysis to \emph{over-complete} dictionaries, i.e., these composed of more atoms than the dimension of the signals.
To the best of our knowledge, comparable analysis have not been carried out yet for noisy signals. In particular, the structure of the proofs of \cite{Gribonval2010,Geng2011}
hinges on the absence of noise and cannot be straightforwardly transposed to take into account some noise;
this point will be discussed subsequently.

In this paper, we therefore analyze the local minima of sparse coding {\em in the presence of noise} and make the following contributions:
\begin{list}{\labelitemi}{\leftmargin=1.1em}\addtolength{\itemsep}{-.215\baselineskip}
\item[--]Within a probabilistic model of sparse signals, we derive a \emph{non-asymptotic} lower bound of the probability of finding
a local minimum in a neighborhood of the reference dictionary.

\item[--] Our work makes it possible to better understand 
(a) how small the neighborhood around the reference dictionary can be,
(b) how many signals are required to hope for identifiability, 
(c) what the impact of the degree of over-completeness is, and
(d) what level of noise appears as manageable.

\item[--]We show that under deterministic coherence-based assumptions, such a local minimum is guaranteed to exist with high probability.
\end{list}
\section{Problem statement}

We introduce in this section the material required to define our problem and state our results.

\paragraph{Notation.}
For any integer $p$, we define the set $\SET{p} \defin \{1,\dots,p\}$.
For all vectors $\vb \in \Real^p$, we denote by $\sign(\vb) \in \{ -1,0,1 \}^p$ the vector
such that its $j$-th entry $[\sign(\vb)]_j$ is equal to zero if $\vb_j=0$, and to one (respectively, minus one) if $\vb_j > 0$ (respectively, $\vb_j < 0$).  
We extensively manipulate matrix norms in the sequel.
For any matrix $\Ab \in \Real^{n\times p}$, we define its Frobenius norm by 
$\|\Ab\|_\fro\defin [\sum_{i=1}^n\sum_{j=1}^p \Ab_{ij}^2 ]^{1/2}$; 
similarly, we denote the spectral norm of $\Ab$ by
$\triple \Ab \triple_2 \defin \max_{ \|\xb\|_2\leq1 } \|\Ab\xb\|_2$, and refer to the operator $\ell_\infty$-norm as
$\triple \Ab \triple_\infty \defin \max_{ \|\xb\|_\infty\leq 1 } \|\Ab\xb\|_\infty = \max_{ i\in\SET{n} } \sum_{j=1}^p |\Ab_{ij}|$.

For any square matrix $\Bb \in \Real^{n\times n}$, we denote by $\diag(\Bb) \in \Real^n$ the vector formed by extracting the diagonal terms of $\Bb$, and conversely, for any $\bb \in \Real^n$, 
we use $\Diag(\bb) \in \Real^{n\times n}$ to represent the (square) diagonal matrix whose diagonal elements are built from the vector $\bb$. For any $m \times p$ matrix $\Ab$ and index set $\J \subset \SET{p}$ we denote by $\Ab_{\J}$ the matrix obtained by concatenating the columns of $\Ab$ indexed by $\J$.
Finally, the sphere in $\Real^p$ is denoted $\Scal^p\defin\{\vb\in\Real^p;\ \|\vb\|_2=1\}$ and $\Scal^p_{+}\defin \Scal^{p} \cap \Real_{+}^p$.

\subsection{Background material on sparse coding}\label{sec:sparsecoding}

Let us consider a set of $n$ signals $\Xb \defin [\xb^1,\dots,\xb^n]\! \in\! \Real^{m\times n}$ of dimension $m$,
along with a dictionary 
$\Db \defin [\db^1,\dots,\db^p]\! \in \Real^{m\times p}$ formed of $p$ atoms---also known as dictionary elements.
Sparse coding simultaneously learns $\Db$ and
a set of $n$ sparse $p$-dimensional vectors $\Ab \defin[\alphab^1,\dots,\alphab^n]\! \in\! \Real^{p\times n}$, such that each signal $\xb^i$ can be well approximated by 
$
\xb^i \approx \Db \alphab^i
$
for $i$ in $\SET{n}$.
By sparse, we mean that the vector $\alphab^i$ has $k \ll p$ non-zero coefficients, 
so that we aim at reconstructing $\xb^i$ from only a few atoms.
Before introducing the sparse coding formulation~\citep{Mairal2010, Olshausen1997}, 
we need some definitions:
\begin{definition} For any dictionary $\Db \in \Real^{m\times p}$ 
and signal $\xb \in \Real^m$, we define 
\begin{eqnarray}
\label{eq:Li}
\Lcal_{\xb}(\Db,\alphab) 
&\defin& \frac{1}{2} \|\xb-\Db\alphab\|_2^2 + \lambda \|\alphab\|_1\notag\\
\label{eq:fi}
f_\xb(\Db) 
&\defin& 
\min_{\alphab\in\Real^p} \Lcal_{\xb}(\Db,\alphab). 
\end{eqnarray}
Similarly for any set of $n$ signals $\Xb \defin [\xb^1,\dots,\xb^n] \in \Real^{m \times n}$, we introduce 
\begin{equation*}\label{eq:f}
F_n(\Db) \defin \frac{1}{n}\sum_{i=1}^n f_{\xb^i}(\Db). 
\end{equation*}
\end{definition}
Based on problem~(\ref{eq:fi}), refered to as 
Lasso in statistics~\citep{Tibshirani1996}, and 
basis pursuit in signal processing~\citep{Chen1998},
the standard approach to perform sparse coding~\citep{Olshausen1997,Mairal2010} solves the minimization problem 
\begin{equation}\label{eq:min_fn}
\min_{\Db \in \mathcal{D}} F_n(\Db),
\end{equation}
where the regularization parameter $\lambda$ in~(\ref{eq:fi}) controls the level of sparsity, 
while $\mathcal{D} \subseteq \Real^{m\times p}$ is a compact set; 
in this paper, $\mathcal{D}$ denotes the set of dictionaries with unit $\ell_2$-norm atoms, which is a natural choice in image processing~\citep{Mairal2010,Gribonval2010}.
Note however that other choices for the set $\mathcal{D}$ may also be relevant depending on the application 
at hand~(see, e.g.,~\cite{Jenatton2010b} where in the context of topic models, the atoms in $\mathcal{D}$ belong to the unit simplex).

\subsection{Main objectives}\label{sec:main_obj}
The goal of the paper is to characterize some local minima of the function $F_n$ under a generative model for the signals $\xb^i$.
Throughout the paper, we assume the observed signals are generated \emph{independently} according to a specified probabilistic model. 
The considered signals are typically drawn  as $\xb^{i} \defin \Dbo \alphabo^{i} + \varepsilonb^{i}$ where $\Dbo$ is a fixed reference dictionary, 
$\alphabo^{i}$ is a sparse coefficient vector, and $\varepsilonb^{i}$ is a noise term. 
The specifics of the underlying probabilistic model are given in Sec.~\ref{sec:gen_model}. 
Under this model, we can state more precisely our objective: we want to show that 
$$
\Pr\big(F_n\ \text{has a local minimum in a ``neighborhood'' of}\ \Dbo\big) \approx 1.
$$
We loosely refer to a certain ``neighborhood'' since 
in our regularized formulation, a local minimum cannot appear exactly at $\Dbo$.
The proper meaning of this neighborhood is the subject of Sec.~\ref{sec:manifold}.

\paragraph{Intrinsic ambiguities of sparse coding.} Importantly, we have so far referred to $\Dbo$ as \emph{the} reference dictionary generating the signals. 
However, and as already discussed in~\cite{Gribonval2010,Geng2011} and more generally the related literature on blind source separation 
and independent component analysis~\citep{ComonJutten2010}, 
it is known that the objective of~(\ref{eq:min_fn}) is invariant by sign flips and atoms permutations.
As a result, while solving~(\ref{eq:min_fn}),
we cannot hope to identify the specific $\Dbo$.
We focus instead on the local identifiability of the whole \emph{equivalence class} defined by the transformations described above.
From now on, we simply refer to $\Dbo$ to denote one element of this equivalence class. 
Also, since these transformations are \emph{discrete}, 
our local analysis is not affected by invariance issues, as soon as we are sufficiently close to some representant of $\Dbo$. 
\subsection{Local minima on the oblique manifold}\label{sec:manifold}
The minimization of $F_n$ is carried out over $\mathcal{D}$, which is the set of dictionaries with unit $\ell_2$-norm atoms.
This set turns out to be a manifold, known as the \emph{oblique manifold}~\citep{Absil2008}.
Since $\Dbo$ is assumed to belong to $\mathcal{D}$, it is therefore natural to consider the behavior of $F_n$ according to the 
geometry and topology of $\mathcal{D}$. To this end, we consider a specific (local) parametrization of $\mathcal{D}$.

\paragraph{Parametrization of the oblique manifold.}
Specifically, let us consider the set of matrices
$$
\mathcal{W}_{\Dbo} \defin \big\{ \Wb \in \Real^{m\times p};\ \diag(\Wb^\top\Dbo)=\zerob\ \mathrm{and}\ \diag(\Wb^\top\Wb)=\oneb  \big\}.
$$
In words, a matrix $\Wb \in \mathcal{W}_{\Dbo}$ has unit norm columns $\|\wb^{j}\|_{2}=1$ that are orthogonal to the corresponding columns of $\Dbo$: 
$[\wb^{j}]^{\top}\db^{j} = 0$, for any $j \in \SET{p}$.
Now, for any matrix $\Wb \in \mathcal{W}_{\Dbo}$, for any unit norm \emph{velocity} vector 
$\vb \in \Scal^p$,
and for all $t\in\Real$,
we introduce the parameterized~dictionary:
\begin{equation}\label{eq:Dt}
\Db(\Dbo,\Wb,\vb,t)\defin\Dbo\Diag[ \cos(\vb t) ]+\Wb\Diag[ \sin(\vb t) ],
\end{equation}
where $\Diag[ \cos(\vb t) ]$ and   $\Diag[ \sin(\vb t) ]\in \Real^{p \times p}$ 
stand for the diagonal matrices with diagonal terms equal to $\{\cos(\vb_{\! j} t)\}_{j\in\SET{p}}$ and $\{ {\sin(\vb_{\! j} t)} \}_{j\in\SET{p}} $ respectively. 
By construction, we have $\Db(\Dbo,\Wb,\vb,t) \in \mathcal{D}$ for all $t\in\Real$ and $\Db(\Dbo,\Wb,\vb,0)=\Dbo$. 
To ease notation, we will denote $\Db(\Wb,\vb,t)$, leaving the dependence on the reference dictionary $\Dbo$ implicit. Also, when it will be made clear from the context, we will drop the dependence on $\Wb,\vb$ in $\Db$. 
Note that the set of matrices given by $\Wb\Diag(\vb)$
corresponds to the tangent space of $\Dcal$ at $\Dbo$, intersected with the set of matrices in $\Real^{m\times p}$ with unit Frobenius norm 
(since we have $\|\Wb\Diag(\vb)\|_\fro=1$). 
\paragraph{Characterization of local minima on the oblique manifold.}
We can exploit the above parametrization of the manifold $\Dcal$ to characterize the existence of a local minimum as follows:
\begin{proposition}[Local minimum characterization]\label{prop:localmin}
Let $t > 0$ be some fixed scalar and define
\begin{equation}
\label{eq:DefDeltaFn}
\Delta F_{n}(\Wb,\vb,t) \defin F_{n}(\Db(\Wb,\vb,t)) - F_{n}(\Db_{0}).
\end{equation}
If we have
$$
\inf_{\Wb \in \Wcal_{\Dbo},\ \vb \in \Scal_{+}^{p}} \Delta F_{n}(\Wb,\vb,t) > 0,
$$
then $F_n: \Dcal \rightarrow \Real_+$ admits a local minimum in 
$
\big\{ \Db \in \Dcal;\ \|\Dbo-\Db\|_\fro < t  \big\}.
$
\end{proposition}
The detailed proof of this result is given in Sec.~A of the appendix. 
It relies on the continuity of $F_n$ and the fact that the curves $\Db(\Wb,\vb,t)$
define a surjective mapping onto $\Dcal$ (see Lemma~1 in the appendix).
We next describe some other ingredients required to state our results.
\subsection{Closed-form expression for $F_{n}$?}
Although the function $F_n$ is Lipschitz-continuous~\citep{Mairal2010}, 
its minimization is challenging since it is non-convex and subject to the non-linear constraints of $\mathcal{D}$.
Moreover, $F_n$ is defined through the minimization over the vectors $\Ab$, 
which, at first sight, does not lead to a simple and convenient expression.
However, it is known that $F_n$ has a simple closed-form in some favorable scenarios. 
\paragraph{Closed-form expression for $f_{\xb}$.} We leverage here a key property of the function $f_\xb$.
Denote by $\hat{\alphab}\in\Real^p$ a solution of problem~(\ref{eq:fi}), that is, the minimization defining $f_\xb$. 
By the convexity of the problem, there always exists such a solution such that, denoting $\J\defin\{ j\in\SET{p};\, \hat{\alphab}_j\neq 0\}$ its support, 
the dictionary $\Db_\J \in \Real^{m\times |\J|}$ restricted to the atoms indexed by $\J$ has linearly independent columns 
(hence $\Db_\J^\top\Db_\J$ is invertible). Denoting $ \hat{\sb} \in \{-1,0,1\}^p$ the sign of $\hat{\alphab}$ and $\J$ its support, $\hat{\alphab}$ has a closed-form expression 
in terms of $\Db_{\J}$, $\xb$ and $\hat{\sb}$~(see, e.g.,~\cite{Wainwright2009,Fuchs2005}). This property is appealing in that it makes it possible to obtain a closed-form expression for $f_\xb$  (and hence, $F_n$),  {\em provided that we can control the sign patterns of $\hat{\alphab}$}.
In light of this remark, it is natural to define:
\begin{definition}\label{def:phi}
Let $\sb \in \{-1,0,1\}^p$ be an arbitrary sign vector and $\J$ be its support. For $\xb \in \Real^m$ and $\Db \in \Real^{m\times p}$,
we define 
\begin{equation*}
\label{eq:defphi}
\phi_\xb(\Db|\sb) \defin   \inf_{\alphab\in\Real^p, \ \support(\alphab)=\J }\frac{1}{2}\|\xb - \Db\alphab\|_2^2+\lambda\sb^\top\alphab.
\end{equation*}
Whenever $\Db_\J^\top \Db_\J$ is invertible,  the minimum is achieved at $\tilde{\alphab} = \tilde{\alphab}(\Db,\xb,\sb)$ defined by
\begin{equation*}
\label{eq:ClosedFormMinimizer}
\tilde{\alphab}_\J = \big[ \Db_\J^\top\Db_\J \big]^{-1} \big[ \Db_\J^\top \xb -\lambda \sb_\J \big] \in \Real^{|\J|}\quad
\text{and}\quad 
\tilde{\alphab}_{\J^c} = \zerob,
\end{equation*}
and we have
\begin{equation}\label{eq:phi}
\phi_\xb(\Db|\sb) =  
\frac{1}{2} \big[ \|\xb\|_2^2 - (\Db_\J^\top \xb - \lambda \sb_\J)^\top 
(\Db_\J^\top \Db_\J)^{-1} 
(\Db_\J^\top \xb -\lambda \sb_\J ) \big].
\end{equation}
Moreover, if $\sign(\tilde{\alphab}) = \sb$, then
\begin{equation*}
\phi_\xb(\Db|\sb) =   \min_{\alphab\in\Real^p, \ \sign(\alphab)=\sb }\frac{1}{2}\|\xb - \Db\alphab\|_2^2+\lambda\sb^\top\alphab 
=  \min_{\alphab\in\Real^p,\ \sign(\alphab)=\sb } \Lcal_{\xb}(\Db,\alphab) = \Lcal_{\xb}(\Db,\tilde{\alphab}).
\end{equation*}
We define $\Phi_n(\Db|\Sb)$ analogously to $F_n(\Db)$, for a sign matrix $\Sb \in \{-1,0,1\}^{p\times n}$.
\end{definition}
Hence, with $\hat{\sb}$ the sign of the (unknown) minimizer $\hat{\alphab}$, we have
$f_\xb(\Db) = \Lcal_{\xb}(\Db,\hat{\alphab}) = \phi_\xb(\Db|\hat{\sb})$.

Showing that the function $F_n$ is accurately approximated by $\Phi_n(\cdot|\Sb)$ for a controlled $\Sb$ will be a key ingredient of our approach. This will exploit sign recovery properties of $\ell_1$-regularized least-squares problems, 
a topic which is already well-understood~(see, e.g.,~\cite{Wainwright2009,Fuchs2005} and references therein).
\subsection{Coherence assumption on the reference dictionary $\Dbo$}\label{sec:coherence}
We consider a standard sufficient support recovery condition referred to as the \emph{exact recovery condition} 
in signal processing~\citep{Fuchs2005,Tropp2004} or the \emph{irrepresentability condition} (IC) in the machine learning and statistics communities~\citep{Wainwright2009, Zhao2006}.
It is a key element to control the supports of the solutions of $\ell_1$-regularized least-squares problems.
To keep our analysis reasonably simple, we will impose the irrepresentability condition \emph{via} a condition on the \emph{mutual coherence} of the reference dictionary $\Dbo$, which is a stronger requirement~\cite{Van2009}.  
This quantity is defined 
(see, e.g.,~\cite{Fuchs2005,Donoho2001}) as 
$$
\mu_0 \defin \max_{i,j\in\SET{p}, i\neq j}| [\dbo^i]^\top [\dbo^j]| \in [0,1].
$$
The term $\mu_0$ gives a measure of the level of correlation between columns of $\Dbo$. It is for instance equal to zero in the case of an orthogonal dictionary, and to one if $\Dbo$ contains two colinear columns. 
Similarly, we introduce $\mu(\Wb,\vb,t)$ for the dictionary $\Db(\Wb,\vb,t)$ defined in~(\ref{eq:Dt}). 
For any $\Wb,\vb$, $t\geq 0$, we have the simple inequality:
\begin{equation}\label{eq:inequality_mu0}
\mu(\Wb,\vb,t)\defin \max_{i,j\in\SET{p}, i\neq j}|   [\db^i(\Wb,\vb,t)]^\top [\db^j(\Wb,\vb,t)]|  \leq \mu(t) \defin \mu_0 + 3t.
\end{equation}
In particular, we have $\mu(\Wb,\vb,0)=\mu_0$. 
For the theoretical analysis we conduct, 
we consider a deterministic coherence-based assumption, as considered for instance in the previous work on dictionary learning by~\cite{Geng2011}, such that the coherence $\mu_0$ and the level of sparsity $k$ of the coefficient vectors $\alphab^{i}$ 
should be inversely proportional, i.e., $k \mu_0 = O(1)$.
In  light of~(\ref{eq:inequality_mu0}), such an upper bound on $\mu_0$ will loosely transfer to $\mu(t)$ provided that $t$ is small enough. 
In fact, and as further developed in the appendix, most of the elements of our proofs work based on a restricted isometry property (RIP), 
which is known to be weaker than the coherence assumption~\citep{Van2009}. 
However, since we still face a problem related to~IC when using RIP, we keep the coherence in our analysis.
Unifying our proofs under a RIP criterion is the object of future work.

\subsection{Probabilistic model of sparse signals}\label{sec:gen_model}
Equipped with the main concepts, we now present our signal model.
Given a \emph{fixed} reference dictionary $\Dbo \in \mathcal{D}$, 
each noisy sparse signal $\xb \in \Real^m$ is built \emph{independently} from the following steps:

\noindent (1) {\bf Support generation:} Draw uniformly without replacement $k$ atoms out of the $p$ available in $\Dbo$.
This procedure thus defines a support $\J\defin\{j\in\SET{p};\ \delta(j)=1\}$ whose size is $|\J|=k$, 
and where $\delta(j)$ denotes the indicator function equal to one if the $j$-th atom is selected, zero otherwise,
so~that 
$$\textstyle
\Exp[\delta(j)]=\frac{k}{p},\ \text{and for}\ i\neq j,\ \text{we further have}\ 
\Exp[\delta(j)\delta(i)]=\frac{k(k-1)}{p(p-1)}.
$$
Our result holds for any support generation scheme yielding  the above expectations.

\noindent (2) {\bf Coefficient generation:} Define a sparse vector $\alphabo \in \Real^p$ supported on $\J$ whose entries in $\J$ are generated i.i.d.~according to a {\em sub-Gaussian distribution}: for $j$ not in $\J$, $[\alphabo]_j$ is set to zero; on the other hand, we assume there exists some $c > 0$ such that for $j \in \J$ we have, for all $t \in \Real$, $\Exp\{ \exp(t[\alphabo]_{j}) \} \leq \exp (c^{2}t^{2}/2)$ . 
We denote $\sigma_{\alpha}$ the smallest value of $c$ such that this property holds. For background about sub-Gaussian random variables, 
%see~\cite{Buldygin2000,Vershynin2010}. 
see, e.g.,~\cite{Buldygin2000}.
For simplicity of the analysis we restrict to the case where the distribution also has all its mass {\em bounded away from zero}.  
Formally, there exist $\loweralpha>0$ such that 
$
\Pr(|[\alphabo]_j| <  \loweralpha\ |\ j \in \J) = 0.
$ 

\noindent (3) {\bf Noise:} Eventually generate the signal $\xb=\Dbo \alphabo + \varepsilonb$, 
where the entries of the additive noise $\varepsilonb\in \Real^m$ are assumed i.i.d.~sub-Gaussian with parameter $\sigma$.

\section{Main results}\label{sec:main_results}
This section describes the main results of this paper which 
show that under appropriate scalings of the dimensions $(m,p)$, number of samples $n$, 
and model parameters $k,\loweralpha,\sigma_{\alpha},\sigma,\mu_0$, it is possible to prove that, with high probability,
the problem~(\ref{eq:min_fn}) admits a local minimum in a neighborhood of $\Dbo$ of controlled size, 
for appropriate choices of the regularization parameter $\lambda$.
The detailed proofs of the following results may be found in the appendix, 
but we provide their main outlines in Sec.~\ref{sec:building_blocks}.
\def\gammaDbo{\gamma_{\scriptscriptstyle \Dbo}}
\begin{theorem}[Local minimum of sparse coding]\label{thm:main_thm}
Let us consider our generative model of signals for some reference dictionary $\Dbo \in \Real^{m \times p}$ with coherence $\mu_0$,
and define $1/\gammaDbo \defin \triple \Dbo \triple_2 \cdot k \mu_0$, where $\triple \Dbo \triple_2$ refers to the spectral norm of $\Dbo$.
If the following conditions hold:
\begin{itemize}
 \item[]\textbf{(Coherence)} \hspace*{1.1cm}
$
\quad \Omega\big(\sqrt{\log(p)}\big) = \gammaDbo = O\big(\sqrt{\log(n)}\big),
$
 \item[]\textbf{(Sample complexity)}
$\displaystyle
\quad \frac{\log(n)}{n} =  O\Big(
\frac{\mu_0^2}{m\cdot p^3 \cdot \gammaDbo^2} \Big),
$
\end{itemize}
then, with probability exceeding $1 - [\frac{mpn}{9}]^{-\frac{mp}{2}} - e^{-4\sqrt{n}}$, 
problem~(\ref{eq:min_fn}) admits a local minimum in
$$\displaystyle
\bigg\{ \Db \in \Dcal;\ \|\Dbo-\Db\|_\fro = 
O\Big( 
\max\Big\{
p \cdot \gammaDbo \cdot \Big[  e^{-\frac{\gammaDbo^2}{2}} +  \sqrt{ m p \log(n)/n} \Big],\  
\frac{\sigma}{\sigma_\alpha}\cdot \sqrt{m}\Big\}
 \Big)\
\bigg\}.
$$
\end{theorem}
First, it is worth noting that this theorem is presented on purpose in a simplified form, 
in order to highlight its message. In particular, all quantities related to the distribution of $\alphabo$ (e.g., $\sigma_\alpha$) 
are assumed to be $O(1)$ and are therefore kept ``hidden'' in the big-O notation. 
A detailed statement of this theorem is however available in the appendix (see Theorem~\ref{th:minimum_sparsecoding}).

In words, the main message of Theorem~\ref{thm:main_thm} is that provided (a) the reference dictionary is incoherent enough, 
and (b) we observe enough signals, 
we can guarantee the existence of a local minimum for problem~(\ref{eq:min_fn}) in a ball centered at $\Dbo$.
We can see that the radius of this ball decomposes according to three different contributions:
(1) the coherence of $\Dbo$, via the term $\gammaDbo$, (2) the number of signals, and (3) the level of noise.
These three factors limit the possible resolution we can guarantee.

While a coherence condition scaling in $k \mu_0 = O(1)$ is standard for sparse models (see, e.g.,~\cite{Fuchs2005}), 
we impose a slightly more conservative constraint in $O(1/\sqrt{\log(p)})$. 
A typical example for which our result applies is the Hadamard-Dirac dictionary built 
as the concatenaton of a Hadamard matrix and the identity matrix. In this case, we have 
$p = 2m$, $\triple \Dbo \triple_2 = \sqrt{2}$, and $\mu_0 = 1/\sqrt{m}$ with $k=O(\sqrt{m/\log(2m)})$.
In Sec.~\ref{sec:exp}, 
we use such over-complete dictionaries for our simulations.
In addition, observe that because of the upperbound on $\gammaDbo$, Theorem~\ref{thm:main_thm} does not handle per se the case of orthogonal dictionary, 
which we remedy in Theorem~\ref{thm:main_thm_ortho}.

Perhaps surprisingly (and disappointingly), 
our result indicates that, even in a low-noise setting with sufficiently many signals (i.e., the asymptotic regime in $n$), 
we cannot arbitrarily lower the resolution of the local minimum because of the coherence $\mu_0$. In fact, 
the term $e^{-\gammaDbo^2/2}$ is a direct consequence of our proof technique which relies on exact recovery.
It is however worth noting that, since $e^{-\gammaDbo^2/2}$ decreases exponentially fast in $\gammaDbo$, 
the dependence on $\mu_0$ is quite mild 
(e.g., for a radius $\tau$, we have a constraint scaling in $\triple \Dbo \triple_2 \cdot k \mu_0 = O(1/\sqrt{\log(1/\tau)})$).
We next state a complementary theorem for orthogonal dictionaries 
where the radius is not constrained anymore by the coherence:
\paragraph{Local correctness of sparse coding with orthogonal dictionaries:} 
If we now assume that $\Dbo$ is orthogonal (i.e., $\mu_0=0$ and $p=m$ with $\triple \Dbo \triple_2=1$), 
we obtain the following result:
\begin{theorem}[Local minimum of sparse coding---Orthogonal dictionary]\label{thm:main_thm_ortho}
Let us consider our generative model of signals for some reference, orthogonal dictionary $\Dbo \in \Real^{m \times m}$.
If we have:
\begin{itemize}
\item[]\textbf{(Sample complexity)}
$\displaystyle
\quad \frac{\log^3(n)}{n} =  O\Big(
\frac{1}{k^2 \cdot m^4} \Big),
$
\end{itemize}
then, with probability exceeding $1 - [\frac{m^2n}{9}]^{-\frac{m^2}{2}} - e^{-4\sqrt{n}}$, 
problem~(\ref{eq:min_fn}) admits a local minimum in
$$
\bigg\{ \Db \in \Dcal;\ \|\Dbo-\Db\|_\fro = 
O\Big( 
\max\Big\{
m\cdot\log(n)\cdot (\sqrt{\log(n)} + m)/\sqrt{n},\  
\frac{\sigma}{\sigma_\alpha} \cdot \sqrt{m}\Big\}
 \Big)\
\bigg\}.
$$
\end{theorem}
Interestingly, we observe in this case that, given sufficiently many signals, 
we can localize arbitrarily well (up to the noise level) the local minimum around $\Dbo$.
We now discuss relations with previous work in the noiseless setting.
\paragraph{Local correctness of sparse coding without noise:}
If we remove the noise from our signal model, i.e., $\sigma = 0$, the result of Theorems~\ref{thm:main_thm}-\ref{thm:main_thm_ortho}
remains unchanged, except that the radius is not limited anymore by $\frac{\sigma}{\sigma_\alpha} \sqrt{m}$.
We mention that~\cite{Gribonval2010} obtain a sample complexity in $O(p^2\log(p))$ in the noiseless and \emph{square} dictionary setting,
while the result of~\cite{Geng2011} leads to a scaling in $O(p^3)$ (assuming both $k=O(1)$ and $\triple \Dbo \triple_2=O(1)$)
in the noiseless, over-complete case.
In comparison, our analysis suggests a sample complexity in $O(mp^3)$.

These discrepancies are due to the fact that we want to handle the noisy setting;
this has led us to consider a scheme of proof radically different from those proposed in the related work~\cite{Gribonval2010,Geng2011}.
In particular, our formulation in problem~(\ref{eq:min_fn}) differs from that of~\cite{Gribonval2010,Geng2011} 
where the $\ell_1$-norm of~$\Ab$ is minimized over the 
\emph{equality} constraint $\Db\Ab=\Xb$ and the dictionary normalization $\Db \in \Dcal$. 
Optimality is then characterized through the linearization of the equality constraint, 
a technique that could not be easily extended to the noisy case.
We next discuss the main building blocks of the results and give a high-level structure of the proof.
\section{Architecture of the proof of Theorem~\ref{thm:main_thm}}\label{sec:building_blocks}
Our proof strategy consists in using Proposition~\ref{prop:localmin}, that is, 
controlling the sign of $\Delta F_{n}(\Wb,\vb,t)$ defined in~(\ref{eq:DefDeltaFn}).
In fact, since we expect to have for many training samples the equality $f_\xb(\Db(\Wb,\vb,t)) = \phi_\xb(\Db(\Wb,\vb,t)|\sign(\alphabo))$ uniformly for all $(\Wb,\vb)$ , the main idea is to first concentrate on the study of the smooth function
\begin{equation}
\label{eq:DefDeltaPhin}
\Delta \Phi_{n}(\Wb,\vb,t) \defin 
\Phi_n( \Db(\Wb,\vb,t) | \sign(\Abo) ) - \Phi_n( \Dbo | \sign(\Abo) ),
\end{equation}
instead of the original function $\Delta F_{n}(\Wb,\vb,t)$. 

\paragraph{Control of $\Delta\Phi_{n}$:}
This first step consists in uniformly lower bounding $\Delta\Phi_{n}$ with high probability.
As opposed to $\Delta F_{n}$, the function $\Delta\Phi_{n}$ is available explicitly, see~(\ref{def:phi}) and (\ref{eq:DefDeltaPhin}), 
and corresponds to bilinear/quadratic forms in $(\alphabo,\sign(\alphabo),\varepsilonb)$ which we can concentrate around their expectations.
Finally, the uniformity with respect to $(\Wb,\vb)$ is obtained by a standard $\epsilon$-net argument.

\paragraph{Control of $\Delta F_{n}$ via $\Delta\Phi_{n}$:}
The second step consists in lower bounding $\Delta F_{n}$ in terms of $\Delta \Phi_{n}$ uniformly for all
parameters $(\Wb,\vb) \in \Wcal_{\Dbo} \times \Scal^p$. For a given $t \geq 0$, consider the independent events $\{\Ecal_{\mathrm{coincide}}^{i}(t)\}_{i\in\SET{n}}$ defined by
\[
\Ecal_{\mathrm{coincide}}^{i}(t) \defin
\left\{\omega\ \Big|\ 
f_{\xb^{i}(\omega)}(\Db(\Wb,\vb,t)) = \phi_{\xb^{i}(\omega)}(\Db(\Wb,\vb,t)|\sbo),
\quad \forall (\Wb,\vb) \in \Wcal_{\Dbo} \times \Scal^p 
\right\},
\] 
with $\sbo = \sign(\alphabo)$. In words, the event $\Ecal_{\mathrm{coincide}}^{i}(t)$ corresponds to the fact that target function $f_{\xb^{i}(\omega)}(\Db(\cdot,\cdot,t))$ coincides with the idealized one $\phi_{\xb^{i}(\omega)}(\Db(\cdot,\cdot,t)|\sbo)$ for the ``radius'' $t$. 

Importantly, the event $\Ecal_{\mathrm{coincide}}^{i}(t)$ only  involves a \emph{single} signal;  when we consider a collection of $n$ independent signals, 
we should instead study the event $\bigcap_{i=1}^n \Ecal_\mathrm{coincide}^i(t) $ to guarantee that $\Phi_n$ and~$F_n$ (and therefore, $\Delta\Phi_{n}$ and $\Delta F_{n}$) do coincide. However, as the number of observations $n$ becomes large, 
it is unrealistic and not possible to ensure exact recovery both \emph{simultaneously} for the $n$ signals and \emph{with high probability}. 
To get around this issue, we seek to prove that $\Delta F_{n}$ is well approximated by $\Delta\Phi_{n}$ (rather than equal to it) uniformly for all $(\Wb,\vb)$. 
We show that, when $f_{\xb^{i}}(\Db(t))$ and $\phi_{\xb^{i}}(\Db(t)|\sbo)$ \emph{do not} coincide, 
their difference can be bounded, and we obtain:
\begin{equation*}\label{eq:lowerbound_fn}
\Delta F_n(\Wb,\vb,t)  \geq \Delta \Phi_n(\Wb,\vb,t) - r_n.
\end{equation*}
where we detail the definition of the residual term 
\begin{equation*}\label{eq:residual_lowerbound_fn}
r_n(\omega) \defin \frac1n \sum_{i=1}^{n} \indicator{ [\Ecal^{i}_\mathrm{coincide}(t) \cap \Ecal^i_\mathrm{coincide}(0)]^c   }(\omega) 
\cdot \left\{\Lcal_{\xb^{i}}(\Db,\alphabo^{i}) + \Lcal_{\xb^{i}}(\Dbo,\alphabo^{i})\right\}.
\end{equation*}
In the appendix, we show that with high probability:
$
r_n = O([ t^{2} \cdot \sigma_{\alpha}^{2} + 2m \cdot \sigma^{2} + 2\lambda k\sigma_{\alpha}] \cdot (3-\log \kappa)\kappa)
$
with $\kappa \defin \max_{i\in\SET{n}}\Pr(\big[\Ecal^{i}_\mathrm{coincide}(t) \cap \Ecal^i_\mathrm{coincide}(0)\big]^{c})$. 
To bound the size of $r_{n}$, we now control $\kappa$.

\paragraph{Control of $\kappa$, exact sign recovery for \emph{perturbed} dictionaries:}
We need to determine sufficient conditions under which
$\phi_\xb(\Db(\Wb,\vb,t)|\sign(\alphabo))$ and $f_\xb(\Db(\Wb,\vb,t))$ coincide for all $(\Wb,\vb)$, and control the probability of this event.
As briefly exposed in Sec.~\ref{sec:sparsecoding}, 
it turns out that this question comes down to 
studying exact recovery for some $\ell_1$-regularized least-squares problems.
Exact sign recovery in the problem associated with 
$f_\xb(\Dbo)$ has already been well-studied~(see, e.g.,~\cite{ Wainwright2009,Fuchs2005, Zhao2006}).
However, in our context, we need the same conclusion to hold \emph{not only at the dictionary $\Dbo$, 
but also at $\Db(\Wb,\vb,t)\neq\Dbo$} \emph{uniformly} for all parameters $(\Wb,\vb)$. It turns out that going away from the reference dictionary $\Dbo$ acts as a second source of noise 
whose variance depends on the radius~$t$.
We make this statement precise in Propositions 2-3 in the supplementary material.
These results are in the same line as Theorem~1 in~\cite{Mehta2012}.
\paragraph{Discussing when the lower-bound on $\Delta F_{n}$ is positive:} With all the previous elements in place, we have 
a lower-bound for $\inf_{\Wb \in \Wcal_{\Dbo},\vb \in \Scal^{p}} \Delta F_{n}(\Wb,\vb,t)$, valid with high probability.
It finally suffices to discuss when it is stricly positive to conclude with Proposition~\ref{prop:localmin}.
\section{Experiments}\label{sec:exp}
We illustrate the results from Sec.~\ref{sec:main_results}.
Although we do not manage to highlight the exact scalings in $(p,m)$ which we proved in Theorems~\ref{thm:main_thm}-\ref{thm:main_thm_ortho}, 
our experiments still underline the main interesting trends put forward by our results, 
such as the dependencies with respect to $n$ and $\sigma$.

Throughout this section, the non-zero coefficients of $\alphabo$ are uniformly drawn with $|[\alphabo]_j| \in [0.1, 10]$
and the noise follows a standard Gaussian distribution with variance $\sigma$.
We detail two important aspects of the experiments, namely, the choice of $\lambda$, and how we deal with the invariance of problem~(\ref{eq:min_fn}) 
(see Sec.~\ref{sec:main_obj}).
Since our analysis relies on exact recovery, we first tune $\lambda$ over a logarithmic grid to match the oracle sparsity level. Note that this tuning step is performed over an auxiliary set of signals.  
On the other hand, we know that the dictionary $\hat{\Db}$ that we learn by minimizing problem~(\ref{eq:min_fn}) may differ from $\Dbo$ up to sign flips and atom permutations. 
Since both $\hat{\Db}$ and $\Dbo$ have normalized atoms, finding the closest dictionary (in Frobenius norm) up to these transformations 
is equivalent to an assignment problem based on the absolute correlation matrix $\hat{\Db}^\top \Dbo$, which can be efficiently solved using the Hungarian algorithm~\citep{Kuhn1955}.

To solve problem~(\ref{eq:min_fn}), we use the stochastic algorithm from~\cite{Mairal2010}\footnote{The code is available at~\texttt{http://www.di.ens.fr/willow/SPAMS/}.}
where the batch size is fixed to $512$, while the number of epochs is chosen so as to pass over each signal 25 times (on average). 
We consider two types of initialization, i.e., either from (1) a random dictionary, or (2) the correct $\Dbo$.

To begin with, we illustrate Theorem~\ref{thm:main_thm} with $\Dbo$ a Hadamard-Dirac (over-complete) dictionary.
The sparsity level is fixed such that $\triple \Dbo \triple_2 \cdot k\mu_0 = O(1/\sqrt{\log(p)})$, 
and we consider a small enough noise level, so that 
the radius is primarily limited by the number $n$ of signals.
The normalized error $\|\Dbo - \hat{\Db}\|_\fro/\sqrt{mp^3}$ versus $n$ is plotted in Fig.~\ref{fig:exp}.
We then focus on Theorem~\ref{thm:main_thm_ortho}, with $\Dbo$ a Hadamard (orthogonal) dictionary.
We consider sufficiently many signals ($n = 75,000$) so that the radius is only limited by $\sqrt{m}\cdot \sigma/\sigma_\alpha$.
The normalized error $\|\Dbo - \hat{\Db}\|_\fro/\sqrt{m}$ versus the level of noise is displayed in Fig.~\ref{fig:exp}.
\begin{figure}[!h]
\centering
\begin{tabular}{cc}
\hspace*{-0.2cm}\includegraphics[width=0.49\linewidth]{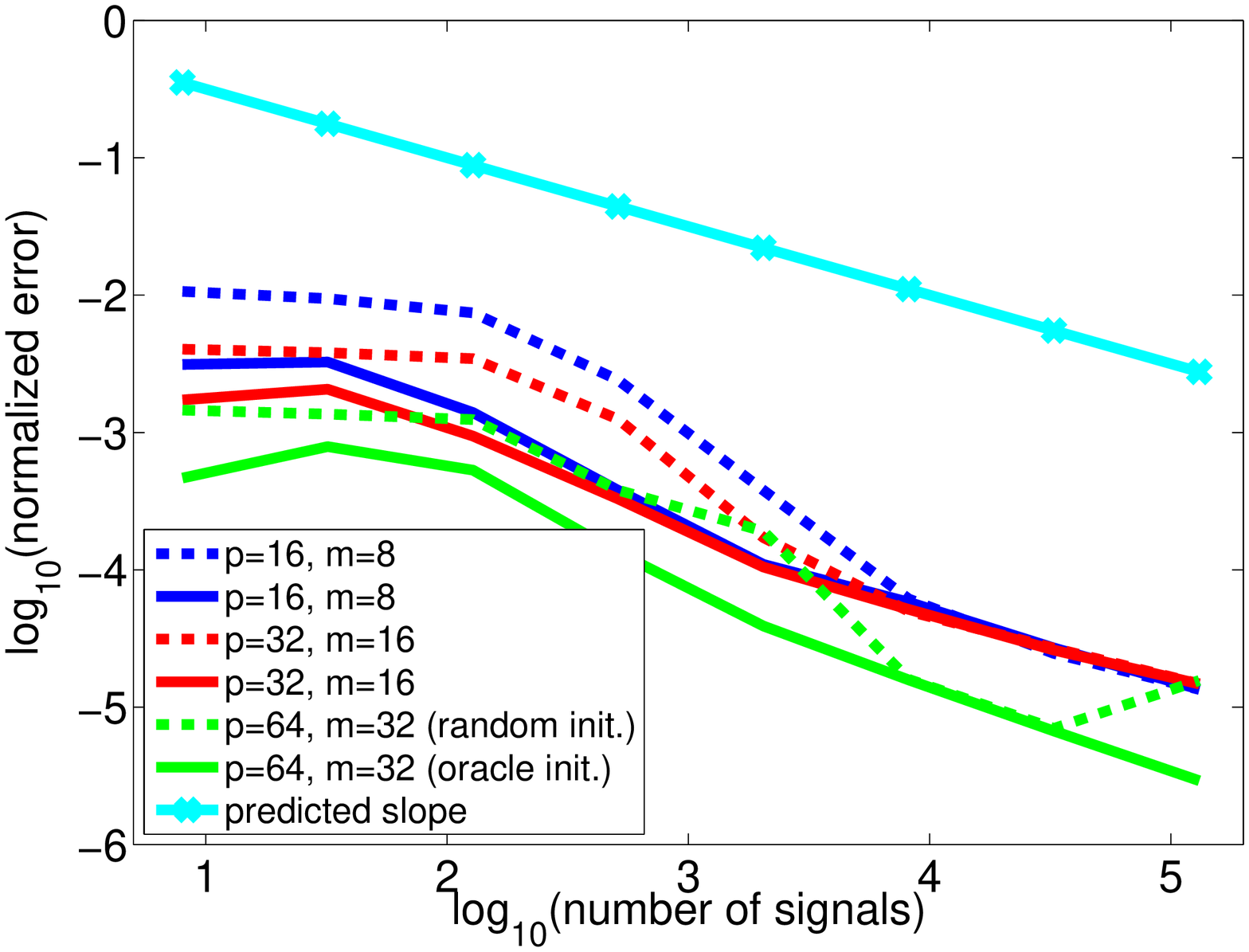} &
\includegraphics[width=0.49\linewidth]{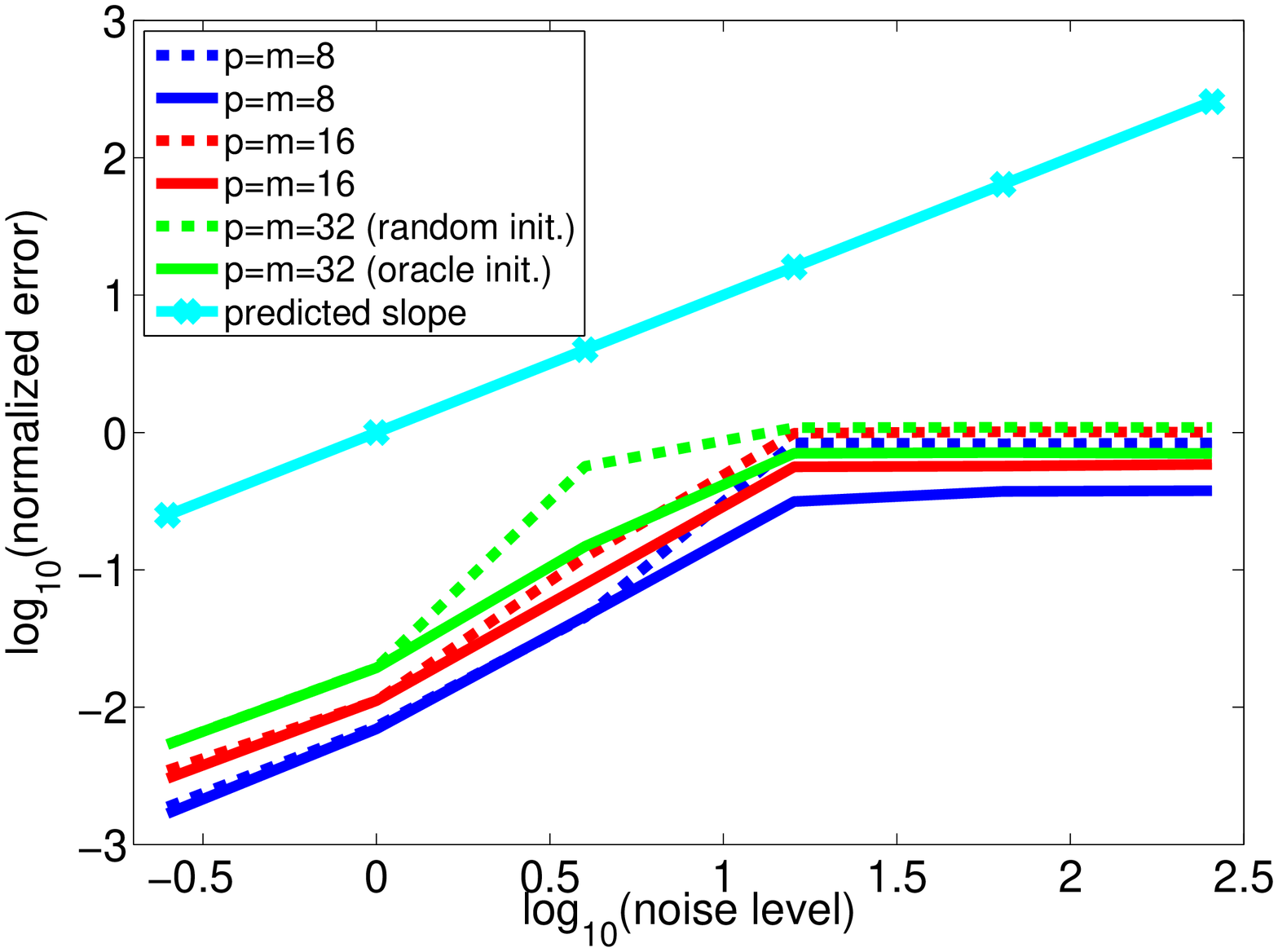} 
\end{tabular}
\vspace*{-0.3cm}
\caption{Normalized error between $\Dbo$ and the solution of problem~(\ref{eq:min_fn}), versus the number of signals (left) and the noise level (right).
The curves represent the median error based on 5 runs, for random and oracle initializations.
More details can be found in the text; best seen in color.}
\label{fig:exp}
\end{figure}

The curves represented in Fig.~\ref{fig:exp} do not perfectly superimposed, 
thus implying that our results do not capture the exact scalings in $(p,m)$ (our bounds appear in fact as too pessimistic).
However, our theory seems to account for the main dependencies with respect to $n$ and $\sigma$, 
as the good agreement with the predicted slopes proves it.  
Interestingly, while we would expect the curves in the left plot of Fig.~\ref{fig:exp} 
to tail off at some point because of the coherence (term $e^{-\gammaDbo^2/2}$ in the bound of the radius), 
it seems that there is in practice a much milder dependency with respect to the coherence.
Finally, we can observe that both the random and oracle initializations seem to lead to the same behavior, 
thus raising the questions of the potential \emph{global} characterization of these local minima. 
\section{Conclusion}
We have conducted a non-asymptotic analysis of the local minima of sparse coding in the presence of noise, 
thus extending prior work which focused on noiseless settings~\citep{Gribonval2010,Geng2011}.
Within a probabilistic model of sparse signals, we have shown that a local minimum exists with high probability around the reference dictionary. 

Our study can be further developed in multiple ways.
On the one hand, while we have assumed \emph{deterministic} coherence-based conditions scaling in $O(1/k)$,
it may interesting to consider non-deterministic assumptions~\citep{Candes2009}, 
which are likely to lead to improved scalings.
On the other hand, we may also use more realistic generative models for $\alphabo$, for instance, 
spike and slab models~\citep{Ishwaran2005}, 
or signals with compressible priors~\citep{Gribonval2011}.

Also, we believe that our approach can handle the presence of outliers, provided their total energy remains small enough;
we plan to make this argument formal in future work.

Finally, it remains challenging to extend our local properties to global ones due to the intrinsic non-convexity of the problem; 
an appropriate use of convex relaxation techniques~\citep{Bach2008c} may prove useful in this context.
\section*{Acknowledgements}
This work was supported by the European Research Council (SIERRA and SIPA Projects) and 
by the EU FP7, SMALL project, FET-Open grant number 225913.
\bibliographystyle{plainnat}
\bibliography{./main_bibliography}

\begin{thebibliography}{38}
\providecommand{\natexlab}[1]{#1}
\providecommand{\url}[1]{\texttt{#1}}
\expandafter\ifx\csname urlstyle\endcsname\relax
  \providecommand{\doi}[1]{doi: #1}\else
  \providecommand{\doi}{doi: \begingroup \urlstyle{rm}\Url}\fi

\bibitem[Absil et~al.(2008)Absil, Mahony, and Sepulchre]{Absil2008}
P.~A. Absil, R.~Mahony, and R.~Sepulchre.
\newblock \emph{Optimization algorithms on matrix manifolds}.
\newblock Princeton University Press, 2008.

\bibitem[Bach et~al.(2008)Bach, Mairal, and Ponce]{Bach2008c}
F.~Bach, J.~Mairal, and J.~Ponce.
\newblock Convex sparse matrix factorizations.
\newblock Technical report, Preprint arXiv:0812.1869, 2008.

\bibitem[Bach et~al.(2011)Bach, Jenatton, Mairal, and Obozinski]{Bach2011}
F.~Bach, R.~Jenatton, J.~Mairal, and G.~Obozinski.
\newblock Optimization with sparsity-inducing penalties.
\newblock \emph{Foundations and Trends in Machine Learning}, 4\penalty0
  (1):\penalty0 1--106, 2011.

\bibitem[Bradley and Bagnell(2009)]{Bradley2009a}
D.~M. Bradley and J.~A. Bagnell.
\newblock Convex coding.
\newblock In \emph{Proc. UAI}, 2009.

\bibitem[Buldygin and Kozachenko(2000)]{Buldygin2000}
V.~V. Buldygin and I.~U.~V. Kozachenko.
\newblock \emph{Metric characterization of random variables and random
  processes}, volume 188.
\newblock American Mathematical Society, 2000.

\bibitem[Cand{\`e}s and Plan(2009)]{Candes2009}
E.~J. Cand{\`e}s and Y.~Plan.
\newblock {Near-ideal model selection by $\ell_1$ minimization}.
\newblock \emph{Annals of Statistics}, 37\penalty0 (5A):\penalty0 2145--2177,
  2009.

\bibitem[Chen et~al.(1998)Chen, Donoho, and Saunders]{Chen1998}
S.~S. Chen, D.~L. Donoho, and M.~A. Saunders.
\newblock Atomic decomposition by basis pursuit.
\newblock \emph{SIAM Journal on Scientific Computing}, 20\penalty0
  (1):\penalty0 33--61, 1998.

\bibitem[Comon and Jutten(2010)]{ComonJutten2010}
P.~Comon and C.~Jutten, editors.
\newblock \emph{Handbook of Blind Source Separation, Independent Component
  Analysis and Applications}.
\newblock Academic Press, 2010.

\bibitem[Cucker and Smale(2002)]{Cucker2002}
F.~Cucker and S.~Smale.
\newblock On the mathematical foundations of learning.
\newblock \emph{Bulletin of the American Mathematical Society}, 39:\penalty0
  1--49, 2002.

\bibitem[De~la Pe{\~n}a and Gin{\'e}(1999)]{Del1999}
V.~De~la Pe{\~n}a and E.~Gin{\'e}.
\newblock \emph{Decoupling: from dependence to independence}.
\newblock Springer Verlag, 1999.

\bibitem[Donoho and Huo(2001)]{Donoho2001}
D.~L. Donoho and X.~Huo.
\newblock Uncertainty principles and ideal atomic decomposition.
\newblock \emph{IEEE T. Inform. Theory}, 47\penalty0 (7):\penalty0 2845--2862,
  2001.

\bibitem[Dym(2007)]{Dym2007}
H.~Dym.
\newblock \emph{Linear algebra in action}.
\newblock 2007.

\bibitem[Fuchs(2005)]{Fuchs2005}
J.~J. Fuchs.
\newblock {Recovery of exact sparse representations in the presence of bounded
  noise}.
\newblock \emph{IEEE T. Inform. Theory}, 51\penalty0 (10):\penalty0 3601--3608,
  2005.

\bibitem[Gautschi(1998)]{Gautschi1998}
W.~Gautschi.
\newblock The incomplete {G}amma functions since {T}ricomi.
\newblock In \emph{In Tricomi's Ideas and Contemporary Applied Mathematics,
  Atti dei Convegni Lincei, n.147, Accademia Nazionale dei Lincei}, 1998.

\bibitem[Geng et~al.(2011)Geng, Wang, and Wright]{Geng2011}
Q.~Geng, H.~Wang, and J.~Wright.
\newblock {On the Local Correctness of L1 Minimization for Dictionary
  Learning}.
\newblock Technical report, Preprint arXiv:1101.5672, 2011.

\bibitem[Gribonval and Schnass(2010)]{Gribonval2010}
R.~Gribonval and K.~Schnass.
\newblock Dictionary identification---sparse matrix-factorization via
  $\ell_1$-minimization.
\newblock \emph{IEEE T. Inform. Theory}, 56\penalty0 (7):\penalty0 3523--3539,
  2010.

\bibitem[Gribonval et~al.(2011)Gribonval, Cevher, and Davies]{Gribonval2011}
R.~Gribonval, V.~Cevher, and M.~E. Davies.
\newblock Compressible distributions for high-dimensional statistics.
\newblock Technical report, preprint arXiv:1102.1249, 2011.

\bibitem[Horn and Johnson(1990)]{Horn1990}
R.~A. Horn and C.~R. Johnson.
\newblock \emph{{Matrix analysis}}.
\newblock Cambridge University Press, 1990.

\bibitem[Hsu et~al.(2011)Hsu, Kakade, and Zhang]{Hsu2011}
D.~Hsu, S.~M. Kakade, and T.~Zhang.
\newblock A tail inequality for quadratic forms of subgaussian random vectors.
\newblock Technical report, Preprint arXiv:1110.2842, 2011.

\bibitem[Ishwaran and Rao(2005)]{Ishwaran2005}
H.~Ishwaran and J.~S. Rao.
\newblock Spike and slab variable selection: frequentist and {B}ayesian
  strategies.
\newblock \emph{Annals of Statistics}, 33\penalty0 (2):\penalty0 730--773,
  2005.

\bibitem[Jenatton et~al.(2011)Jenatton, Mairal, Obozinski, and
  Bach]{Jenatton2010b}
R.~Jenatton, J.~Mairal, G.~Obozinski, and F.~Bach.
\newblock Proximal methods for hierarchical sparse coding.
\newblock \emph{Journal of Machine Learning Research}, 12:\penalty0 2297--2334,
  2011.

\bibitem[Krause and Cevher(2010)]{Krause2010}
A.~Krause and V.~Cevher.
\newblock Submodular dictionary selection for sparse representation.
\newblock In \emph{Proceedings of the International Conference on Machine
  Learning (ICML)}, 2010.

\bibitem[Kuhn(1955)]{Kuhn1955}
H.~W. Kuhn.
\newblock The {H}ungarian method for the assignment problem.
\newblock \emph{Naval research logistics quarterly}, 2\penalty0 (1-2):\penalty0
  83--97, 1955.

\bibitem[Magnus and Neudecker(1988)]{Magnus1988}
J.~R. Magnus and H.~Neudecker.
\newblock \emph{Matrix differential calculus with applications in statistics
  and econometrics}.
\newblock John Wiley \& Sons, 1988.

\bibitem[Mairal et~al.(2010)Mairal, Bach, Ponce, and Sapiro]{Mairal2010}
J.~Mairal, F.~Bach, J.~Ponce, and G.~Sapiro.
\newblock Online learning for matrix factorization and sparse coding.
\newblock \emph{Journal of Machine Learning Research}, 11\penalty0
  (1):\penalty0 19--60, 2010.

\bibitem[Mallat(2008)]{Mallat:2008aa}
S.~Mallat.
\newblock \emph{A Wavelet Tour of Signal Processing}.
\newblock Academic Press, 3rd edition, December 2008.

\bibitem[Maurer and Pontil(2010)]{Maurer2010}
A.~Maurer and M.~Pontil.
\newblock $k$-dimensional coding schemes in hilbert spaces.
\newblock \emph{IEEE T. Inform. Theory}, 56\penalty0 (11):\penalty0 5839--5846,
  2010.

\bibitem[Mehta and Gray(2012)]{Mehta2012}
N.~A. Mehta and A.~G. Gray.
\newblock On the sample complexity of predictive sparse coding.
\newblock Technical report, preprint arXiv:1202.4050, 2012.

\bibitem[Olshausen and Field(1997)]{Olshausen1997}
B.~A. Olshausen and D.~J. Field.
\newblock Sparse coding with an overcomplete basis set: A strategy employed by
  {V}1?
\newblock \emph{Vision Research}, 37:\penalty0 3311--3325, 1997.

\bibitem[Tibshirani(1996)]{Tibshirani1996}
R.~Tibshirani.
\newblock Regression shrinkage and selection via the {L}asso.
\newblock \emph{Journal of the Royal Statistical Society. Series B}, pages
  267--288, 1996.

\bibitem[Tropp(2004)]{Tropp2004}
J.~A. Tropp.
\newblock {Greed is good: Algorithmic results for sparse approximation}.
\newblock \emph{IEEE T. Inform. Theory}, 50\penalty0 (10):\penalty0 2231--2242,
  2004.

\bibitem[Vainsencher et~al.(2010)Vainsencher, Mannor, and
  Bruckstein]{Vainsencher2010}
D.~Vainsencher, S.~Mannor, and A.~M. Bruckstein.
\newblock The sample complexity of dictionary learning.
\newblock Technical report, Preprint arXiv:1011.5395, 2010.

\bibitem[{Van de Geer} and B{\"u}hlmann(2009)]{Van2009}
S.~{Van de Geer} and P.~B{\"u}hlmann.
\newblock {On the conditions used to prove oracle results for the Lasso}.
\newblock \emph{Electronic Journal of Statistics}, 3:\penalty0 1360--1392,
  2009.

\bibitem[Vershynin(2010)]{Vershynin2010}
R.~Vershynin.
\newblock Introduction to the non-asymptotic analysis of random matrices.
\newblock Technical report, Preprint arXiv:1011.3027, 2010.

\bibitem[Wainwright(2009)]{Wainwright2009}
M.~J. Wainwright.
\newblock Sharp thresholds for noisy and high-dimensional recovery of sparsity
  using $\ell_1$-~constrained quadratic programming.
\newblock \emph{IEEE T. Inform. Theory}, 55:\penalty0 2183--2202, 2009.

\bibitem[Zhang(2009)]{Zhang2009}
T.~Zhang.
\newblock Some sharp performance bounds for least squares regression with l1
  regularization.
\newblock \emph{Annals of Statistics}, 37\penalty0 (5A):\penalty0 2109--2144,
  2009.

\bibitem[Zhao and Yu(2006)]{Zhao2006}
P.~Zhao and B.~Yu.
\newblock {On model selection consistency of Lasso}.
\newblock \emph{Journal of Machine Learning Research}, 7:\penalty0 2541--2563,
  2006.

\bibitem[Zhou et~al.(2009)Zhou, Chen, Paisley, Ren, Sapiro, and
  Carin]{Zhou2009}
M.~Zhou, H.~Chen, J.~Paisley, L.~Ren, G.~Sapiro, and L.~Carin.
\newblock Non-parametric {B}ayesian dictionary learning for sparse image
  representations.
\newblock In \emph{Adv. NIPS}, 2009.

\end{thebibliography}

\appendix

\section{Detailed Statements of the Main results}
We gather in this appendix the detailed statements and the proofs of the simplified results presented in the core of the paper.
In particular, we show in this section that under appropriate scalings of the problem dimensions $(m,p)$, number of training samples $n$, and model parameters $k,\loweralpha,\sigma_{\alpha},\sigma,\mu_0$, it is possible to prove that, with high probability,
the problem of sparse coding admits a local minimum in a certain neighborhood of $\Dbo$ of controlled size, for appropriate choices of the regularization parameter $\lambda$.
\subsection{Minimum local of sparse coding}\label{sec:minimum_sparsecoding}
We present here a complete and detailed version of our result upon which the theorems presented in the paper are built.
\begin{theorem}[Local minimum of sparse coding]
\label{th:minimum_sparsecoding}
Let us consider our generative model of signals for some reference dictionary $\Dbo \in \Real^{m \times p}$ with coherence $\mu_0$.
Introduce the parameters 
$
q_\alpha \defin \frac{\Exp[\alpha^2]}{\sigma_\alpha^2}
$
and
$
\Qcal_\alpha \defin \frac{\Exp[\alpha^2]}{\sigma_\alpha\cdot\Exp[|\alpha|]}
$
which depend on the distribution of $\alphabo$ only.
Consider the following quantities:
\begin{eqnarray*}
\tau = \tau(\Dbo,\alphabo) &\defin& \min\Big\{ \frac{\loweralpha}{ \sigma_\alpha}, 
\frac{1}{3c_0}\cdot \frac{\Qcal_\alpha}{k \triple \Dbo \triple_2}\Big\}\\
\gamma = \gamma(n,\Dbo,\alphabo) &\defin& 
\frac{1}{2} \min\Bigg\{\sqrt{2\log(n)},
\frac{1}{2\sqrt{2} c_0 c_\gamma} \cdot 
\frac{\Qcal_\alpha}{\triple \Dbo \triple_2 \cdot k \mu_0}
\Bigg\},
\end{eqnarray*}
and let us define the radius $t \in \Real_+$ by 
$$
t \defin 
\max\bigg\{
\frac{4\sqrt{2} c_\gamma}{q_\alpha} \cdot p \cdot\Big\{ c_1 \gamma^3 e^{-\gamma^2} + 2c_2\cdot \gamma\cdot \Big[ m p \frac{\log(n)}{n}\Big]^{1/2}\Big\},  
\frac{\sigma}{\sigma_\alpha}\cdot \sqrt{m}\bigg\}
$$
for some universal constants $c_{*}$.
Provided the following conditions are satisfied:
\begin{itemize}
 \item[]\textbf{(Coherence)} \hspace*{1.5cm}
$\displaystyle
\triple \Dbo \triple_2 \cdot k \mu_0
\leq 
\frac{1}{4\sqrt{2} c_0 c_\gamma}\cdot \frac{\Qcal_\alpha}
{\sqrt{\log( 69 c_1 c_\gamma^2 \cdot \frac{1}{q_\alpha} 
\cdot \frac{p}{\tau}  )}},
$
 \item[]\textbf{(Sample complexity)}
$\displaystyle\
 \frac{\log(n)}{n} \leq 
\frac{q_\alpha^2}{c_3} \cdot 
\frac{1}{m\cdot p^3} \cdot
\frac{\tau^2}{\gamma^4},
$
\end{itemize}
one can find a regularization parameter 
$
\lambda
$
proportional to
$ 
\gamma \cdot \sigma_{\alpha} \cdot  t
$, 
and with probability exceeding
$$
1 - \Big(\frac{mpn}{9}\Big)^{-\frac{mp}{2}} - e^{-4\sqrt{n}},
$$
there exists a local minimum in 
$
\Big\{ \Db \in \Dcal;\ \|\Dbo-\Db\|_\fro < t  \Big\}.
$ 
\end{theorem}
As it will discussed at greater length in Sec.~\ref{sec:building_blocks}, 
we can see that the probability of success of Theorem~\ref{th:minimum_sparsecoding}
can be decomposed into the contributions of the concentration of the surrogate function and the residual term. 
We next present a second result which assumes a more constrained signal model:
\begin{theorem}[Local minimum of sparse coding with noiseless/bounded signals]\label{th:minimum_sparsecoding_simple}
Let us consider our generative model of signals for some reference dictionary $\Dbo \in \Real^{m \times p}$ with coherence $\mu_0$.
Further assume that $\alphabo$ is almost surely upper bounded by $\upperalpha$ and that there is no noise, that is, $\sigma =0$.
Introduce the parameters 
$
q_\alpha \defin \frac{\Exp[\alpha^2]}{\upperalpha\cdot\Exp[|\alpha|]}
$
and 
$
\Qcal_\alpha \defin \frac{\Exp[\alpha^2]}{\sigma_\alpha\cdot\Exp[|\alpha|]}
$
which depend on the distribution of $\alphabo$ only.
Consider the radius $t \in \Real_+$: 
$$
t \defin \frac{8 c_1 c_\lambda}{q_\alpha}
\bigg[ k m p^3\cdot \frac{\log(n)}{n} \bigg]^{1/2}
$$
for some universal constants $c_{*}$.
Provided the following conditions are satisfied:
\begin{itemize}
 \item[]\textbf{(Coherence)} \hspace*{1.5cm}
$\displaystyle\triple \Dbo \triple_2\cdot k^{3/2}\mu_0 
\leq 
\frac{1}{c_0 c_\lambda} \cdot q_\alpha$ 
 \item[]\textbf{(Sample complexity)}
$\displaystyle\
\frac{\log(n)}{n} \leq 
\frac{1}{k^2 m p^3} \cdot \bigg[\frac{\Exp[\alpha^2]}{\upperalpha^2} \cdot \frac{1}{9 c_1 c_\lambda^2}
\cdot  
\min\bigg\{ \frac{\loweralpha}{\sigma_\alpha}, \frac{1}{5c_0} \cdot \frac{\Qcal_\alpha}{k \cdot \triple \Dbo \triple_2} \bigg\} \bigg]^{2},
$
\end{itemize}
one can find a regularization parameter $\lambda$ proportional to $\sqrt{k}\cdot \upperalpha\cdot t$, 
and with probability exceeding
$$
1 - \Big(\frac{mpn}{9}\Big)^{-mp/2},
$$
there exists a local minimum in 
$
\Big\{ \Db \in \Dcal;\ \|\Dbo-\Db\|_\fro < t  \Big\}.
$
\end{theorem}
These two theorems, which are proved in Section~\ref{app:minimum_sparsecoding}, heavily relies on the following central result.
\subsection{The backbone of the analysis}\label{sec:main_result}
We concentrate on the result which constitutes the backbone of our analysis. Indeed, we next show how the difference
\begin{equation}
\label{eq:DefDeltaFn}
\Delta F_{n}(\Wb,\vb,t) \defin F_{n}(\Db(\Wb,\vb,t)) - F_{n}(\Db_{0}).
\end{equation}
is lower bounded with high probability and uniformly with respect to all possible choices of the parameters $(\Wb,\vb)$. 
The theorem and corollaries displayed in the core of the paper are consequences of this general theorem, 
discussing under which conditions/scalings this lower bound can be proved to be sufficient (i.e., strictly positive)
to exploit Proposition~1 and conclude to the existence of a local minimum for $t$ appropriately chosen. 
We define
\begin{equation}
\label{eq:DefCrho}
Q_t  \defin  \frac{1}{\sqrt{1-k\mu(t)}}\quad \text{and}\quad C_t  \defin  \frac{1}{\sqrt{1-\delta_k(\Dbo)}-t},
\end{equation}
where the quantity $\delta_k(\Dbo)$ is the RIP constant itself defined in Section~\ref{app:RIP}.
\begin{theorem}
\label{th:control_Fn}
Let $\loweralpha,\sigma_{\alpha}$ be the  parameters of the coefficient model.  
Consider $\Dbo$ a dictionary in $\Real^{m \times p}$ with $\mu_{0} < 1/2$ and let $k$, $t>0$ be such that
\begin{eqnarray}
\label{eq:AssumptionKT1}
k\mu(t) &<& 1/2\\
\frac{3t}{2-Q_{t}^{2}}  &<& 
\label{eq:AssumptionKT2}
\frac{4\loweralpha}{9\sigma_{\alpha}}
\end{eqnarray}
Then for small enough noise levels $\sigma$ one can find a regularization parameter $\lambda>0$ such that
\begin{equation}\label{eq:BoundsLambdaTh}
\frac{3}{2-Q_{t}^{2}}
\cdot \sqrt{t^{2} \sigma_{\alpha}^{2} + m\sigma^{2}}
\leq  \lambda \leq \frac 49 \loweralpha.
\end{equation}
Given $\sigma$ and $\lambda$ satisfying~(\ref{eq:BoundsLambdaTh}),  we define
\begin{equation}
\gamma \defin \frac{\lambda(2-Q_{t}^{2})}{\sqrt{5} \cdot \sqrt{t^{2} \sigma_{\alpha}^{2} + m\sigma^{2}}} \geq \sqrt{2\log2}.
\end{equation}
Let $\xb^{i} \in \Real^m$, $i \in \SET{n}$, where $n/\log n \geq mp$, be generated according to the signal model. 
Then, except with probability at most $\left(\frac{mpn}{9}\right)^{-mp/2} + \exp(-4n \cdot e^{-\gamma^{2}})$ we have
\begin{eqnarray}
\inf_{\Wb \in \Wcal_{\Dbo},\vb \in \Scal^{p}} \Delta F_{n}(\Wb,\vb,t)
&\geq &
(1 - \Kcal^2) \cdot \frac {\Exp[ \alpha_{0}^{2}]}{2} 
\cdot  
\frac{k}{p} \cdot t^{2}\notag\\
&&
-Q_t^2\left(\frac{16}{9} Q_{t}^2 + 3 \right)\cdot \Exp\{|\alpha_{0}|\} \cdot t \cdot \frac{k}{p} \cdot \triple \Dbo \triple_2 \cdot k\mu(t) \cdot \lambda \notag\\
&&
- A \cdot \gamma^{2} \cdot e^{-\gamma^{2}}\notag\\
&& -B \cdot \sqrt{mp \frac{\log n}{n}},
\end{eqnarray}
where
\begin{eqnarray}
\Kcal &\defin & C_{t} \cdot \big(\triple \Dbo \triple_2 \cdot \sqrt{k/p} + t\big)\\
A &\defin&
367 \cdot \left(t^{2}  \sigma_{\alpha}^{2} + 2m  \sigma^{2} + 2\lambda k\sigma_{\alpha}\right)\\
B &\defin &
3045 \left( k\sigma_{\alpha}^{2} \cdot t+2 m\sigma^{2} + 2\lambda k \sigma_{\alpha} \right).
\end{eqnarray}
\end{theorem}
Roughly speaking, the lower bound we obtain can be decomposed into three terms:
(1) the expected value of our surrogate function valid uniformly for all parameters $(\vb,\Wb)$,
(2) the contributions of the residual term (discussed in the next section) introducing the quantity $\gamma$,
and (3) the probabilisitc concentrations over the $n$ signals of the surrogate function and the residual term.

The proof of the theorem and its main building blocks are detailed in the next section.
\section{Architecture of the proof of Theorem~\ref{th:control_Fn}}\label{sec:building_blocks}
Since we expect to have for many training samples the equality $f_\xb(\Db(\Wb,\vb,t)) = \phi_\xb(\Db(\Wb,\vb,t)|\sign(\alphabo))$ uniformly for all $(\Wb,\vb)$ , the main idea is to first concentrate on the study of the smooth function
\begin{equation}
\label{eq:DefDeltaPhin}
\Delta \Phi_{n}(\Wb,\vb,t) \defin 
\Phi_n( \Db(\Wb,\vb,t) | \sign(\Abo) ) - \Phi_n( \Dbo | \sign(\Abo) ),
\end{equation}
instead of the original function $\Delta F_{n}(\Wb,\vb,t)$. 

\subsection{Control of $\Delta\Phi_{n}$}
The first step consists in uniformly lower bounding $\Delta\Phi_{n}$ with high probability.
\begin{proposition}\label{prop:maindeltaphi}
Assume that $k\mu(t)\leq1/2$ then for any $n$ such that 
\begin{equation}
\frac{n}{\log n} \geq mp,
\end{equation}
except with probability at most $\left(\frac{mpn}{9}\right)^{-mp/2}$, we have
\begin{eqnarray}
\label{eq:LowerBoundDeltaPhi}
\inf_{\Wb \in \Wcal_{\Dbo},\vb \in \Scal^{p}} \Delta\Phi_{n}(\Wb,\vb,t)
&\geq&
(1 - \Kcal^2) \cdot \frac {\Exp[ \alpha_{0}^{2}]}{2} 
\cdot  
\frac{k}{p} \cdot t^{2}\notag\\
&&
-Q_t^2 \cdot t \cdot \frac{k}{p} \cdot \triple \Dbo \triple_2 \cdot k\mu(t) \cdot \lambda \cdot \left(4 Q_{t}^2 \lambda + 3\Exp\{|\alpha_{0}|\} \right)
\notag\\
&& -B \cdot \sqrt{mp \frac{\log n}{n}},
 \end{eqnarray}
 where
\begin{eqnarray*}
\Kcal &\defin &
C_{t} \cdot \big(\triple \Dbo \triple_2 \cdot \sqrt{k/p} + t\big)\\
B &\defin& 
3045 \left( k\sigma_{\alpha}^{2} \cdot t+2 m\sigma^{2} + \lambda k \sigma_{\alpha} +\lambda^{2}k \cdot t \right).
\end{eqnarray*}
\end{proposition}
The proof of this proposition is given in Section~\ref{sec:proof_control_difference_phi}.

\subsection{Control of $\Delta F_{n}$ in terms of $\Delta\Phi_{n}$}\label{sec:difference}
The second step consists in lower bounding $\Delta F_{n}$ in terms of $\Delta \Phi_{n}$ uniformly for all $(\Wb,\vb) \in \Wcal_{\Dbo} \times \Scal^p$. For a given $t \geq 0$, consider the independent events $\{\Ecal_{\mathrm{coincide}}^{i}(t)\}_{i\in\SET{n}}$ defined by
\[
\Ecal_{\mathrm{coincide}}^{i}(t) \defin
\left\{\omega\ \Big|\ 
f_{\xb^{i}(\omega)}(\Db(\Wb,\vb,t)) = \phi_{\xb^{i}(\omega)}(\Db(\Wb,\vb,t)|\sbo),
\quad \forall (\Wb,\vb) \in \Wcal_{\Dbo} \times \Scal^p 
\right\},
\] 
with $\sbo = \sign(\alphabo)$. In words, the event $\Ecal_{\mathrm{coincide}}^{i}(t)$ corresponds to the fact that target function $f_{\xb^{i}(\omega)}(\Db(\cdot,\cdot,t))$ coincides with the idealized one $\phi_{\xb^{i}(\omega)}(\Db(\cdot,\cdot,t)|\sbo)$ for the ``radius'' $t$. 

Importantly, the event $\Ecal_{\mathrm{coincide}}^{i}(t)$ only  involves a \emph{single} signal;  when we consider a collection of $n$ independent signals, 
we should instead study the event $\bigcap_{i=1}^n \Ecal_\mathrm{coincide}^i(t) $ to guarantee that $\Phi_n$ and $F_n$ (and therefore, $\Delta\Phi_{n}$ and $\Delta F_{n}$) do coincide. However, as the number of observations $n$ becomes large, 
it is unrealistic and not possible to ensure exact recovery both \emph{simultaneously} for the $n$ signals and \emph{with high probability}. 

To get around this issue, we will relax our expectations and seek to prove that $\Delta F_{n}$ is well approximated by $\Delta\Phi_{n}$ (rather than equal to it) uniformly for all $(\Wb,\vb)$. This will be achieved by showing that, when $f_{\xb^{i}}(\Db(t))$ and $\phi_{\xb^{i}}(\Db(t)|\sbo)$ \emph{do not} coincide, their difference can be bounded. For any $\Db \in \Real^{m\times p}$, 
we have by the very definition~\eqref{eq:fi}, $0 \leq f_\xb(\Db) \leq \Lcal_{\xb}(\Db,\alphabo)$. We have as well by the definition~\eqref{eq:defphi}:
\[
0 \leq \phi_{\xb}(\Db|\sign(\alphabo))
\leq
\min_{\alphab\in\Real^p, \ \sign(\alphab)=\sign(\alphabo) }\frac{1}{2}\|\xb - \Db\alphab\|_2^2+\lambda\cdot\sign(\alphabo)^\top\alphab
\leq \Lcal_{\xb}(\Db,\alphabo).
\]
It follows that for all $(\Wb,\vb) \in \Wcal_{\Dbo} \times \Scal^p$ we have, with $\Db = \Db(\Wb,\vb,t)$, 
\begin{eqnarray*}
\phi_{\xb}(\Dbo|\sbo) - \phi_{\xb}(\Db|\sbo) + f_\xb(\Db) - f_\xb(\Dbo) 
&\geq& 
- \phi_{\xb}(\Db|\sbo) - f_{\xb}(\Dbo)
\geq
-\left\{\Lcal_{\xb}(\Db,\alphabo) + \Lcal_{\xb}(\Dbo,\alphabo)\right\}.
\end{eqnarray*}
When both functions coincide uniformly at radius $t$ (the event $\Ecal_{\mathrm{coincide}}(t)$ holds) \emph{and} at radius zero ($\phi_{\xb}(\Dbo|\sbo) = f_\xb(\Dbo)$, i.e., the event $\Ecal_{\mathrm{coincide}}(0)$ holds), the left hand side is indeed zero. As a result we have, uniformly for all $(\Wb,\vb) \in \Wcal_{\Dbo} \times \Scal^p$:
\begin{eqnarray*}
f_{\xb^{i}}(\Db) - f_{\xb^{i}}(\Dbo) 
&\geq& 
\phi_{\xb}(\Db|\sbo) - \phi_{\xb}(\Dbo|\sbo)  - r_{\xb^{i}},\\
\text{with}\ r_{\xb^{i}} &\defin&  \indicator{ \big[\Ecal^i_\mathrm{coincide}(t) \cap \Ecal^i_\mathrm{coincide}(0)\big]^{c}   }(\omega)  
\cdot \left\{\Lcal_{\xb^{i}}(\Db,\alphabo^{i}) + \Lcal_{\xb^{i}}(\Dbo,\alphabo^{i})\right\}.
\end{eqnarray*}
Averaging the above inequality over a set of $n$ signals, we obtain a similar uniform lower bound for $\Delta F_n$:
\begin{equation}\label{eq:lowerbound_fn}
\Delta F_n(\Wb,\vb,t)  \geq \Delta \Phi_n(\Wb,\vb,t) - r_n.
\end{equation}
where we detail the definition
\begin{equation}\label{eq:residual_lowerbound_fn}
r_n(\omega) \defin \frac1n \sum_{i=1}^{n} \indicator{ [\Ecal^{i}_\mathrm{coincide}(t) \cap \Ecal^i_\mathrm{coincide}(0)]^c   }(\omega) 
\cdot \left\{\Lcal_{\xb^{i}}(\Db,\alphabo^{i}) + \Lcal_{\xb^{i}}(\Dbo,\alphabo^{i})\right\}.
\end{equation}
Using Lemma~\ref{lem:suprxi} and Corollary~\ref{cor:concentration_subgaussian} in the Appendix, one can show that with high probability:
\begin{eqnarray*}
r_n \leq 
25 \left( t^{2} \cdot \sigma_{\alpha}^{2} + 2m \cdot \sigma^{2} + 2\lambda k\sigma_{\alpha}\right)(1+\log 2) \cdot (3-\log \kappa)\kappa
\end{eqnarray*}
with $\kappa \defin \max_{i\in\SET{n}}\Pr(\big[\Ecal^{i}_\mathrm{coincide}(t) \cap \Ecal^i_\mathrm{coincide}(0)\big]^{c})$. 
To bound the size of the residual $r_{n}$, we now control $\kappa$.

\subsubsection{Control of $\kappa$: exact sign recovery for \emph{perturbed} dictionaries}\label{sec:recovery}
The objective of this section is to determine sufficient conditions under which
$\phi_\xb(\Db(\Wb,\vb,t)|\sign(\alphabo))$ and $f_\xb(\Db(\Wb,\vb,t))$ coincide for all $(\Wb,\vb)$, and control the probability of this event.
We make this statement precise in the following proposition, proved in Appendix~\ref{app:robustsignrecovery}. 
\begin{proposition}[Exact recovery for perturbed dictionaries and one training sample]\label{prop:exact_recovery}
Condider $\Dbo$ a dictionary in $\Real^{m \times p}$ and let $k,t$ such that $k\mu(t)<1/2$. 
Let $\loweralpha,\sigma_{\alpha},\sigma$ be the remaining parameters of our signal model, 
and let $\xb \in \Real^m$ be generated according to this model. Assume that the regularization parameter $\lambda$ satisfies
\begin{eqnarray*}
0 < \lambda \leq \frac 49 \loweralpha.
\end{eqnarray*}
Consider $0 \leq t' \leq t$. Except with probability at most
\[
\Pr(\Ecal^{c}_{\mathrm{coincide}}(t')) \leq 
2 \cdot
\exp\left(-\frac{\lambda^{2} (2-Q_{t}^{2})^{2}}{5 (t'^{2} \cdot \sigma_{\alpha}^{2} + m\sigma^{2})}\right)
\]
we have, uniformly for all $(\Wb,\vb) \in \Wcal_{\Dbo} \times \Scal^p$, the vector $\hat{\alphab}(t')\in\Real^p$ defined by
\[
\hat{\alphab}(t')=\binom{ \big[ [\Db(t')]_\J^\top[\Db(t')]_\J \big]^{-1} \big[ [\Db(t')]_\J^\top \xb -\lambda \sign([\alphabo]_\J) \big] }{\zerob},
\]
is the unique solution of $\ \min_{\alphab \in \Real^p} [\frac{1}{2}\|\xb-\Db(t')\alphab\|_2^2+\lambda\|\alphab\|_1]$,
and $\sign( \hat{\alphab}(t') ) = \sign(\alphabo )$. 
\end{proposition}
We also need a modified version of this proposition to handle a simplified, noiseless setting where the coefficients $\alphabo$ are 
almost surely upper bounded.  Its proof can be found in Section~\ref{app:robustsignrecovery} as well.
\begin{proposition}[Exact recovery for perturbed dictionaries and one training sample; noiseless and bounded $\alphabo$]\label{prop:simplified_exact_recovery}
Condider $\Dbo$ a dictionary in $\Real^{m \times p}$ and let $k,t$ such that $k\mu(t)<1/2$. 
Consider our signal model with the following additional assumptions:
\begin{eqnarray*}
&\sigma = 0\ &(\textbf{\text{no noise}})\\
&\Pr(|[\alphabo]_j| >  \upperalpha | j \in J) = 0,\quad \text{for some}\ \upperalpha\geq\loweralpha>0 &(\textbf{\text{signal boundedness}}).
\end{eqnarray*}
Let $\xb \in \Real^m$ be generated according to this model. Assume that the regularization parameter $\lambda$ satisfies
\begin{eqnarray*}
\frac{\sqrt{k} \upperalpha}{2-Q_{t}^{2}} t < \lambda \leq \frac 49 \loweralpha.
\end{eqnarray*}
Consider $0 \leq t' \leq t$. Almost surely, 
we have, uniformly for all $(\Wb,\vb) \in \Wcal_{\Dbo} \times \Scal^p$, the vector $\hat{\alphab}(t')\in\Real^p$ defined by
\[
\hat{\alphab}(t')=\binom{ \big[ [\Db(t')]_\J^\top[\Db(t')]_\J \big]^{-1} \big[ [\Db(t')]_\J^\top \xb -\lambda \sign([\alphabo]_\J) \big] }{\zerob},
\]
is the unique solution of $\ \min_{\alphab \in \Real^p} [\frac{1}{2}\|\xb-\Db(t')\alphab\|_2^2+\lambda\|\alphab\|_1]$,
and $\sign( \hat{\alphab}(t') ) = \sign(\alphabo )$.
In other words,  it holds that $\Pr(\Ecal^{c}_{\mathrm{coincide}}(t'))=0$.
\end{proposition}

\subsubsection{Control of the residual}

The last step of the proof of Theorem~\ref{th:control_Fn} consists in controlling the residual term~(\ref{eq:residual_lowerbound_fn}).
Its proof can be found in Section~\ref{sec:proof_control_residual}.

\begin{proposition}
\label{prop:control_rn}
Let $\loweralpha,\sigma_{\alpha}$ be the parameters of the coefficient model.  
Consider $\Dbo$ a dictionary in $\Real^{m \times p}$ with $\mu_{0} < 1/2$ and let $k,t$ be such that 
\begin{eqnarray}
\label{eq:AssumptionKT1}
k\mu(t) &<& 1/2\\
\frac{3t}{2-Q_{t}^{2}}  &< & 
\label{eq:AssumptionKT2}
\frac{4\loweralpha}{9\sigma_{\alpha}}
\end{eqnarray}
Then for small enough noise levels $\sigma$ one can find a regularization parameter $\lambda>0$ such that
\begin{equation}
\label{eq:BoundsLambda}
\frac{3}{2-Q_{t}^{2}}
\cdot \sqrt{t^{2} \sigma_{\alpha}^{2} + m\sigma^{2}}
\leq  \lambda \leq \frac 49 \loweralpha.
\end{equation}
Given $\sigma$ and $\lambda$ satisfying~\eqref{eq:BoundsLambda},  we define
\begin{equation}
\gamma \defin \frac{\lambda(2-Q_{t}^{2})}{\sqrt{5} \cdot \sqrt{t^{2}  \sigma_{\alpha}^{2} + m \sigma^{2}}} \geq \sqrt{2\log2}.
\end{equation}
Let $\xb^{i} \in \Real^m$, $i \in \SET{n}$ be generated according to our noisy signal model. Then, 
\begin{equation}
r_{n} \leq
\left(t^{2} \sigma_{\alpha}^{2} + 2m \sigma^{2} + 2\lambda k\sigma_{\alpha}\right)
\cdot 367 \cdot \gamma^{2} \cdot e^{-\gamma^{2}}.
\end{equation}
except with probability at most $\exp(-4n \cdot e^{-\gamma^{2}})$.
\end{proposition}
We have stated the main results and showed how they are structured in key propositions, 
which we now prove.
\appendix

\section{Proof of Proposition~1}\label{app:pflocalmin}
The topology we consider on $\Dcal$ is the one induced by its natural embedding in $\Real^{m \times p}$: the open sets are the intersection of open sets of $\Real^{m \times p}$ with $\Dcal$. Recall that all norms are equivalent on $\Real^{m \times p}$ and induce the same topology. For convenience we will consider the balls associated to the Froebenius norm.
To prove the existence of a local minimum for $F_n$, say at $\Db^\star$,
we will show the existence of a ball centered at $\Db^\star$, 
\(
\Bcal_h \defin \big\{  \Db \in \Dcal;\ \|\Db^\star - \Db\|_\fro \leq h  \big\}
\)
such that for any $\Db \in \Bcal_{h}$, we have $F_n(\Db^\star) \leq F_n(\Db)$.

\paragraph{First step:} We recall the notation $\Scal_+^p \defin \Scal^p \cap \Real_+^p$ for the sphere intersected with 
the positive orthant. Moreover, we introduce 
$$
\Zcal_t \defin \Big\{  \Db(\Wb,\vb,t');\, \Wb\in \Wcal_{\Dbo}, \vb \in \Scal^p_+, t'\in[0,t],\ 
\text{and}\ t'\|\vb\|_\infty \leq \pi   \Big\}.
$$
The set $\Zcal_t$ is compact as the image of a compact set by the continuous function 
$(\Wb,\vb,t') \mapsto \Db(\Wb,\vb,t')$.
As a result, the continuous function $F_n$ admits a global minimum in $\Zcal_t$ which we denote by 
$
\Db^\star = \Db(\Wb^\star,\vb^\star,t^\star).
$
Moreover, and according to the assumption of Proposition~1, we  have $t^{\star} < t$.

\paragraph{Second step:} 
We will now  prove the existence of $h >0 $ such that $\Bcal_h \subseteq \Zcal_t$. This will imply that $F_{n}(\Db^{\star}) \leq F_{n}(\Db)$ for $\Db \in \Bcal_{h}$, hence that $\Db^{\star}$ is a local mimimum. 
First, we formalize the following lemma.
\begin{lemma}
\label{lem:paramonto} Given any matrix $\Db_{1} \in \Dcal$, any matrix $\Db_{2} \in \Dcal$ can be described as $\Db_{2} = \Db(\Db_{1},\Wb,\vb,\tau)$, with $\Wb \in \Wcal_{\Db_{1}}$, $\vb \in \Scal_{+}^{p}$ and $\tau \geq 0$ such that $\tau \|\vb\|_{\infty} \leq \pi$. Moreover, we have 
\begin{eqnarray}
\frac{2}{\pi}\tau \vb_{j} & \leq & \|\db^{j}_{2}-\db^{j}_{1}\|_2 = 2 \sin \left(\frac{\tau \vb_{j}}{2}\right) \leq \tau \vb_{j},\quad \forall j,\\
\frac{2}{\pi}\tau & \leq & \|\Db_{2}-\Db_{1}\|_{\fro} \leq \tau. 
\end{eqnarray}
Vice-versa, $\Db_{1} = \Db(\Db_{2},\Wb',\vb',\tau')$ for some $\Wb' \in \Wcal_{\Db_{2}}$, and \emph{with the same} $\vb'=\vb \in \Scal_{+}^{p}$, $\tau'=\tau \geq 0$.
\end{lemma}
\begin{proof}
The result is trivial if $\Db_{2} = \Db_{1}$, hence we focus on the case $\Db_{2} \neq \Db_{1}$. Each column $\db^j_{2}$ of $\Db_{2}$ can be uniquely expressed as
\[
\db^j_{2} = \ub + \zb,\ \text{with}\ \ub \in \text{span}(\db^j_{1})\ \text{and}\ \ub^\top \zb = 0.
\]
Since $\|\db^j_{2}\|_2=1$, the previous relation can be rewritten as 
\[
\db^j_{2} = \cos(\theta_j)\db^j_{1} + \sin(\theta_j) \wb^j, 
\]
for some $\theta_j \in [0,\pi]$ and some unit vector $\wb^j$ orthogonal to $\db^{j}_{1}$ (except for the case $\theta_j=0$, the vector $\wb^j$ is unique). The sign indetermination in $\wb^j$ is handled thanks to the convention $\sin(\theta_j)\geq 0$. One can define a matrix $\Wb \in \Wcal_{\Db_{1}}$ which $j$-th column is $\wb^{j}$. Denote $\thetab \defin (\theta_1,\dots,\theta_p)$ and $\tau \defin \|\thetab\|_2$. Since $\Db_{2} \neq \Db_{1}$ we have $\tau>0$ 
and we can define $\vb \in \Scal_+^p$ with coordinates
\[
\vb_j \defin \frac{\theta_j}{\tau}.
\]
%We have thus built a triplet $(\Wb,\vb,\tau)$ such that $\Db' =  \Db(\Db,\Wb,\vb,\tau)$. 
Next we notice that $\tau \|\vb\|_\infty = \|\thetab\|_{\infty} \leq \pi$ and
\begin{eqnarray*}
\|\db^{j}_{2} - \db^{j}_{1}\|_2^{2}  &=&  \| (1-\cos(\vb_j \tau))\db^j - \sin(\vb_j \tau)\wb^{j} \|_2^2
= ( 1 -  \cos(\vb_j \tau) )^{2} + \sin^{2}(\vb_j \tau)\\
&=& 2(1-\cos (\vb_{j} \tau)) = 4 \sin^{2} (\vb_{j}\tau/2).
\end{eqnarray*}
We conclude using the inequalities $\frac{2}{\pi} \leq \frac{\sin u}{u} \leq 1$ for $0 \leq u \leq \pi/2$ and the fact that $\|\vb\|_{2}=1$. The reciprocal $\Db_{1} = \Db(\Db_{2},\Wb',\vb',\tau')$ is obvious, and the fact that $\vb' = \vb$, $\tau'=\tau$ follows from the equality $\|\db_{1}^{j}-\db_{2}^{j}\|_{2} = 2\sin \vb_{j}\tau = 2 \sin \vb'_{j} \tau'$ for all $j$.
\end{proof}
Using the parameterization built in Lemma~\ref{lem:paramonto} for 
$\Db \in \Bcal_{h}$, there remains to prove that $\Db = \Db(\Db_{0},\Wb,\vb,\tau)$ belongs to $\Zcal_t$ provided that $h$ is small enough.
For that, we need to show that $\tau < t$ (we will need of course to assume that $h$ is small enough).
To this end, notice that 
\begin{eqnarray*}
\|\Db^\star - \Db\|_\fro^2  &=& 
\sum_{j=1}^p \| (\cos(\vb_j^\star t^\star)-\cos(\vb_j \tau))\dbo^j + 
\sin(\vb_j^\star t^\star)\wb^{\star,j} - \sin(\vb_j \tau)\wb^{j} \|_2^2\\
&=& 2\sum_{j=1}^p ( 1 -  \cos(\vb_j^\star t^\star) \cos(\vb_j \tau) -
\sin(\vb_j^\star t^\star)\sin(\vb_j \tau) [\wb^{j}]^\top \wb^{\star,j} )\\
\end{eqnarray*}
where the simplifications in the second equality come from the fact that both $\Wb$ and $\Wb^\star$
have their columns normalized and orthogonal to the corresponding columns of $\Dbo$.
Since $t^\star \|\vb^\star\|_\infty \leq \pi$ and $\tau \|\vb \|_\infty \leq \pi$, 
the product of sine terms is positive, so that with $| [\wb^{j}]^\top \wb^{\star,j}  | \leq 1$, we obtain
\[
\|\Db^\star - \Db\|_\fro^2 \geq 2\sum_{j=1}^p ( 1 -  \cos(\vb_j^\star t^\star) \cos(\vb_j \tau) -
\sin(\vb_j^\star t^\star)\sin(\vb_j \tau)  ) = 2\sum_{j=1}^p ( 1 - \cos(\Delta_j))=4\sum_{j=1}^p \sin^2(\Delta_j/2)
\]
where $\Delta_j \defin \vb_j^\star t^\star - \vb_j \tau$. Now, since $0 \leq t^{\star}\vb^{\star}_{j},\tau \vb_{j} \leq \pi$, we have $\Delta_{j}/2 \in [-\pi/2\ ,\ \pi/2]$, hence using that $\sin^2(u) \geq \frac{4}{\pi^2} u^2$ for $|u| \leq \frac{\pi}{2}$, we finally have
\begin{eqnarray*}
h^2 \geq \|\Db^\star - \Db\|_\fro^2 &\geq& \frac{4}{\pi^2} \sum_{j=1}^p \Delta_j^2
= \frac{4}{\pi^2} \sum_{j=1}^p (  [\vb_j^\star t^\star]^2 + [\vb_j \tau]^2 - 2 t^\star \tau  \vb^\star_j \vb_j) \geq \frac{4}{\pi^2} ( t^\star -\tau)^2,
\end{eqnarray*}
where we have exploited that both $\vb^{\star}$ and $\vb$ are normalized.
As a consequence, we have $\tau \leq t^{\star}+\frac{\pi}{2} h$ hence for $h < \frac{2}{\pi}(t-t^{\star})$ we guarantee  $\tau < t$, so that $\Db \in \Zcal_t$. We conclude that $\Bcal_h \subseteq \Zcal_t$ for $h < \frac{2}{\pi}(t-t^{\star})$.

\paragraph{Third and last step:}
To recapitulate, we have shown that there exists a ball $\Bcal_h$ in $\Dcal$, 
such that $\Bcal_h \subseteq \Zcal_t$ and for any $\Db \in \Bcal_h$, we have
$$
F_n(\Db) \geq F_n(\Db^\star),
$$
since the previous inequality is true over the entire set $\Zcal_t$.
We can finally observe using Lemma~\ref{lem:paramonto} that 
\begin{eqnarray*}
\|\Dbo - \Db^\star\|_\fro^2 = 2\sum_{j=1}^p \|\db^{\star,j}-\dbo\|_{2}^{2} \leq \sum_{j=1}^p [\vb_j^\star t^\star]^2 \leq [t^\star]^2 < t^2,
\end{eqnarray*}
which leads to the advertised conclusion.
\section{Proof of Theorem~\ref{th:minimum_sparsecoding} and~\ref{th:minimum_sparsecoding_simple} }\label{app:minimum_sparsecoding}
We start with the more general theorem:
\subsection{Proof of Theorem~\ref{th:minimum_sparsecoding} }

We recall that we assume in Theorem~\ref{th:control_Fn} that
$
c_\lambda\cdot t  <
\frac{4\loweralpha}{9\sigma_{\alpha}}
$
and for small enough noise levels $\sigma$ one can find a regularization parameter $\lambda>0$ such that
\begin{equation*}
c_\lambda
\cdot \sqrt{t^{2} \sigma_{\alpha}^{2} + m\sigma^{2}}
\leq  \lambda \leq \frac 49 \loweralpha.
\end{equation*}
Given such $\sigma$ and $\lambda$, we define
$$
\gamma \defin  \frac{\lambda}{ c_\gamma  \sqrt{t^2 \sigma_\alpha^2 + m\sigma^2 } } \geq \sqrt{2\log(2)}.
$$
Here, $c_\lambda$ and $c_\gamma \defin \frac{\sqrt{5}}{3} c_\lambda$ stand for some universal constants which can be made explicit thanks to Theorem~\ref{th:control_Fn}.
\paragraph{Goal:} 
To determine when the lower bound proved in Theorem~\ref{th:control_Fn} is stricly positive, it is sufficient to consider when it holds that 
\begin{eqnarray*}
\Exp[\alpha^2]\cdot \frac{k}{p} \cdot t^2 &-& 
c_0 \lambda \cdot \Exp[|\alpha|]\cdot \frac{k}{p} \cdot t \cdot \triple \Dbo \triple_2 \cdot k \mu(t)\\
&-& c_1 (t^2\sigma_\alpha^2 + 2m\sigma^2 + 2 \lambda k \sigma_\alpha^2 )\cdot\gamma^2 e^{-\gamma^2}\\
&-& c_2 (t k \sigma_\alpha^2 + 2m\sigma^2 + 2\lambda k \sigma_\alpha^2 ) \cdot \Lambda_n\quad\text{with}\quad
\Lambda_n \defin \Big[ mp \frac{\log(n)}{n} \Big]^{\frac{1}{2}}\\
&\geq& t\cdot (-a_2 t^2 + a_1 t - a_0) > 0,
\end{eqnarray*}
for some universal constants $c_j$ which we can make explicit based on Theorem~\ref{th:control_Fn}, but which we keep hidden for clarity.

\paragraph{Probability of success:} The probability of success of Theorem~\ref{th:control_Fn} is given by
$$
1 - \Big(\frac{mpn}{9}\Big)^{-mp/2} - \exp(-4ne^{-\gamma^2}).
$$
This induces a first condition over $\gamma$ (a upperbound), namely 
\begin{equation*}
ne^{-\gamma^2} \geq \epsilon_n\ \Rightarrow\ \gamma^2 \leq \log(n) - \log(\epsilon_n),\ \text{for some}\ 
\epsilon_n \rightarrow \infty.
\end{equation*}
From now on, we make the choice 
$\epsilon_n = \sqrt{n}$, so that $\exp(-4ne^{-\gamma^2}) \leq \exp(-4\sqrt{n})$, 
along with the condition 
\begin{equation}\label{eq:gamma_n}
 \gamma^2 \leq \frac{1}{2}\log(n).
\end{equation}.

\paragraph{Noiseless/low-noise regime:}
Even though they are conceptually two different regimes, 
the treatment of the noisy and noiseless regimes follow the very same reasoning. From now on, we therefore assume that 
\begin{equation}\label{eq:noise}
m \sigma^2 \leq t^2 \sigma_\alpha^2,
\end{equation}
which determines the upper level of noise we will be able to handle. 

\paragraph{Second-order polynomial function in $t$:}
By simply using~(\ref{eq:noise}), $\lambda \leq \sqrt{2} c_\gamma \cdot \gamma \cdot \sigma_\alpha t$ and 
$3+2\sqrt{2} c_\gamma \leq 4\sqrt{2} c_\gamma$,
we now make explicit the $a_j$, $j \in \{0,1,2\}$, which define the second-order polynomial function in $t$:
\begin{eqnarray*}
a_2 & \defin & 3\sqrt{2} c_0 c_\gamma \cdot \sigma_\alpha\Exp[|\alpha|] \cdot \frac{k}{p} \cdot \triple \Dbo \triple_2 \cdot k \cdot \gamma\\
a_1 & \defin & \Exp[\alpha^2]\cdot \frac{k}{p} 
- \sqrt{2} c_0 c_\gamma \cdot \sigma_\alpha \Exp[|\alpha|]\cdot \frac{k}{p} \cdot \triple \Dbo \triple_2 \cdot k\mu_0 \cdot \gamma 
- 3c_1 \sigma_\alpha^2 \cdot \gamma^2 e^{-\gamma^2}\\
a_0 &\defin& 2\sqrt{2} c_\gamma \cdot k \sigma_\alpha^2\cdot[ c_1 \gamma^3 e^{-\gamma^2} + 2c_2\cdot \gamma\cdot \Lambda_n].
\end{eqnarray*}
We will make use of the following simple lemma to discuss the sign of this polynomial function:
\begin{lemma}\label{lem:poly}
Let $(a_0,a_1,a_2) \in \Real_+^3$.
If $4 a_0 a_2 < a_1^2$, and
$
t \in  \big[ \frac{2a_0}{a_1}, \frac{a_1}{2a_2} \big],
$
then 
$-a_2 t^2 + a_1 t - a_0 > 0$.
\end{lemma}
\paragraph{Some key definitions:}
Let $\theta$ be defined as 
$$
\theta \defin \min\bigg\{ \frac{\loweralpha}{\sigma_\alpha}, \frac{1}{3c_0} \cdot \frac{\Qcal_\alpha}{k \cdot \triple \Dbo \triple_2} \bigg\}.
$$
We also define $\gamma_{\min} > 1$ the unique number such that 
\begin{equation}\label{eq:gamma_min}
\gamma_{\min}^4 e^{-\gamma_{\min}^2} \defin \frac{q_\alpha}{69 c_1 c_\gamma^2} 
\cdot \frac{1}{p} \cdot \theta,  
\end{equation}
and
\begin{equation}\label{eq:gamma_max}
\gamma_{\max} \defin \frac{1}{2}
\min\bigg\{\sqrt{2\log(n)},
\frac{1}{2\sqrt{2} c_0 c_\gamma} \cdot
\frac{\Qcal_\alpha}{ k \mu_0\cdot \triple \Dbo \triple_2}\bigg\}.
\end{equation}
Moreover, we consider
\begin{equation}\label{eq:Lambda_max}
\Lambda_{n,\max} \defin \frac{q_\alpha}{138 c_2 c_\gamma^2} 
\cdot \frac{1}{p \cdot \gamma^2} \cdot \theta,  
\end{equation}

\paragraph{First step, non-emptiness of $[\gamma_{\min},\gamma_{\max}]$:}
We first check that the interval $[\gamma_{\min},\gamma_{\max}]$ is not empty.
On the one hand, if the value of $\gamma_{\max}$ is obtained by $\sqrt{1/2 \log(n)}$, 
we use the fact that $\gamma_{\min} < \gamma_{\max}$ is equivalent to $\gamma_{\min}^4 e^{-\gamma_{\min}^2} >  \gamma_{\max}^4 e^{-\gamma_{\max}^2}$.
In particular, we have
$$
 \gamma_{\max}^4 e^{-\gamma_{\max}^2} = \frac{\log^2(n)}{4\sqrt{n}} < \gamma_{\min}^4 e^{-\gamma_{\min}^2}, 
$$
a condition that will be implied by the more stringent condition $\Lambda_n \leq \Lambda_{n,\max}$.

On the other hand, and in the second scenario for $\gamma_{\max}$, we conclude based on the following lemma:
\begin{lemma}\label{lem:ineq_lambert}
Let $a>1$ and $b \in (0,1/5]$.
If 
$
a^4 e^{-a^2} = b 
$,
then 
$ 
\sqrt{\log(1/b)} \leq a \leq 2\sqrt{\log(1/b)}.
$
\end{lemma}
The sufficient condition which stems form this lemma reads
$$
k \mu_0\cdot \triple \Dbo \triple_2 \leq \frac{1}{4\sqrt{2} c_0 c_\gamma}\cdot \frac{\Qcal_\alpha}
{\sqrt{\log\big( 69 c_1 c_\gamma^2 \cdot \frac{1}{q_\alpha} 
\cdot \frac{p}{\theta}  \big)}}.
$$
\paragraph{Second step, lower bound on $a_1$:}
For any $\gamma \in [\gamma_{\min},\gamma_{\max}]$, it is first easy to check that 
$$
\sqrt{2} c_0 c_\gamma \cdot \sigma_\alpha \Exp[|\alpha|]\cdot \frac{k}{p} \cdot \triple \Dbo \triple_2 \cdot k\mu_0 \cdot \gamma 
\leq \frac{1}{4} \Exp[\alpha^2]\cdot \frac{k}{p}.
$$
Moreover, since $\gamma^2 e^{-\gamma^2} \leq \gamma^4 e^{-\gamma^2}$ and 
$
\frac{1}{12 c_1} q_\alpha \cdot \frac{k}{p} > \gamma_{\min}^4 e^{-\gamma_{\min}^2} , 
$
we therefore obtain that 
$$
a_1 \geq \Exp[\alpha^2]\cdot \frac{k}{p} - \frac{1}{4}\cdot \Exp[\alpha^2]\cdot \frac{k}{p} - \frac{1}{4}\cdot \Exp[\alpha^2]\cdot \frac{k}{p}
\geq \frac{1}{2} \cdot \Exp[\alpha^2]\cdot \frac{k}{p}.
$$
\paragraph{Third step, the condition $4a_0a_2 < a_1^2$:}
Since we have $a_1 > \frac{1}{2} \cdot \Exp[\alpha^2]\cdot \frac{k}{p}$, and 
$$
a_2 \leq 4\sqrt{2} c_\gamma \cdot k \sigma_\alpha^2\cdot \max\Big\{ c_1 \gamma^3 e^{-\gamma^2}, 2c_2\cdot \gamma\cdot \Lambda_n\Big\},
$$
simple computations show that 
$
\gamma \geq \gamma_{\min}
$
and 
$
\Lambda_n \leq \Lambda_{n, \max}
$, as defined in~(\ref{eq:gamma_min}) and (\ref{eq:Lambda_max}), lead to $4a_0a_2 < a_1^2$.
\paragraph{Conclusions:}
We have proved that for $\gamma \in [\gamma_{\min},\gamma_{\max}]$, 
$
\Lambda_n \leq \Lambda_{n, \max}
$,
and
$$
k \mu_0\cdot \triple \Dbo \triple_2 \leq \frac{1}{4\sqrt{2} c_0 c_\gamma}\cdot \frac{\Qcal_\alpha}
{\sqrt{\log\big( 69 c_1 c_\gamma^2 \cdot \frac{1}{q_\alpha} 
\cdot \frac{p}{\theta}  \big)}},
$$
the lower bound provided by Theorem~\ref{th:control_Fn} is stricly positive 
for a radius $t \in \big[ \frac{2a_0}{a_1}, \frac{a_1}{2a_2} \big]$ (see Lemma~\ref{lem:poly})
and a noise $\sigma \leq \sigma_\alpha \sqrt{m} t$.
Taking the smallest allowed radius (i.e., $t=\frac{2a_0}{a_1}$ with $\gamma = \gamma_{\max}$) leads to the displayed result.
\subsection{Proof of Theorem~\ref{th:minimum_sparsecoding_simple} }
We now discuss the version of Theorem~\ref{th:minimum_sparsecoding} in the simpler setting where there is no noise (i.e., $\sigma=0$)
and $\alphabo$ is almost surely bounded by $\upperalpha \geq \loweralpha > 0$.
The main consequence of these simplifying assumptions is that there is no residual term to consider anymore and our surrogate function 
coincide almost surely with the true sparse coding function, provided the radius $t$ is small enough,
as proved in Proposition~\ref{prop:simplified_exact_recovery}.
As a result, the terms depending on $\gamma$ in Theorem~\ref{th:control_Fn} disappear, and the probability of success simplifies to
$$
1 - \Big(\frac{mpn}{9}\Big)^{-mp/2}.
$$
Moreover, in light of Proposition~\ref{prop:simplified_exact_recovery}, we now ask for
\begin{equation*}
\frac{1}{3} c_\lambda \sqrt{k} \upperalpha t \leq \lambda \leq \frac 49 \loweralpha. 
\end{equation*}
The backbone of the proof remains identical, we
adapt the discussion about the polynomial function in $t$.
\paragraph{Goal:} 
To determine when the lower bound proved in Theorem~\ref{th:control_Fn} is stricly positive, it is sufficient to consider when it holds that 
\begin{eqnarray*}
\Exp[\alpha^2]\cdot \frac{k}{p} \cdot t^2 &-& 
c_0 \lambda \cdot \Exp[|\alpha|]\cdot \frac{k}{p} \cdot t \cdot \triple \Dbo \triple_2 \cdot k \mu(t)\\
&-& c_1 (t k \sigma_\alpha^2 + 2\lambda k \sigma_\alpha^2 ) \cdot \Lambda_n\quad\text{with}\quad
\Lambda_n \defin \Big[ mp \frac{\log(n)}{n} \Big]^{\frac{1}{2}}\\
&\geq& t\cdot (-a_2 t^2 + a_1 t - a_0) > 0,
\end{eqnarray*}
for some universal constants $c_j$ which we can make explicit based on Theorem~\ref{th:control_Fn}, but which we keep hidden for clarity.
\paragraph{Second-order polynomial function in $t$:}
By making the choice $\lambda \defin \frac{1}{2} c_\lambda \cdot \upperalpha \cdot \sqrt{k} \cdot t$,
we now make explicit the $a_j$, $j \in \{0,1,2\}$, which define the second-order polynomial function in $t$:
\begin{eqnarray*}
a_2 & \defin & \frac{3}{2} c_0 c_\lambda \cdot \upperalpha\cdot \Exp[|\alpha|] \cdot \frac{k}{p} \cdot \triple \Dbo \triple_2 \cdot k^{3/2} \\
a_1 & \defin & \Exp[\alpha^2]\cdot \frac{k}{p} 
- \frac{1}{2} c_0 c_\lambda \cdot \upperalpha\cdot \Exp[|\alpha|]\cdot \frac{k}{p} \cdot \triple \Dbo \triple_2 \cdot k^{3/2} \mu_0 \\
a_0 &\defin& 2 c_1 c_\lambda \cdot k^{3/2} \sigma_\alpha \upperalpha \cdot \Lambda_n.
\end{eqnarray*}
\paragraph{Conclusions:}
Consider the condition
$$
\triple \Dbo \triple_2\cdot k^{3/2}\mu_0 
\leq 
\frac{1}{c_0 c_\lambda} \frac{\Exp[\alpha^2]}{\upperalpha\cdot \sigma_\alpha},
$$
so that $a_1 \geq \frac{1}{2} \cdot \Exp[\alpha^2]\cdot \frac{k}{p}$.
By using again Lemma~\ref{lem:poly}, and by defining 
$$
\Lambda_{n,\max} \defin \frac{1}{9 c_1 c_\lambda^2} \cdot \frac{\Exp[\alpha^2]}{\upperalpha^2}
\cdot \frac{1}{k p} \cdot 
\min\bigg\{ \frac{\loweralpha}{\sigma_\alpha}, \frac{1}{5c_0} \cdot \frac{\Qcal_\alpha}{k \cdot \triple \Dbo \triple_2} \bigg\},
$$
it is easy to check that $\Lambda_n \leq \Lambda_{n,\max} $ implies that 
$4 a_0 a_2 < a_1^2$ along with 
$$
2 \frac{a_0}{a_1} < \frac{8}{9 c_\lambda} \frac{\loweralpha}{\upperalpha} \cdot \frac{1}{\sqrt{k}},
$$
as required by our choice of $\lambda$ and the fact that $\lambda \leq \frac{4}{9} \loweralpha$.
\section{Uniform restricted isometry and coherence properties}\label{app:RIP} 

First, we introduce $\PJb(t) \in \Real^{ m \times m  }$ the orthogonal projector which projects onto the span of $[\Db(t)]_\J$ and establish a result that holds without any assumption on $\Dbo$.
\begin{lemma}\label{lem:orthoprojcomplement}
For any $\Wb \in \Wcal_{\Dbo},\vb \in \Scal^{p}$, $t \geq 0$ and $\J$, 
\begin{eqnarray}
\label{eq:BoundDeltaDJ}
\triple [\Db(t)-\Dbo]_\J \triple_{2}^2
\leq 
\| [\Db(t)-\Dbo]_\J \|_\fro^2 
&\leq& t^2 \cdot \|\vb_{\J}\|_{2}^{2}\\
\label{eq:IdMinusPJDbo}
\triple  (\Ib - \PJb(t)) [\Dbo]_\J \triple_2^{2} 
\leq
\|  (\Ib - \PJb(t)) [\Dbo]_\J \|_\fro^{2} 
&\leq& t^{2} \cdot \|\vb_{\J}\|_{2}^{2}.
\end{eqnarray}
\end{lemma}
\begin{proof}
For the first result we observe
\[
\| [\Db(t)-\Dbo]_\J \|_\fro^2 = \sum_{j\in\J} \|\dbo^j(t)-\dbo^{j}\|^2_2 = 4\sum_{j\in\J}\sin^{2}(\vb_j t/2) \leq 4\sum_{j\in\J} \vb_j^2 \frac{t^2}{4} \leq t^2 \cdot \|\vb_{\J}\|_{2}^{2}.
\]
For the second one, using Lemma~\ref{lem:paramonto} with $\Db_{1} = \Db(t) = \Db(\Dbo,\Wb,\vb,t)$, $\Db_{2} = \Db_{0}$, there exists $\Wb' \in \Wcal_{\Db(t)}$ such that for each $j$, $\dbo^{j} = \db^{j}(t) \cos (\vb_{j}t)+\wb'^{j}\sin(\vb_{j}t)$. Hence, denoting $\Cb = \Diag(\cos(\vb_{j} t))$ and $\Sb = \Diag(\sin(\vb_{j} t))$ we have 
$[\Dbo]_{\J} = [\Db(t)\Cb]_{\J} + [\Wb'\Sb]_{\J}$. 
Each column of $[\Db(t)\Cb]_\J$ belongs to the span of the columns of $[\Db(t)]_\J$, so that 
\begin{equation}
\label{eq:WTrick}
(\Ib - \PJb(t)) [\Dbo]_\J = (\Ib - \PJb(t)) [\Wb'\Sb]_\J.
\end{equation}
As a result,
\begin{eqnarray*}
\|  (\Ib - \PJb(t)) [\Dbo]_\J \|_\fro^{2} 
&=& \| (\Ib - \PJb(t)) [\Wb'\Sb]_\J \|_\fro^{2}
\leq \| [\Wb'\Sb]_\J \|_\fro^{2} = \sum_{j \in \J}\sin^{2}(\vb_{j}t) \leq \|\vb_{\J}\|_{2}^{2} \cdot t^{2}.
\end{eqnarray*}
\end{proof}

Next, we control the norms of $\ThetaJb(t') \defin \big[ \Db_\J^\top(t') \Db_\J(t') \big]^{-1}$ when this is a well-defined matrix. For that, we first recall the definition of the restricted isometry constant of order $k$ of a dictionary $\Db$, $\delta_{k}(\Db)$, as the smallest number $\delta_{k}$ such that for any support set $\J$ of size $|\J| = k$ and $\zb \in \Real^k$,
\begin{equation}
\label{eq:DefRICk}
\left(1-\delta_{k}\right) \|\zb\|_{2}^{2} \leq \|\Db\zb\|_{2}^{2} \leq \left(1+\delta_{k}\right) \|\zb\|_{2}^{2}.
\end{equation}
\begin{lemma}
\label{lem:RIPBounds}
Let $\Dbo \in \Real^{m\times p}$ be a dictionary and $k$ such that $\delta_{k}(\Dbo)<1$. For any $t < \sqrt{1-\delta_{k}(\Dbo)}$ define
\begin{equation}
C_{t} \defin \frac{1}{\sqrt{1-\delta_{k}(\Dbo)}-t}.
\end{equation}
For any $\Wb \in \Wcal_{\Dbo}$, $\vb \in \Scal^{p}$, $0 \leq t' \leq t$ and $\J$ of size $k$, the $\J \times \J$ matrix
\begin{equation}
\label{eq:DefThetaJ}
\ThetaJb(t') \defin \big[ \Db_\J^\top(t') \Db_\J(t') \big]^{-1}
\end{equation}
is well defined and we have 
\begin{eqnarray}
\label{eq:BoundDJ}
\triple \Db_{\J}(t')\triple_{2}  = \triple \Db_{\J}^{\top}(t')\triple_{2}
& \leq & C_{t}\\
\label{eq:BoundThetaJ}
\triple\ThetaJb(t')\triple_{2} 
& \leq & C_{t}^{2}\\
\label{eq:BoundDThetaJ}
\triple \Db_{\J}(t')\ThetaJb(t') \triple_2 
&\leq& C_{t}.
\end{eqnarray}
\end{lemma}

\begin{proof}
By the triangle inequality and Lemma~\ref{lem:orthoprojcomplement}-Equation~\eqref{eq:BoundDeltaDJ}, for any $\J$ of size $k$ and $\zb \in \Real^{k}$ we have
\begin{eqnarray*}
\|\Db_{\J}(t')\zb\|_{2} 
& \geq & \|[\Dbo]_{\J} \zb\|_{2} - \|[\Db(t')-\Dbo]_{\J}\zb\|_{2} 
\geq \big(\sqrt{1-\delta_{k}(\Dbo)} - t' \|\vb_{\J}\|_{2}\big) \cdot \|\zb\|_{2}
\geq \big(\sqrt{1-\delta_{k}(\Dbo)} - t\big) \cdot \|\zb\|_{2}\\
\|\Db_{\J}(t)\zb\|_{2} & \leq & \|[\Dbo]_{\J} \zb\|_{2} + \|[\Db(t)-\Dbo]_{\J}\zb\|_{2}
\leq \big(\sqrt{1+\delta_{k}(\Dbo)} + t' \|\vb_{\J}\|_{2}\big) \cdot \|\zb\|_{2}
\leq \big(\sqrt{1+\delta_{k}(\Dbo)} + t\big) \cdot \|\zb\|_{2}.
\end{eqnarray*}
Hence, in the sense of symmetric positive definite matrices
\[
\big(\sqrt{1-\delta_{k}(\Dbo)}-t\big)^{2} \cdot \Ib \preceq \Db_{\J}^{\top}(t')\Db_{\J}(t') \preceq \big(\sqrt{1+\delta_{k}(\Dbo)} + t\big)^{2} \cdot \Ib.
\]
As a result, $\Db_{\J}^{\top}(t')\Db_{\J}(t')$ is invertible so $\ThetaJb(t')$ is indeed well defined, and
\begin{eqnarray*}
\triple \Db_{\J}(t')\triple_{2} 
&=& \triple \Db_{\J}^{\top}(t')\triple_{2} =  \sqrt{\triple\Db_{\J}^\top(t') \Db_{\J}(t') \triple_{2}}
\leq \sqrt{1+\delta_{k}(\Dbo)}+ t \leq \frac{1}{\sqrt{1-\delta_{k}(\Dbo)}-t}\\
\triple\ThetaJb(t')\triple_{2} 
&=&
\triple(\Db_{\J}^\top(t') \Db_{\J}(t'))^{-1} \triple_{2} 
\leq \frac{1}{\big(\sqrt{1-\delta_{k}(\Dbo)}-t\big)^{2}}\\
\triple \Db_{\J}(t')\ThetaJb(t') \triple_2 
&=&  \sqrt{ \triple \ThetaJb(t') \Db_{\J}^\top(t') \Db_{\J}(t')\ThetaJb(t') \triple_2 }
= \sqrt{ \triple \ThetaJb(t')  \triple_2 } \leq \frac{1}{\sqrt{1-\delta_{k}(\Dbo)}-t}.
\end{eqnarray*}
\end{proof}
To continue, we control certain norms of the dictionary when it has low coherence:
\begin{lemma}\label{lem:CoherenceBounds}
Let $\Dbo \in \Real^{m\times p}$ be a dictionary with coherence $\mu$ and normalized columns (i.e., with unit $\ell_2$-norm).
For any $\J \subseteq \SET{p}$ with $|\J| \leq k$,
We have
$$
\triple \Db_{\J}^{\top}\Db_{\J}-\Ib\triple_2 \leq \|\Db_{\J}^{\top}\Db_{\J}-\Ib\|_\fro \leq k \mu,
$$
along with 
$$
\triple\Db_\J\Db_\J^\top \triple_2 =\triple\Db_\J^\top \Db_\J \triple_2 \leq 1+k\mu\quad 
\text{and}\quad \delta_{k}(\Db) \leq k \mu.
$$ 
Similarly, it holds 
$$
\triple\Db_\J^\top \Db_\J \triple_\infty \leq 1+k\mu\quad \text{and}\quad \triple\Db_{\J^c}^\top \Db_\J \triple_\infty \leq k\mu.
$$
Moreover, introduce for any $\Ab \in \Real^{k\times k}$ the matrix norm 
$$
N(\Ab) \defin k\cdot \max_{i,j \in \SET{k}} | \Ab_{i,j} |
$$
and consider
$$
\ThetaJb \defin \big[ \Db_\J^\top \Db_\J \big]^{-1}.
$$
If we further assume $k\mu<1$, then $\ThetaJb$ is well-defined and
$$
\max\Big\{ \triple \ThetaJb-\Ib \triple_\infty, \triple \ThetaJb-\Ib \triple_2, \| \ThetaJb-\Ib \|_\fro , N(\ThetaJb-\Ib) \Big\} \leq \frac{k\mu}{1-k\mu},
$$
along with
$$
\max\Big\{ \triple \ThetaJb \triple_\infty, \triple \ThetaJb \triple_2  \Big\} \leq \frac{1}{1-k\mu}.
$$

\end{lemma}
\begin{proof}
These properties are already well-known \citep[see, e.g.][]{Tropp2004,Fuchs2005}. We briefly prove them.
First, we introduce $\Hb=\Db_\J^\top \Db_\J-\Ib$.
A straightforward elementwise upper bound leads to 
$$
\triple \Hb\triple_2 \leq \|\Hb\|_\fro = \sum_{i\in \J} \sum_{j \in \J\backslash\{i\}} ([\db^{i}]^{\top}\db^{j})^{2} \leq k(k-1) \mu^{2} \leq k^{2} \mu^{2}.
$$ 
This proves that in the sense of positive definite matrices, 
$(1-k\mu)\Ib \preceq \Db_{\J}^{\top}\Db_{\J} \preceq (1+k\mu) \Ib$, 
which shows in turn the bound on $\delta_{k}(\Db)$.
Moreover, and since $\triple \Ib \triple_2 = 1$ with $\triple \Ab^\top \Ab \triple_2 = \triple \Ab \Ab^\top \triple_2$ for any matrix $\Ab$, we have
$$
\triple \Db_\J^\top \Db_\J\triple_2 = \triple\Db_\J \Db_\J^\top \triple_2 \leq  1+k\mu.
$$
By definition of $\triple.\triple_\infty$, we also have 
$$
\triple\Db_\J^\top \Db_\J \triple_\infty \leq 1+\triple \Hb \triple_\infty = 1+\max_{i\in\J}\sum_{j\in\J,j\neq i}|[\db^i]^\top\db^j| \leq 1+k\mu.
$$
Note that for $\triple\Db_{\J^c}^\top \Db_\J \triple_\infty$, there are no diagonal terms to take into account.

Now, if $k\mu<1$ holds, then we have 
$\max\{ \triple \Hb \triple_\infty, \triple \Hb \triple_2, \|\Hb \|_\fro, N(\Hb) \} \leq k\mu < 1$
and there are convergent series expansion of $[\Ib+\Hb]^{-1}$ in each of these norms~\citep{Horn1990}.
By sub-multiplicativity, we obtain
\[
\| \ThetaJb-\Ib \| = \| \sum_{t=1}^\infty (-1)^t \Hb^t \|\leq k\mu/(1-k\mu)
\]
where $\|.\|$ stands for one the four aforementioned matrix norms.
The last result lies in the fact that for the norms  $\triple \cdot \triple_\infty$, $\triple \cdot \triple_2$, we have $\triple \Ib\triple = 1$ and
$$
\triple \ThetaJb \triple \leq \triple \ThetaJb-\Ib \triple + \triple \Ib\triple  \leq 1+ k\mu/(1-k\mu) = 1/(1-k\mu).
$$
\end{proof}
We now derive a simple corollary which will be useful for the computation of expectations:
\begin{corollary}
\label{cor:coherence_elementwise}
Let $\Db \in \Real^{m\times p}$ be a dictionary with normalized columns and coherence $\mu$.  
With the notation from Lemma~\ref{lem:CoherenceBounds}, if $k\mu < 1$, we have 
for any $a \in \{1,2\}$ and for any $\J \subseteq \SET{p}$ with $|\J| \leq k$,
$$
\max_{i,j \in \SET{k}, i\neq j}| [\ThetaJb^a]_{i,j} | \leq \frac{a \mu}{(1-k\mu)^a}.
$$ 
\end{corollary}
\begin{proof}
We first make use of Lemma~\ref{lem:CoherenceBounds} which gives 
$$
N(\ThetaJb-\Ib) = k \cdot \max_{i,j \in \SET{k}}| [\ThetaJb - \Ib]_{i,j} | \leq \frac{k\mu}{1-k\mu},
$$
which notably implies that 
$$
\max_{i,j \in \SET{k}, i\neq j}| [\ThetaJb]_{i,j} | \leq \frac{\mu}{(1-k\mu)}.
$$
We continue by noticing that $[\ThetaJb-\Ib]^2 = \ThetaJb^2 - \Ib + 2( \Ib - \ThetaJb )$ and
by sub-multiplicativity of $N$
$$
N([\ThetaJb-\Ib]^2) \leq [N(\ThetaJb-\Ib)]^2 \leq \frac{(k\mu)^2}{(1-k\mu)^2}.
$$
Applying the triangle inequality, we obtain
$$
N(\ThetaJb^2 - \Ib) \leq 2N(\ThetaJb-\Ib) + \frac{(k\mu)^2}{(1-k\mu)^2} 
\leq \frac{ 2k\mu(1-k\mu)+(k\mu)^2}{(1-k\mu)^2} = \frac{ 2k\mu-(k\mu)^2}{(1-k\mu)^2} \leq \frac{ 2k\mu}{(1-k\mu)^2}.
$$
As a result, we finally get
$$
\max_{i,j \in \SET{k}, i\neq j}| [\ThetaJb^2]_{i,j} | \leq \frac{ 2\mu}{(1-k\mu)^2},
$$
hence the advertised conclusion.
\end{proof}
\begin{corollary}
\label{cor:coherence_prop2}
Let $\Dbo \in \Real^{m\times p}$ be a dictionary with normalized columns. If $k \mu(t) <1/2$ then, 
for any $\Wb \in \Wcal_{\Dbo}$, $\vb \in \Scal^{p}$ and $0 \leq t' \leq t$ we have 
\begin{equation}
\label{eq:BoundCondIrrep}
 \triple \big[\Db_{\J^c}^\top(t') \Db_{\J}(t')\big] \big[ \Db_\J^\top(t') \Db_\J(t') \big]^{-1} \triple_\infty 
\leq \frac{k\mu(t)}{1-k\mu(t)} = k\mu(t)Q_{t}^{2} = Q_{t}^{2}-1 <1,
\end{equation}
where we introduce
$$
Q_{t} \defin \frac{1}{\sqrt{1-k\mu(t)}} \geq C_t.
$$
\end{corollary}
\section{Expectation over $\J$}
\begin{lemma}\label{lem:ExpectOverJ}
Let $\Dbo \in \Real^{m\times p}$ be any dictionary and $\J$ a random support. Denoting by $\delta(i) \defin \indicator{\J}(i)$ the indicator function of $\J$, we assume that for all $i \neq j \in \SET{p}$
\begin{eqnarray*}
\Exp\{\delta(i)\} &=& \frac{k}{p}\\
\Exp\{\delta(i)\delta(j)\} &=& \frac{k(k-1)}{p(p-1)}.
\end{eqnarray*}
Then we have for any $\vb \in \Scal^{p}$ and $0 \leq t' \leq t$,
\begin{eqnarray}
\label{eq:ExpectOverJRIPBoundGramDJOffdiag}
\Exp \{\|[\Dbo]_{\J}^{\top}[\Dbo]_{\J}-\Ib\|_{\fro}^{2}\}
&=& \|\Dbo^{\top}\Dbo-\Ib\|_{\fro}^{2} \cdot \frac{k(k-1)}{p(p-1)}\\
\label{eq:ExpectOverJNormvJ}
\Exp \{\|\vb_{\J}\|_{2}^{2}\}
&=&
\frac{k}{p}\\
\label{eq:ExpectCrossOffDiagV}
\Exp\{\| \Db_{\J}^{\top}(t')\Db_{\J}(t')-\Ib\|_{\fro} \cdot \|\vb_{\J}\|_{2}\}
& \leq &
\left(\| \Dbo^{\top}\Dbo-\Ib\|_{\fro} \cdot \sqrt{\frac{k-1}{p-1}}\right) \cdot \frac{k}{p} + 2 \cdot C_{t} \cdot t \cdot \frac{k}{p}.
\end{eqnarray}
\end{lemma}
\begin{proof}
To obtain~\eqref{eq:ExpectOverJRIPBoundGramDJOffdiag} and~\eqref{eq:ExpectOverJNormvJ} we simply expand
\begin{eqnarray*}
\Exp \{\|[\Dbo]_{\J}^{\top}[\Dbo]_{\J}-\Ib\|_{\fro}^{2} \}
&=&
\Exp\Big\{\sum_{i \in \SET{p}} \sum_{j \in \SET{p}, j \neq i} \delta(i)\delta(j) \cdot [\dbo^{i}]^{\top}\dbo^{j}\Big\}
=
\sum_{i \in \SET{p}} \sum_{j \in \SET{p}, j \neq i} \frac{k(k-1)}{p(p-1)} \cdot [\dbo^{i}]^{\top}\dbo^{j}\\
\Exp\{\|\vb_{\J}\|_{2}^{2}\}
&=&
\Exp\Big\{\sum_{i \in \SET{p}} \delta(i) \cdot \vb_{i}^{2}\Big\}
=
\sum_{i \in \SET{p}} \frac{k}{p} \vb_{i}^{2} = \frac{k}{p} \cdot \|\vb\|_{2}^{2} = \frac{k}{p}.
\end{eqnarray*}
Now, by Lemma~\ref{lem:RIPBounds} and the Cauchy-Schwartz inequality for random variables
\begin{eqnarray*}
\Exp\{\| \Db_{\J}^{\top}(t')\Db_{\J}(t')-\Ib\|_{\fro} \cdot \|\vb_{\J}\|_{2}\}
& \leq &
\Exp\{ \| [\Dbo]_{\J}^{\top}[\Dbo]_{\J}-\Ib\|_{\fro} \cdot \|\vb_{\J}\|_{2}\}+ 2 \cdot C_{t} \cdot t \cdot \Exp\{\|\vb_{\J}\|_{2}^{2}\}\\
& \leq &
\sqrt{\Exp\{ \| [\Dbo]_{\J}^{\top}[\Dbo]_{\J}-\Ib\|_{\fro}^{2}\}} \cdot \sqrt{\Exp\{\|\vb_{\J}\|_{2}^{2}\}} + 2 \cdot C_{t} \cdot t \cdot \frac{k}{p}\\
& \leq &
\| \Dbo^{\top}\Dbo-\Ib\|_{\fro} \cdot \sqrt{\frac{k(k-1)}{p(p-1)}} \cdot \sqrt{\frac{k}{p}} + 2 \cdot C_{t} \cdot t \cdot \frac{k}{p}
\end{eqnarray*}
\end{proof}
\section{Proof of Proposition~\ref{prop:maindeltaphi}}~\label{sec:proof_control_difference_phi}
In this section, we establish the results required to lower bound $\Delta\Phi_n(\Wb,\vb,t)$. 
We denote
\begin{equation}
\label{eq:DefDeltaPhiXi}
\Delta \phi_{\xb^{i}}(\Wb,\vb,t) \defin \phi_{\xb^{i}}(\Db(\Wb,\vb,t)|\sbo^{i})-\phi_{\xb^{i}}(\Dbo|\sbo^{i}).
\end{equation} 
The overall approach consists of the following steps:

\begin{enumerate}
\item {\bf Concentration around the expectation:} 
\begin{lemma}\label{lem:concentrationmain}
Under our signal model, for any $\Wb \in \Wcal_{\Dbo}$, $\vb \in \Scal^{p}$, $\tau \in [0,\sqrt{n}]$, we have
\begin{equation}\label{eq:concentrationmain}
\Pr\Big( 
\Delta\Phi_n(\Wb,\vb,t) <
\Exp\{\Delta\phi_\xb(\Wb,\vb,t)\}
-c(t) \frac{\tau}{\sqrt{n}}
\Big) \leq 
2 \cdot \exp(-\tau^{2})
\end{equation}
with 
\begin{equation} 
\label{eq:concentrationmain_cst}
c(t) \defin
102 \cdot \left(t^{2}  \sigma_{\alpha}^{2} + 2m  \sigma^{2}+ 2\lambda k\sigma_{\alpha}\right)
\end{equation}
\end{lemma}
\item {\bf Control of the Lipschitz constant:} the second step consists in showing that $(\Wb,\vb) \mapsto \Delta\Phi_n(\Wb,\vb,t)$ is Lipschitz with controlled constant with respect to the metric
\begin{equation}
\label{eq:DefMetric}
d\big((\Wb,\vb),(\Wb',\vb')\big) 
\defin 
\max\left\{  \max_{j\in \SET{p}}\| \wb^{j} -\wb'^{j} \|_2\ ,\  \| \vb -\vb' \|_2\right\}.
\end{equation}
\begin{lemma}\label{lem:lipschitzmain}
Assume that $t < \sqrt{1-\delta_k(\Dbo)}$. Under our signal model we have for any $\tau \in [0, \sqrt{n}]$, except with probability at most $2\exp(-\tau^{2})$: for all $(\Wb,\vb)$ and $(\Wb',\vb')$
\begin{equation*}
\label{eq:lipschitzmain}
\big|\Delta\Phi_{n}(\Wb,\vb,t) - \Delta\Phi_{n}(\Wb',\vb',t)\big|
\leq
L \cdot \left(1+\frac{4\tau}{\sqrt{n}}\right) \cdot d\big((\Wb,\vb),(\Wb',\vb')\big).
\end{equation*}
where 
\begin{equation}
\label{eq:lipschitzmain_cst}
L \defin 30 C_{t}^{3} \cdot t  \cdot \left[5(k\sigma_{\alpha}^{2}+m\sigma^{2})+\lambda^{2}k\right] 
\end{equation}
\end{lemma}
\item {\bf $\epsilon$-net argument:} combining Lemmata~\ref{lem:concentrationmain}-\ref{lem:lipschitzmain} together with an estimate of the size of an $\epsilon$-net of $\Wcal \times \Scal^{p}$ with respect to the considered metric, we obtain
\begin{lemma}\label{lem:epsnetmain}
Assume that $t < \sqrt{1-\delta_{k}(\Dbo)}$ and that $C_{t} \leq 1.5$. Under our signal model, and assuming that 
\begin{equation*}
\label{eq:epsnet_sizen}
\frac{n}{\log n} \geq mp
\end{equation*}
we have, except with probability at most 
$\left(\frac{mpn}{9}\right)^{-mp/2}$,
\begin{equation*}\label{eq:epsnetmain}
\inf_{\Wb \in \Wcal_{\Dbo},\vb \in \Scal^{p}}
\Delta\Phi_n(\Wb,\vb,t) \geq
\inf_{\Wb \in \Wcal_{\Dbo},\vb \in \Scal^{p}} \Exp\{\Delta\phi_\xb(\Wb,\vb,t)\}
-B \cdot \sqrt{mp \frac{\log n}{n}}.
\end{equation*}
with 
\begin{equation} 
\label{eq:epsnetmain_cst}
B \defin
3045 \left( k\sigma_{\alpha}^{2} \cdot t+2 m\sigma^{2} + \lambda k \sigma_{\alpha} +\lambda^{2}k \cdot t \right).
\end{equation}
\end{lemma}

\item {\bf Control in expectation:} 
\begin{lemma}\label{lem:expectationmain}
Assume that $k\mu(t) < 1/2$. Under our signal model, we have
\begin{eqnarray}
\label{eq:expectationmain}
\inf_{\Wb \in \Wcal_{\Dbo},\vb \in \Scal^{p}}
\Exp \{\Delta\phi_\xb(\Wb,\vb,t)\} 
& \geq &
(1 - \Kcal^2) \cdot \frac {\Exp[ \alpha_{0}^{2}]}{2} 
\cdot  
\frac{k}{p} \cdot t^{2}\notag\\
&&
-Q_t^2 \cdot t \cdot \frac{k}{p} \cdot \triple \Dbo \triple_2 \cdot k\mu(t) \cdot \lambda \cdot \left(4 Q_{t}^2 \lambda + 3\Exp\{|\alpha_{0}|\} \right)\notag.
\end{eqnarray}
with $\Kcal \defin C_{t} \cdot (\triple \Dbo \triple_2 \cdot \sqrt{k/p} + t)$.
\end{lemma}
\end{enumerate}
We obtain Proposition~\ref{prop:maindeltaphi} by combining Lemmata~\ref{lem:epsnetmain}-\ref{lem:expectationmain}.
We now proceed to the proof of these lemmata.

\subsection{Expansion of $\Delta\phi_{\xb}$} 
We expand $\Delta\phi_{\xb}$ into the sum of six terms.
\begin{lemma}\label{lem:delta_phi}
We have 
\begin{eqnarray}
\Delta\phi_{\xb}(t) 
&\defin& 
\phi_\xb(\Db(t)|\sb) - \phi_\xb(\Dbo|\sb)\notag\\
&=& 
\frac{1}{2} \xb^\top \big[ \PJb(0)  -  \PJb(t) \big] \xb 
 - \lambda \sb_{\J}^\top \big[ \ThetaJb(0) [\Dbo]_{\J}^\top  - \ThetaJb(t) [\Db(t)]_{\J}^\top \big] \xb \notag\\
&& +\frac{\lambda^2}{2} \sb_{\J}^\top \big[  \ThetaJb(0)   - \ThetaJb(t)  \big]  \sb_{\J}
 \label{eq:DevelDeltaPhi1}\\
 &=& 
\zeta_{\alphab,\alphab}(t) 
+ \zeta_{\alphab,\varepsilonb}(t) 
+\zeta_{\varepsilonb,\varepsilonb}(t) 
+\zeta_{\sb,\alphab}(t) 
+\zeta_{\sb,\varepsilonb}(t) 
+\zeta_{\sb,\sb}(t)\label{eq:DevelDeltaPhi2}
\end{eqnarray}
where
\begin{eqnarray}
\zeta_{\alphab,\alphab}(t) 
&\defin &\frac{1}{2} \alphabo^\top  \Dbo^\top  
\Big[\PJb(0) - \PJb(t)\Big] \Dbo \alphabo \label{eq:diff_alphaalpha}\\
\zeta_{\alphab,\varepsilonb}(t) 
&\defin& \alphabo^\top \Dbo^\top  
\Big[\PJb(0) - \PJb(t)\Big]\varepsilonb\label{eq:diff_alphaeps} \\
\zeta_{\varepsilonb,\varepsilonb}(t) 
&\defin&\frac{1}{2} \varepsilonb^\top \Big[\PJb(0) - \PJb(t)\Big]\varepsilonb\label{eq:diff_epseps}\\
\zeta_{\sb,\alphab}(t) 
&\defin&-\lambda [\sign(\alphabo)]_{\J}^\top 
\Big[ \ThetaJb(0) [\Dbo]_{\J}^\top - \ThetaJb(t) [\Db(t)]_{\J}^\top \Big] \Dbo \alphabo\label{eq:diff_signalpha}\\
\zeta_{\sb,\varepsilonb}(t) 
&\defin&-\lambda [\sign(\alphabo)]_{\J}^\top 
\Big[ \ThetaJb(0) [\Dbo]_{\J}^\top - \ThetaJb(t) [\Db(t)]_{\J}^\top \Big] \varepsilonb \label{eq:diff_signeps}\\
\zeta_{\sb,\sb}(t) 
&\defin&\frac{\lambda^2}{2} [\sign(\alphabo)]_{\J}^\top 
\Big[  \ThetaJb(0)  -  \ThetaJb(t) \Big] \sign(\alphabo)_{\J}\label{eq:diff_signsign}.
\end{eqnarray}
\end{lemma}
\begin{proof}
Denoting $\sb = \sign(\alphabo)$ and $\J \subseteq \SET{p}$ the support of $\alphabo$, we have by definition (see Equation~\eqref{eq:phi}):
\begin{eqnarray}
\phi_\xb(\Db(t)|\sb) &=& 
\frac{1}{2} \big[ \|\xb\|_2^2 - ( [\Db(t)]_\J^\top \xb - \lambda \sb_\J)^\top 
([\Db(t)]_\J^\top [\Db(t)]_\J)^{-1} 
([\Db(t)]_\J^\top \xb -\lambda \sb_\J ) \big] \notag\\
&=& \frac{1}{2} \|\xb\|_2^2 - \frac{1}{2} \xb^\top \PJb(t) \xb + \lambda \sb_{\J}^\top  \ThetaJb(t) [\Db(t)]_{\J}^\top \xb 
- \frac{\lambda^2}{2} \sb_{\J}^\top \ThetaJb(t) \sb_{\J}.\label{eq:PhiExpansion}
\end{eqnarray}
This yields~\eqref{eq:DevelDeltaPhi1} and we conclude thanks to $\xb = \Dbo \alphabo + \varepsilonb = [\Dbo]_{\J} [\alphabo]_{\J} + \varepsilonb$.
\end{proof}

\subsection{Proof of Lemma~\ref{lem:concentrationmain}}\label{sec:concentration}
Fix $\Wb$ and $\vb$ and denote $y^{i}(t) \defin \phi_{\xb^{i}}(\Db(\Wb,\vb,t)|\sbo^{i})$. By definition of $\phi_{\xb}$ we have $y^{i}(t) \leq \Lcal_{\xb^{i}}(\Db(\Wb,\vb,t),\alphabo^{i})$ hence, using Lemma~\ref{lem:suprxi} we have for any $\tau \geq 1$ 
\[
\Pr(y^{i}(t) \geq A_\Lcal(t) \cdot \tau) \leq e^{-\tau}
\]
where
\[
A_\Lcal(t) \defin \frac{5(1+\log 2)}{2} \cdot \left(t^{2} \sigma_{\alpha}^{2} + m  \sigma^{2}+ \lambda k\sigma_{\alpha}\right).
\]
Hence, exploiting Corollary~\ref{cor:concentration_subgaussian} with $\kappa=1$ and $0\leq \tau \leq \sqrt{n}$, we obtain, 
\[
\Pr\left(\Big|\frac{1}{n} \sum_{i=1}^{n} \left( y^{i}(t) -\Exp\{ y^{i}(t)\}\right)\Big|
\geq 
24 A_\Lcal(t) \cdot \frac{\tau}{\sqrt{n}}
\right) 
\leq \exp(-\tau^{2}).
\]
Observing that $\Delta\Phi_{n}(\Wb,\vb,t) = \frac{1}{n} \sum_{i=1}^{n} (y^{i}(t)-y^{i}(0))$ we obtain, by a union bound,
\[
\Pr\left(
\Big|\Delta\Phi_{n}(\Wb,\vb,t) -\Exp\{ \Delta\Phi_{n}(\Wb,\vb,t)\}\Big|
\geq 
24 (A_\Lcal(t) +A_\Lcal(0))\cdot \frac{\tau}{\sqrt{n}}
\right) 
\leq 2\exp(-\tau^{2}).
\]
We conclude by expliciting
\[
24(A_\Lcal(t) +A_\Lcal(0)) = 60(1+\log 2)\left(t^{2}  \sigma_{\alpha}^{2} + 2m  \sigma^{2}+ 2\lambda k\sigma_{\alpha}\right) \leq c(t).
\]

\subsection{Proof of Lemma~\ref{lem:lipschitzmain}}
Given the expansion~\eqref{eq:DevelDeltaPhi1}, using the shorthands $\PJb = \PJb(\Wb,\vb,t)$ and $\PJb' = \PJb(\Wb',\vb',t)$, 
as well as for other similar quantities, 
and averaging over $n$,  we obtain
\begin{eqnarray*}
\Big|\Delta\Phi_{n}-\Delta\Phi_{n}'\Big|
& \leq &
\frac{1}{2n} \sum_{i=1}^{n} \|\xb^{i}\|_{2}^{2} 
\cdot 
\max_{i\in\SET{n}} \triple \PJib-\PJib'\triple_{2}
+
\frac{\lambda \sqrt{k}}{n} \sum_{i=1}^{n} \|\xb^{i}\|_{2}
\cdot
\max_{i\in\SET{n}} \triple \ThetaJib\Db-\ThetaJib'\Db'\triple_{2}\\
&&+
 \frac{\lambda^{2} k}{2} \cdot \max_{i\in\SET{n}} \triple \ThetaJib-\ThetaJib'\triple_{2} 
\end{eqnarray*}
Using Lemma~\ref{lem:BoundDeltaPJThetaJ} this yields the Lipschitz bound $\Big|\Delta\Phi_{n}-\Delta\Phi_{n}'\Big| \leq L_{n} \cdot d\big((\Wb,\vb),(\Wb',\vb')\big)$ with
\begin{eqnarray}
\label{eq:LipschitzBound1}
L_{n}
& \leq &
\frac{5}{2} \cdot t \cdot C_{t}^{3} \cdot
\left\{
\frac{1}{n} \sum_{i=1}^{n} \|\xb^{i}\|_{2}^{2} + 2\lambda\sqrt{k} \cdot \frac{1}{n} \sum_{i=1}^{n} \|\xb^{i}\|_{2} + \lambda^{2}k
\right\}
\end{eqnarray}
Using Lemma~\ref{lem:l2norm_signal} we check that $y^{i}=\|\xb^{i}\|_{2}^{2}$  satisfies the hypothesis (see Eq.~\eqref{eq:DefExponentialDecay}) of Lemma~\ref{lem:truncated_moments}  with $A = 5(k\sigma_{\alpha}^{2}+m\sigma^{2})$. Hence, exploiting Corollary~\ref{cor:concentration_subgaussian} with $\kappa=1$ and $0\leq \tau \leq \sqrt{n}$, we obtain, except with probability at most $2\exp(-\tau^{2})$
\begin{eqnarray*}
\frac{1}{n} \sum_{i=1}^{n} \|\xb^{i}\|_{2}^{2} 
&\leq& 
6 \cdot [5\cdot(k\sigma_{\alpha}^{2}+m\sigma^{2})] \cdot \left(1+\frac{4\tau}{\sqrt{n}}\right)\\
\frac{1}{n} \sum_{i=1}^{n} \|\xb^{i}\|_{2} 
&\leq& 
6 \cdot \sqrt{5 \cdot (k\sigma_{\alpha}^{2}+m\sigma^{2})} \cdot \left(1+\frac{4\tau}{\sqrt{n}}\right).
\end{eqnarray*}
Inserting the above estimates into~\eqref{eq:LipschitzBound1} yields, except with probability at most $2\exp(-\tau^{2})$, 
\begin{eqnarray*}
\label{eq:LipschitzBound2}
L_{n} &\leq &L' \cdot \left(1+\frac{4\tau}{\sqrt{n}}\right)\\
L'
& = &
\frac{30}{2} \cdot C_{t}^{3} \cdot t  \cdot
\left[\sqrt{5\cdot (k\sigma_{\alpha}^{2}+m\sigma^{2})}+\lambda\sqrt{k}\right]^{2}  
\leq 
30 C_{t}^{3} \cdot t  \cdot
\left[5(k\sigma_{\alpha}^{2}+m\sigma^{2})+\lambda^{2}k\right] \defin L\notag.
\end{eqnarray*}

\subsection{Proof of Lemma~\ref{lem:epsnetmain}}

The proof of Lemma~\ref{lem:epsnetmain} exploits the covering number $\Ncal$ of $\Wcal_{\Dbo} \times \Scal^p$ with respect to the metric~
\eqref{eq:DefMetric}. For background about covering numbers, we refer the reader to~\citet{Cucker2002} and references therein.
\begin{lemma}[$\epsilon$-nets for $\Wcal_{\Dbo} \times \Scal^p$]\label{lem:cover}
For the Euclidean metric, and for any $\epsilon >0$, we have 
$$
\Ncal(\Scal^p,\epsilon) \leq \Big(1+\frac{2}{\epsilon}\Big)^p.
$$
Moreover, define on $\Real^{m\times p}$ the norm $\Omega(\Mb) \defin \max_{j\in\SET{p}}\|\mb^j\|_2$.
For the metric induced by $\Omega$, and for any $\epsilon > 0$, we have 
$$
\Ncal(\Wcal_{\Dbo},\epsilon) \leq \Big(1+\frac{2}{\epsilon}\Big)^{p(m-1)}.
$$
\end{lemma}
\begin{proof}
We resort to Lemma 2 in \citet{Vershynin2010}, which gives the first conclusion for the sphere in $\Real^p$.
As for the second result,
remember that the set $\Wcal_{\Dbo}$ is defined as a product of spheres in spaces of dimension $m-1$.
Indeed, we have for any $\Wb \in \Wcal_{\Dbo}$ and for any $j \in \SET{p}$, $\|\wb^j\|_2=1$ along with 
the constraint $[\dbo^j]^\top \wb^j=0$, 
which implies that $\wb^j$ belongs to the orthogonal space of $\text{span}(\dbo^j)$ of dimension $m-1$.
Considering a product of $p$ nets such as that used for $\Scal^p$, 
the second conclusion follows from the definition of the metric based on $\Omega$.
\end{proof}

From Lemma ~\ref{lem:cover} we know that for any $0<\epsilon\leq1$ there exists $\epsilon$-net of $\Wcal \times \Scal^{p}$ with respect to the
metric~\eqref{eq:DefMetric} with at most $(3/\epsilon)^{mp}$ elements. Combining this with Lemmata~\ref{lem:concentrationmain}-\ref{lem:lipschitzmain}, we have for any $0 \leq \tau \leq \sqrt{n}$: except with probability at most $(3/\epsilon)^{mp} \cdot 2 \exp(-\tau^{2}) + 2\exp(-\tau^{2}) \leq 4 \cdot (3/\epsilon)^{mp} \cdot \exp(-\tau^{2})$
\begin{eqnarray*}
\inf_{\Wb,\vb} \Delta\Phi_{n}(\Wb,\vb,t)
&\geq&
\inf_{\Wb,\vb} \Exp\{\Delta\Phi_{n}(\Wb,\vb,t)\} - \left(c(t) \cdot \frac{\tau}{\sqrt{n}} +L \cdot \left(1+\frac{4\tau}{\sqrt{n}}\right) \cdot \epsilon\right)
\end{eqnarray*}
Now we set $\tau \defin \sqrt{mp \log n}$, and $\epsilon \defin \frac{\tau}{\sqrt{n}} = \sqrt{mp \frac{\log n}{n}}$. Under the assumption that
\[
\frac{n}{\log n} \geq mp
\]
we check that $\tau \leq \sqrt{n}$, $\epsilon\leq1$, hence, the probability bound holds. We estimate the probability bound with :
\begin{eqnarray*}
\log \frac{3}{\epsilon} 
& = & 
\log \frac{3}{\sqrt{mp}} + \log \sqrt{\frac{n}{\log n}} 
\leq 
\frac{1}{2} \log \frac{9}{mp} + \frac 12 \log n\\
mp \log \frac{3}{\epsilon} -\tau^{2} 
& \leq & \frac{mp}{2} \log \frac{9}{mp} + \frac {mp}{2} \log n -mp \log n 
=
\frac{mp}{2} \log \frac{9}{mp} -\frac{mp}{2} \log n\\
(3/\epsilon)^{mp} \exp(-\tau^{2}) & \leq & \left(\frac{mpn}{9}\right)^{-mp/2}.
\end{eqnarray*}
Finally, recalling that %since $mp \geq 4 > e$, we have $n \geq n/\log n \geq mp > e$ and $(1+\beta)\log n \geq 5\log n \geq 5$, hence
\begin{eqnarray*}
L & \defin & 30 C_{t}^{3} \cdot t  \cdot \left[5(k\sigma_{\alpha}^{2}+m\sigma^{2})+\lambda^{2}k\right] \\
c(t) & \defin & 102 \cdot \left(t^{2}  \sigma_{\alpha}^{2} + 2m  \sigma^{2}+ 2\lambda k\sigma_{\alpha}\right),
\end{eqnarray*}
and since the assumption $C_{t} \leq 1.5$ implies $150 C_{t}^{3} \leq 507$, we obtain
\begin{eqnarray*}
c(t)+L &\leq &
609 \left( k\sigma_{\alpha}^{2} \cdot t+2 m\sigma^{2} + \lambda k \sigma_{\alpha} +\lambda^{2}k \cdot t \right) \defin B/5\\
c(t) \cdot \frac{\tau}{\sqrt{n}} +L   \cdot \left(1+\frac{4\tau}{\sqrt{n}}\right)\cdot \epsilon
&\leq&
(c(t)+L) \cdot \epsilon + 4 (c(t)+L) \epsilon^{2} = (c(t)+L) 5 \epsilon  \leq B \epsilon.
\end{eqnarray*}

\subsection{Proof of Lemma~\ref{lem:expectationmain}}\label{sec:expectation}
First, we observe that by the statistical independence between $\alphab$ and $\varepsilonb$ we have
\[
\Exp\{\zeta_{\alphab,\varepsilonb}(t) \}
=
\Exp\{\zeta_{\sb,\varepsilonb}(t) \}
=
0.
\]
Moreover, we can rewrite 
\begin{eqnarray*}
\zeta_{\alphab,\alphab}(t)
&=&
\frac12 \cdot \trace\left( [\alphabo]_{\J} [\alphabo]_{\J}^\top \cdot [\Dbo]_{\J}^{\top}  ( \PJb(0) - \PJb(t) ) [\Dbo]_\J \right)\\
\zeta_{\varepsilonb,\varepsilonb}(t) 
&=& \frac{1}{2} \cdot \trace \left(\varepsilonb\varepsilonb^{\top} \cdot (\PJb(0)-\PJb(t))\right)\\
 \zeta_{\sb,\alphab}(t) 
 &=& 
 -\lambda \cdot \trace \left( [\alphabo]_{\J}\, \sign(\alphabo)_{\J}^\top \cdot 
\big[ \ThetaJb(0) [\Dbo]_{\J}^\top-\ThetaJb(t) [\Db(t)]_{\J}^\top \big] [\Dbo]_\J 
\right)\\
\zeta_{\sb,\sb}(t)
&=& 
\frac{\lambda^{2}}{2} \cdot 
\trace\left(\ThetaJb(0)-\ThetaJb(t)\right).
\end{eqnarray*} 
Since the coefficients $\alphab,\varepsilonb$ are independent from the support $\J$ we obtain
\begin{eqnarray}
\Exp\{\zeta_{\alphab,\alphab}(t)\}
&=&
\frac{\Exp\{\alpha^{2}\}}{2} \cdot \Exp_{J}\left\{\trace\left(  [\Dbo]_{\J}^{\top}  ( \Ib - \PJb(t) ) [\Dbo]_\J \right)\right\}\\
\Exp\{\zeta_{\varepsilonb,\varepsilonb}(t) \}
&=& 
\frac{\Exp\{\varepsilon^{2}\}}{2} \cdot  \Exp_{J}\left\{\trace \left( \PJb(0)-\PJb(t)\right)\right\} = 0\\
\Exp\{ \zeta_{\sb,\alphab}(t) \}
 &=& 
- \lambda \cdot \Exp\{|\alpha|\} \cdot \Exp_{J}\left\{ \trace \left( \big[ \ThetaJb(0) [\Dbo]_{\J}^\top-\ThetaJb(t) [\Db(t)]_{\J}^\top \big] [\Dbo]_\J 
\right)\right\}\\
\Exp\{\zeta_{\sb,\sb}(t)\}
&=& 
\frac{\lambda^{2}}{2} \cdot 
\Exp_{J}\left\{ \trace\left(\ThetaJb(0)-\ThetaJb(t)\right)\right\}
\end{eqnarray}
where we used the fact that: (a) $\PJb(0) \Dbo = \Dbo$; (b) since $\PJb(t)$ is an orthogonal projector onto a subspace of dimension $k$, $\trace(\PJb(0)-\PJb(t)) = k-k = 0$.

The lemma below provide estimates of the remaining non-vanishing expectations which come up in the quadratic forms~\eqref{eq:diff_alphaalpha} and~\eqref{eq:diff_signsign} and the bilinear form~\eqref{eq:diff_signalpha}. They directly provide Lemma~\ref{lem:expectationmain} as a corollary. 
\begin{lemma}\label{lem:bias_expectation}
If $k\mu(t) < 1/2$ then for any $\Wb \in \Wcal_{\Dbo}$, $\vb \in \Scal^{p}$  we have 
\begin{eqnarray}
\label{eq:leading_expectation}
\Exp_{\J}\left\{\trace\left(  [\Dbo]_{\J}^{\top}  ( \Ib - \PJb(t) ) [\Dbo]_\J \right)\right\}
&\geq& 
(1-\Kcal^2) \cdot \frac{k}{p} t^{2}\\
\label{eq:bias_expectation}
\left|\Exp_{\J}\left\{ \trace \left( \big[ \ThetaJb(0) [\Dbo]_{\J}^\top - \ThetaJb(t) [\Db(t)_{\J}]^\top\big] [\Dbo]_\J 
\right)\right\}\right|
&\leq& 
3 Q_t^2 \cdot t \cdot \frac{k}{p} \cdot \triple \Dbo\triple_2 \cdot k\mu(t)\\
\label{eq:bias_expectation_signsign}
\left|\Exp_{\J}\left\{ \trace\left(\ThetaJb(0)-\ThetaJb(t)\right)\right\}\right|
&\leq& 
8 Q_t^4 \cdot t \cdot \frac{k}{p} \cdot \triple \Dbo \triple_2 \cdot k\mu(t).
\end{eqnarray}
with $\Kcal \defin C_{t} \cdot (\triple \Dbo \triple_2 \cdot \sqrt{k/p} + t)$.
\end{lemma}
\begin{proof}[Proof of Lemma~\ref{lem:bias_expectation} -  Equation~\eqref{eq:leading_expectation}]
Since $k\mu(t) < 1/2$, we have $t < 1/(6k) \leq 1/6$ and $\delta_k(\Dbo) \leq k\mu_0  < 1/2$, so that $t < \sqrt{1 - \delta_k(\Dbo)}$.
In particular, we have $t < \pi/2$ 
and the matrix $\Cb = \Diag(\cos(\vb_{j}t))$ is invertible. From the equality $\Db(t) = \Dbo \Cb + \Wb \Sb$ with $\Sb = \Diag(\cos(\vb_{j}t))$ we deduce $\Dbo = \Db(t)\Cb^{-1} - \Wb \Tb$ with $\Tb = \Diag(\tan(\vb_{j}t))$. Since the columns of $[\Db(t)\Cb^{-1}]_{\J}$ belong to the span of $[\Db(t)]_{\J}$ we obtain
\begin{eqnarray*}
\trace \big( [\Dbo]_\J^{\top} (\Ib - \PJb(t)) [\Dbo]_\J  \big)
&=&
\|(\Ib - \PJb(t)) [\Dbo]_\J\|_{\fro}^{2} = \|(\Ib-\PJb(t))[\Wb\Tb]_{\J}\|_{\fro}^{2}\\
&=&
\|[\Wb\Tb]_\J\|_{\fro}^{2}-\| \PJb(t) [\Wb\Tb]_\J\|_{\fro}^{2}.
\end{eqnarray*}
For the first term, since $\|\wb^{j}\|_{2}=1$, we have
\[
\|[\Wb\Tb]_\J\|_{\fro}^{2} 
= \sum_{j=1}^{p} \delta(j) \cdot \|\wb^{j}\|_{2}^{2} \cdot \tan^{2}(\vb_{j}t)
= \sum_{j=1}^{p} \delta(j) \cdot \tan^{2}(\vb_{j}t)
\]
hence, since $\|\vb\|_{2}=1$, and $\tan^2(u) \geq u^2$ for  $|u|\leq 1$ we have
\[
\Exp\ \|[\Wb \Tb]_\J\|_{\fro}^{2} = \frac{k}{p} \cdot  \|\Wb \Tb\|_{\fro}^{2} 
= 
\frac{k}{p} \cdot \sum_{j=1}^{p} \tan^{2} (\vb_{j}t)^{2} \geq \frac{k}{p} \cdot \sum_{j=1}^p t^2 \vb_j^2  =  \frac{k}{p} \cdot t^2.
\] 
For the second term, since $\PJb(t) = \Db_{\J}(t)\ThetaJb(t)\Db_{\J}^{\top}(t)$, using Lemma~\ref{lem:RIPBounds}, we have the bound
\begin{eqnarray*}
\| \PJb(t) [\Wb\Tb]_\J \|_{\fro}^{2}  
&\leq &
C_{t}^{2} \cdot \| \Db_{\J}^\top(t) [\Wb\Tb]_\J \|_\fro^2,
\end{eqnarray*}
Moreover, by Lemma~\ref{lem:orthoprojcomplement}, using the Cauchy-Schwarz inequality for random variables
\begin{eqnarray*}
\| \Db_{\J}^\top(t) [\Wb\Tb]_\J \|_\fro 
&\leq &
\| [\Dbo]_{\J}^\top [\Wb\Tb]_\J \|_\fro + \| [\Db(t)-\Dbo]_{\J}^\top [\Wb\Tb]_\J \|_\fro
\leq
\| [\Dbo]_{\J}^\top [\Wb\Tb]_\J \|_\fro + t \cdot \|[\Wb\Tb]_\J \|_\fro,\\
\\
\Exp\{\| \Db_{\J}^\top(t) [\Wb\Tb]_\J \|_\fro^2\}
& \leq &
\Exp\left\{\| [\Dbo]_{\J}^\top [\Wb\Tb]_\J \|_\fro^2\right\} 
+2 t \cdot \Exp\left\{\| [\Dbo]_{\J}^\top [\Wb\Tb]_\J \|_\fro \cdot \|[\Wb\Tb]_\J \|_\fro\right\} +t^{2} \cdot \Exp\left\{\|[\Wb\Tb]_\J \|_\fro^{2}\right\}\\
& \leq &
\Exp\left\{\| [\Dbo]_{\J}^\top [\Wb\Tb]_\J \|_\fro^2\right\} 
+2 t \cdot \sqrt{\Exp\left\{\| [\Dbo]_{\J}^\top [\Wb\Tb]_\J \|_\fro^{2}\right\}}  \cdot \sqrt{\Exp\left\{\|[\Wb\Tb]_\J \|_\fro^{2}\right\}}\\
&& +t^{2} \cdot \Exp\left\{\|[\Wb\Tb]_\J \|_\fro^{2}\right\}\\
&\leq&
\left(\sqrt{\Exp\left\{\| [\Dbo]_{\J}^\top [\Wb\Tb]_\J \|_\fro^{2}\right\}} + t \cdot \sqrt{\frac{k}{p}} \cdot \|\Wb\Tb \|_\fro\right)^{2}
\end{eqnarray*}
Now, proceeding as in Lemma~\ref{lem:ExpectOverJ}, we compute 
\begin{eqnarray*}
\Exp\Big[ \| [\Dbo]_{\J}^\top [\Wb \Tb]_\J \|_\fro^2  \Big] 
&=& 
\frac{k(k-1)}{p(p-1)}  \cdot \| \Dbo^\top \Wb \Tb \|_\fro^2 
\leq
\frac{k(k-1)}{p(p-1)} \cdot \triple \Dbo^{\top} \triple_{2}^{2} \cdot \|\Wb \Tb \|_\fro^2
\end{eqnarray*}
hence
\begin{eqnarray*}
\| \Db_{\J}^\top(t) [\Wb\Tb]_\J \|_\fro 
\leq \frac{k}{p} \cdot \|\Wb \Tb \|_\fro^2 \cdot \left(\sqrt{\frac{k-1}{p-1}} \cdot \triple \Dbo^{\top} \triple_{2} + t\right)^{2}
\leq \frac{k}{p} \cdot \|\Wb \Tb \|_\fro^2 \cdot \left(\sqrt{\frac{k}{p}} \cdot \triple \Dbo \triple_2 + t\right)^{2}.
\end{eqnarray*}
Putting the pieces together, we obtain the lower bound
\begin{eqnarray*}
\trace \big( [\Dbo]_\J^{\top} (\Ib - \PJb(t)) [\Dbo]_\J  \big)
&\geq&
\frac kp \cdot \|\Wb \Tb\|_\fro^2 \cdot \left(1- \left\{ C_{t} \cdot \bigg(\triple \Dbo \triple_2 \cdot \sqrt{\frac{k}{p}} + t\bigg) \right\}^2 \right)\\
&=&
\frac kp \cdot \|\Wb \Tb\|_\fro^2 \cdot \left(1- \Kcal^2 \right).
\end{eqnarray*}
\end{proof}

\begin{proof}[Proof of Lemma~\ref{lem:bias_expectation} - Equation~\eqref{eq:bias_expectation}]

We first develop Equation~\eqref{eq:bias_expectation} and use that $\ThetaJb(0) [\Dbo]_{\J}^\top[\Dbo]_\J =\Ib$
in order to obtain
$$
\trace \left( \big[ \ThetaJb(0) [\Dbo]_{\J}^\top - \ThetaJb(t) \Db(t)_{\J}^\top(t)\big] [\Dbo]_\J \right)
= k - \trace \left( \ThetaJb(t) [\Db(t)]_\J^\top [\Dbo]_\J \right).
$$
Appyling Lemma~\ref{lem:paramonto}, we know there exists $\Wb_t \in \Wcal_{\Db(t)}$ such that 
$$
\Dbo = \Db(t) \Diag(\cos(\vb_j t)) +\Wb_t \Diag(\sin(\vb_j t)), 
$$
and the trace above further simplifies as
\begin{eqnarray*}
\trace \left( \big[ \ThetaJb(0) [\Dbo]_{\J}^\top - \ThetaJb(t) [\Db(t)]_\J^\top\big] [\Dbo]_\J \right) &=& 
k - \sum_{j \in \J} \cos(\vb_j t) - \trace \left( \ThetaJb(t) [\Db(t)]_\J^\top [\Wb_t \Sb(t)]_\J \right),\\
&=& \sum_{j\in\J} ( 1 - \cos(\vb_j t) ) - \trace \left( \ThetaJb(t) [\Db(t)]_{\J}^\top [\Wb_t \Sb(t)]_\J \right),
\end{eqnarray*}
where for short, we refer to $\Diag(\sin(\vb_j t))$ as $\Sb(t)$.

The first term is simple to handle since we have 
$$
\Exp_\J\big[ \sum_{j\in\J} ( 1 - \cos(\vb_j t) ) \big] \leq \frac{t^2}{2} \Exp_\J[ \|\vb_\J\|_2^2 ] = \frac{t^2}{2} \frac{k}{p}.
$$
We now turn to the second term whose control is more involved.
Following~\citet{Geng2011}, we introduce the self-adjoint operator $\Gamma_{\Db(t)}$ 
defined for any $\Mb \in \Real^{m \times p}$ by
$$
\Gamma_{\Db(t)}(\Mb) \defin \Big[ \Gamma_1(t)\, \mb^1,\dots, \Gamma_p(t)\, \mb^p  \Big],\quad\text{with}\
\Gamma_j(t) \defin \Ib - \db(t)^j [\db(t)^j]^\top.
$$
In words, $\Gamma_{\Db(t)}(\Mb)$ projects each column of $\Mb$ onto the orthogonal complement of the corresponding column of
the dictionary $\Db(t)$. In particular, note that for any $\Mb \in \Wcal_{\Db(t)}$, we therefore have $\Gamma_{\Db(t)}(\Mb)=\Mb$.
Considering the symmetric matrix $\Ub(t) = \Exp_\J \big[ \Pi_\J \ThetaJb(t) \Pi_\J^\top \big]$, we next obtain 
\begin{eqnarray*}
\Exp_\J \big[ \trace \left( \ThetaJb(t) [\Db(t)]_\J^\top [\Wb_t \Sb(t)]_\J \right)  \big] &=&
\Exp_\J \big[ \trace \left( \Pi_\J \ThetaJb(t) \Pi_\J^\top [\Db(t)]^\top \Wb_t \Sb(t) \right)  \big] \\
&=&  \trace \left( \Ub(t) [\Db(t)]^\top \Wb_t \Sb(t) \right)  \\
&=&  \trace \left( (\Db(t) \Ub(t))^\top \Gamma_{\Db(t)}(\Wb_t \Sb(t)) \right)  \\
&=&  \trace \left( \Gamma_{\Db(t)}(\Db(t) \Ub(t)) \Wb_t \Sb(t) \right) \\
&\leq& t \|\Gamma_{\Db(t)}(\Db(t) \Ub(t))\|_\fro,
\end{eqnarray*}
where we have successively used the fact that $\Gamma_{\Db(t)}$ is self-adjoint and that for any 
$\Wb \in \Wcal_{\Db(t)}$, the norm $\| \Wb_t \Sb(t)  \|_\fro$ is upper bounded by $t$.

Observe that the $j$-th column of the matrix 
$\Gamma_j(t)\, \Db$ is equal to zero. As a consequence, 
we have
$$
 \|\Gamma_{\Db(t)}(\Db(t) \Ub(t))\|_\fro = \|\Gamma_{\Db(t)}(\Db(t) \Ub_\text{off}(t))\|_\fro,
$$
where $\Ub_\text{off}(t)$ denotes the matrix $\Ub(t)$ with its diagonal terms set to zero.
This leads to 
\begin{eqnarray*}
\|\Gamma_{\Db(t)}(\Db(t) \Ub(t))\|_\fro^2 &=& \|\Gamma_{\Db(t)}(\Db(t) \Ub_\text{off}(t))\|_\fro^2
=  \sum_{j=1}^p \| \Gamma_j(t)\, \Db(t) \ub_\text{off}^j\|_2^2\\
&\leq& \triple \Db(t) \triple_2^2 \sum_{j=1}^p \|\ub_\text{off}^j\|_2^2
= \triple \Db(t) \triple_2^2 \| \Exp_\J \big[ \Pi_\J \ThetaJb(t) \Pi_\J^\top \big]_\text{off}\|_\fro^2,
\end{eqnarray*}
where we have exploited the fact that projectors have their spectral norms bounded by one.
Using Corollary~\ref{cor:coherence_elementwise}, we have for $i\neq j$ with $i,j\in\SET{p}$
$$
| \big[ \Pi_\J \ThetaJb(t) \Pi_\J^\top \big]_{i,j} | \leq \delta(i)\delta(j) \frac{\mu(t)}{1-k\mu(t)}
$$
and 
$$
| \Exp_\J\big[ (\Pi_\J \ThetaJb(t) \Pi_\J^\top)_{i,j} \big] | \leq \frac{k(k-1)}{p(p-1)} \frac{\mu(t)}{1-k\mu(t)},
$$
hence
$$
 \| \Exp_\J \big[ \Pi_\J \ThetaJb(t) \Pi_\J^\top \big]_\text{off}\|_\fro^2 \leq 
 p(p-1) \Big(\frac{k(k-1)}{p(p-1)} \frac{\mu(t)}{1-k\mu(t)} \Big)^2 \leq \frac{k^2}{p^2} \frac{(k\mu(t))^2}{( 1-k\mu(t) )^2}.
$$
To recapitulate and putting all the pieces together, we obtain the following upper bound
\begin{eqnarray*}
\left|\Exp_{\J}\left\{ \trace \left( \big[ \ThetaJb(0) [\Dbo]_{\J}^\top - \ThetaJb(t) [\Db(t)_{\J}]^\top\big] [\Dbo]_\J 
\right)\right\}\right| &\leq& \frac{t^2}{2}\frac{k}{p} + t \triple \Db(t)\triple_2 \frac{k}{p} \frac{k\mu(t)}{1-k\mu(t)}\\
&\leq& t\frac{k}{p} \Big[  \frac{t}{2} + \triple \Db(t)\triple_2 \frac{k\mu(t)}{1-k\mu(t)} \Big].
\end{eqnarray*}
To conclude, we use Lemma~\ref{lem:orthoprojcomplement} to get $\triple \Db(t)\triple_2 \leq 2\triple \Dbo\triple_2$, 
and the fact that $\triple \Dbo\triple_2 \geq 1$.
\end{proof}
\begin{proof}[Proof of Lemma~\ref{lem:bias_expectation} - Equation~\eqref{eq:bias_expectation_signsign}]

We start by writting Equation~\eqref{eq:bias_expectation_signsign}
in the following integral form 
$$
\trace( \ThetaJb(t) - \ThetaJb(0) ) = \int_0^t \trace(\nabla_t\ThetaJb(\tau))d\tau,
$$
where the derivative is computed in Lemma~\ref{lem:derivative}, namely,
$$
\trace(\nabla_t\ThetaJb(t)) = 
-2\trace\Big(  \ThetaJb(t) [\nabla_t\Db(t)]_{\J}^\top \Db_{\J}(t) \ThetaJb(t)  \Big).
$$
Introducing the symmetric matrix $\Ub(t) = \Exp_\J \big[ \Pi_\J [\ThetaJb(t)]^2 \Pi_\J^\top \big]$, we next obtain
by linearity of the trace and the integral
$$
\Exp_\J[ \trace( \ThetaJb(t) - \ThetaJb(0) ) ] = -2\int_0^t \trace\Big( \Db(\tau) \Ub(\tau) [\nabla_t\Db(\tau)]^\top \Big) d\tau
\leq 2 t \max_{\tau\in[0,t]} \Big|\trace\Big( \Db(\tau) \Ub(\tau) [\nabla_t\Db(\tau)]^\top  \Big)\Big|.
$$

Noticing that we are (almost) in the same setting as that of the previous proof, 
we are going to make use again of the operator $\Gamma_{\Db(t)}$ in order to control the off-diagonal terms of $\Ub(t)$.
More precisely, since $\diag( [\nabla_t\Db(\tau)]^\top \Db(\tau) )=\zerob$ and $\|\nabla_t\Db(\tau)\|_\fro = 1$, 
the same reasoning as that followed in the previous proof leads to 
$$
\Exp_\J[ \trace( \ThetaJb(t) - \ThetaJb(0) ) ] \leq 2t \cdot  \max_{\tau\in[0,t]} \triple \Db(\tau) \triple_2 \cdot 
\|\Exp_\J \big[ \Pi_\J [\ThetaJb(\tau)]^2 \Pi_\J^\top \big]_\text{off} \|_\fro.
$$
Invoking Corollary~\ref{cor:coherence_elementwise}, we have for $i\neq j$ with $i,j\in\SET{p}$
$$
| \Exp_\J\big[ (\Pi_\J [\ThetaJb(\tau)]^2 \Pi_\J^\top)_{i,j} \big] | \leq \frac{k(k-1)}{p(p-1)} \frac{2\mu(t)}{(1-k\mu(t))^2},
$$
hence
$$
 \| \Exp_\J \big[ \Pi_\J [\ThetaJb(\tau)]^2 \Pi_\J^\top \big]_\text{off}\|_\fro^2 \leq 
 p(p-1) \Big(\frac{k(k-1)}{p(p-1)} \frac{2\mu(t)}{(1-k\mu(t))^2} \Big)^2 \leq \frac{k^2}{p^2} \frac{(2k\mu(t))^2}{( 1-k\mu(t) )^4},
$$
which gives the advertised conclusion.
\end{proof}
\section{Proof of Proposition~\ref{prop:exact_recovery}}\label{app:robustsignrecovery}

We begin by a few lemmata related to the considered optimization problem.
\begin{lemma}\label{lem:alphabound} 
Let $\J \subseteq \SET{p}$ and $\sb \in \{-1,0,1\}^{|\J|}$. 
Consider a dictionary $\Db \in \Real^{m\times p}$ such that $\Db_\J^\top\Db_\J$ is invertible.
Consider also the vector $\alphab \in \Real^p$ defined by
\[
\alphab=\binom{ [\Db_\J^\top\Db_\J]^{-1}[\Db_\J^\top \xb -\lambda \sb] }{\zerob_{\J^{c}}},
\]
with $\xb \in \Real^m$ and $\lambda$ a nonnegative scalar.
If $\xb=[\Dbo]_\J [{\alphabo}]_\J + \varepsilonb$ for some $(\Dbo,\alphabo,\varepsilonb)\in \Real^{m\times p} \times \Real^p \times \Real^m$, 
then we have
\[
\|[\alphab-\alphabo]_\J\|_\infty \leq \triple [\Db_\J^\top\Db_\J]^{-1} \triple_\infty 
\Big[ \lambda + \| \Db_\J^\top\left(\xb-\Db \alphabo\right) \|_\infty \Big].
\]
\end{lemma}
\begin{proof}
The proof consists of simple algebraic manipulations. We plug the expression of $\xb$ into that of $\alphab$, then use the triangle inequality for $\|.\|_\infty$, along with the definition and the sub-multiplicativity of $\triple.\triple_\infty$.
\end{proof}

\begin{lemma}\label{lem:uniquelasso}
Let $\xb \in \Real^m$ be a signal. Consider $\J \subseteq \SET{p}$ and a dictionary $\Db \in \Real^{m\times p}$ such that $\Db_\J^\top\Db_\J$ is invertible.
Consider also a sign vector $\sb \in \{-1,1\}^{|\J|}$ and define $\hat{\alphab} \in \Real^p$ by
\[
\hat{\alphab}=\binom{ [\Db_\J^\top\Db_\J]^{-1}[\Db_\J^\top \xb -\lambda \sb] }{\zerob_{\J^{c}}},
\]
for some regularization parameter $\lambda\geq0$.
If the following two conditions hold
\[
\begin{cases}
\sign\Big( [\Db_\J^\top\Db_\J]^{-1}[\Db_\J^\top \xb -\lambda \sb]  \Big) = \sb,\\
\| \Db_{\J^c}^\top (\Ib - \Pb_\J) \xb \|_\infty + \lambda \triple \Db_{\J^c}^\top \Db_\J [\Db_\J^\top\Db_\J]^{-1} \triple_\infty < \lambda,
\end{cases}
\]
then
$\hat{\alphab}$ is the unique solution of $\min_{\alphab \in \Real^p} [\frac{1}{2}\|\xb-\Db\alphab\|_2^2+\lambda\|\alphab\|_1]$
and we have $\sign(\hat{\alphab}_\J)=\sb$.
\end{lemma}
\begin{proof}
We first check that $\hat{\alphab}$ is a solution of the Lasso program.
It is well-known \citep[e.g., see][]{Fuchs2005,Wainwright2009} that this statement is equivalent to the existence of a subgradient $\zb \in \partial \|\hat{\alphab}\|_1$
such that $-\Db^\top(\xb-\Db\hat{\alphab})+\lambda\zb=0$, where $\zb_j=\sign(\hat{\alphab}_j)$ if $\hat{\alphab}_j \neq 0$, and $|\zb_j|\leq 1$ otherwise.

We now build from $\sb$ such a subgradient. Given the definition of $\hat{\alphab}$ and the assumption made on its sign, we can take $\zb_\J\defin\sb$.
It now remains to find a subgradient on $\J^c$ that agrees with the fact that $\hat{\alphab}_{\J^c}=\zerob$.
More precisely, we define $\zb_{\J^c}$ by
$$
\lambda \zb_{\J^c} \defin \Db_{\J^c}^\top(\xb-\Db\hat{\alphab}) = 
\Db_{\J^c}^\top (\Ib - \Pb_\J) \xb + \lambda \Db_{\J^c}^\top \Db_\J [\Db_\J^\top\Db_\J]^{-1} \sb.
$$
Using our assumption, we have $\|\zb_{\J^c}\|_\infty < 1$.
We have therefore proved that $\hat{\alphab}$ is a solution of the Lasso program. The uniqueness comes from Lemma~1 in \citet{Wainwright2009}.
\end{proof}

\begin{corollary}\label{cor:robustlassomin}
Assume that $k\mu(t) \leq 1/2$, $0 \leq t' \leq t$,  $\frac{9}{4}\lambda \leq \loweralpha \leq \min_{j \in \J}|[\alphabo]_{j}|$, and that 
\begin{eqnarray}
\label{eq:Cor3Bound1}
\|[\Db(t')]_{\J}^{\top} \left(\xb-\Db(t') \alphabo \right)\|_{\infty} 
&< &
\lambda(2-Q_{t}^{2})\\
\label{eq:Cor3Bound2}
\| [\Db(t')]_{\J^c}^\top (\Ib - \PJb(t')) \xb \|_\infty 
&<&  \lambda(2-Q_{t}^{2})
\end{eqnarray}
Then $\hat{\alphab}(t')$ is the unique solution of $\min_{\alphab \in \Real^p} [\frac{1}{2}\|\xb-\Db(t')\alphab\|_2^2+\lambda\|\alphab\|_1]$
\end{corollary}
\begin{proof}
Since $k\mu(t)\leq1/2$, we have $Q_{t}^{2} \leq 2$, and by Corollary~\ref{cor:coherence_prop2} we have, uniformly for all $(\Wb,\vb)$ and $0 \leq t' \leq t$
\begin{eqnarray*}
\triple \big[[\Db(t')]_\J^\top [\Db(t')]_\J \big]^{-1} \triple_\infty
&\leq &
Q_{t}^{2}
\\
 \triple [\Db(t')]_{\J^c}^\top [\Db(t')]_\J \big( [\Db(t')]_\J^\top[\Db(t')]_\J \big)^{-1} \triple_\infty 
&\leq& Q_{t}^{2}-1 \leq 1
\end{eqnarray*}
Exploiting Lemma~\ref{lem:alphabound} and the bound~\eqref{eq:Cor3Bound1} we have 
\begin{eqnarray*}
\|[ \hat{\alphab}(t') - \alphabo ]_\J\|_\infty &\leq&    
\triple \big[[\Db(t')]_\J^\top [\Db(t')]_\J \big]^{-1} \triple_\infty 
\Big[ \lambda + \|[\Db(t')]_\J^\top(\xb-\Db(t')\alphabo) \|_\infty\Big] \\
& < &
Q_{t}^{2} \cdot \lambda \cdot \left[1+(2-Q_{t}^{2})\right] 
= 
\lambda \cdot Q_{t}^{2} \cdot (3-Q_{t}^{2})
 \leq  \frac{9}{4} \lambda  \leq \loweralpha \leq \min_{j\in\J} \big|[\alphabo]_j\big|,
\end{eqnarray*}
where we used that $u(3-u) \leq 9/4$ for all $u \in \Real$. We conclude that $\sign( \hat{\alphab}(t') ) = \sign(\alphabo )$. 

It remains to prove that $\hat{\alphab}(t')$ is the unique solution of the Lasso program.
To this end, we take advantage of Lemma~\ref{lem:uniquelasso}. We recall the quantity which needs to be smaller than $\lambda$
\[
\| [\Db(t')]_{\J^c}^\top (\Ib - \PJb(t')) \xb \|_\infty + \lambda \triple [\Db(t')_{\J^c}]^\top [\Db(t')]_\J \big( [\Db(t')]_\J^\top [\Db(t')]_\J\big)^{-1} \triple_\infty.
\]
The quantity above is first upper bounded by
\[
\| [\Db(t')]_{\J^c}^\top (\Ib - \PJb(t')) \xb \|_\infty + \lambda (Q_{t}^{2}-1),
\]
and then, exploiting the bound~\eqref{eq:Cor3Bound2},  \emph{strictly} upper bounded by $\lambda(2-Q_{t}^{2}) + \lambda (Q_{t}^{2}-1) = \lambda.$
Putting together the pieces with $\sign( \hat{\alphab}(t') ) = \sign(\alphabo )$, Lemma~\ref{lem:uniquelasso} leads to the desired conclusion.

\end{proof}

We can now proceed to the proof of Proposition~\ref{prop:exact_recovery}. Since $\|\db^{j}(t')\|_{2}=1$ for all $j$, we have
\begin{eqnarray}
\label{eq:Prop2Bound1}
\|[\Db(t')]_{\J}^{\top} (\xb-\Db(t') \alphabo)\|_{\infty} 
&\leq &
\|\xb-\Db(t') \alphabo\|_{2}\\
\label{eq:Prop2Bound2}
\| [\Db(t')]_{\J^c}^\top (\Ib - \PJb(t')) \xb \|_\infty 
&\leq&  \|(\Ib - \PJb(t')) \xb  \|_2%\\
\end{eqnarray}
Using Lemma~\ref{lem:l2norm_signal}, provided that
\[
\tau = \frac{\lambda^{2} (2-Q_{t}^{2})^{2}}{5 \cdot (t'^{2} \cdot \sigma_{\alpha}^{2} + m \cdot \sigma^{2})} \geq  1
\]
we have
\begin{eqnarray*}
\Pr\left( \|\xb-\Db(t') \alphabo \|_2 \geq \lambda (2-Q_{t}^{2})\right)
& = &
\Pr\left( \|\xb-\Db(t') \alphabo \|_2^2 \geq 5 (t'^{2} \cdot \sigma_{\alpha}^2 +m \cdot \sigma^{2})\tau \right) \leq \exp(-\tau)\\
\Pr\left( \|(\Ib-\PJb(t'))\xb \|_2 \geq \lambda (2-Q_{t}^{2})\right)
& = &
\Pr\left( \|(\Ib-\PJb(t'))\xb \|_2^2 \geq 5 (t'^{2} \cdot \sigma_{\alpha}^2 +m \cdot \sigma^{2})\tau \right) \leq \exp(-\tau)
\end{eqnarray*}
With a union bound, we conclude that $\|\xb-\Db(t') \alphabo \|_2 < \lambda (2-Q_{t}^{2})$ and $\|(\Ib-\PJb(t'))\xb \|_2 \geq \lambda (2-Q_{t}^{2})$, except with probability at most
\[
\Pr\left(\Ecal_{\text{coincide}}^{c}(t')\right)  
\leq 
2 \cdot \exp(-\tau) = 2 \cdot \exp\left(- \frac{\lambda^{2} (2-Q_{t}^{2})^{2}}{5 \cdot (t'^{2} \cdot \sigma_{\alpha}^{2} + m \cdot \sigma^{2})}\right).
\]

\section{Proof of Proposition~\ref{prop:simplified_exact_recovery}}

We now consider the proof of Proposition~\ref{prop:simplified_exact_recovery} whose structure is identical to that of Proposition~\ref{prop:exact_recovery}.
We recall that we are in noiseless setting, i.e., $\sigma=0$, and we assume that the coefficients of $\alphabo$ are almost surely bounded by $\upperalpha$.

In the light of Lemma~\ref{lem:orthoprojcomplement}, let us first observe that almost surely
$$
\|[\Db(t')]_{\J}^{\top} (\xb-\Db(t') \alphabo)\|_{\infty} \leq
\|\xb-\Db(t') \alphabo\|_{2} \leq \| [\Dbo-\Db(t')]_\J [\alphabo]_\J \|_{2} \leq t \cdot \| [\alphabo]_\J \|_{2} \leq \sqrt{k} \upperalpha t.
$$ 
Similarly, it follows
$$
\| [\Db(t')]_{\J^c}^\top (\Ib - \PJb(t')) \xb \|_\infty 
\leq  \|(\Ib - \PJb(t')) \xb  \|_2 \leq \sqrt{k} \upperalpha t.
$$
Now, we can apply Corollary~\ref{cor:robustlassomin} provided that $\sqrt{k} \upperalpha t \leq \lambda (2-Q_{t}^2)$, 
as required by Proposition~\ref{prop:simplified_exact_recovery}. This leads to the desired conclusion. 
\section{Proof of Proposition~\ref{prop:control_rn}}\label{sec:proof_control_residual}
Exploiting Proposition~\ref{prop:exact_recovery} we have
\begin{equation}
\label{eq:DefKappa}
\max_{i \in \SET{n}}\Pr([\Ecal^{i}_\mathrm{coincide}(t) \cup \Ecal^{i}_\mathrm{coincide}(0)]^{c}) 
\leq 
4 
\exp\left(-\frac{[\lambda (2-Q_{t}^{2})]^{2}}{5 (t^{2} \sigma_{\alpha}^{2} + m\sigma^{2})}\right)
\defin \kappa.
\end{equation}
The assumption~\eqref{eq:AssumptionKT2} is equivalent to
\[
\frac{3}{2-Q_{t}^{2}} \cdot t \cdot \sigma_{\alpha}  < \frac{4}{9} \loweralpha
\]
hence there exists indeed $\sigma > 0$ and $\lambda$ satisfying the assumption~\eqref{eq:BoundsLambda}. Moreover, since $5 \log 4 \approx  6.93 \leq 9$, the assumption~\eqref{eq:BoundsLambda} implies that
\[
\frac{\gamma^{2}}{\log4} = \frac{\lambda^{2} (2-Q_{t}^{2})^{2}}{5 \log 4 \cdot (t^{2}\sigma_{\alpha}^{2} + m\sigma^{2})} \geq 1,
\] 
hence $\gamma^{2} \geq \log 4$ and $\kappa = 4 \cdot e^{-\gamma^{2}} \leq 1$. Therefore, we can exploit Lemma~\ref{lem:suprxi} and Corollary~\ref{cor:concentration_subgaussian}. Given~\eqref{eq:residual_lowerbound_fn}, with 
\[
A_{r} \defin\left(t^{2} \cdot \sigma_{\alpha}^{2} +2m \cdot \sigma^{2}+ 2\lambda k \cdot \sigma_{\alpha}\right) \cdot  \frac{5(1+\log 2)}{2}
\]
we have, except with probability at most $\exp(-n\kappa) = \exp(-4n e^{-\gamma^{2}})$, 
\begin{eqnarray*}
r_{n} 
&\leq&
10 A_{r} \cdot (3-\log \kappa)  \cdot \kappa  
=
A \cdot 10 \cdot (3-\log 4 + \gamma^{2})  \cdot 4 e^{-\gamma^{2}} \\ 
&\leq&
\left(t^{2} \cdot \sigma_{\alpha}^{2} + 2m \cdot \sigma^{2} + 2\lambda k\sigma_{\alpha}\right)
\cdot 10(1+\log 2) \cdot 10 \cdot \frac{3}{\log 4} \cdot \gamma^{2} \cdot e^{-\gamma^{2}}\\
&\leq&
\left(t^{2} \cdot \sigma_{\alpha}^{2} + 2m \cdot \sigma^{2} + 2\lambda k\sigma_{\alpha}\right)
\cdot 367 \cdot \gamma^{2} \cdot e^{-\gamma^{2}}.
\end{eqnarray*}
\section{Technical lemmas}
The final section of this appendix gathers  technical lemmas required by the main results of the paper.
\subsection{Control on the differences of operators}

We will now establish several lemmata regarding the difference of operators that appear in the paper.

The following result will exploit Taylor formula with remainder, based on  simple matrix and vector derivative computations of $\Db(\Wb,\vb,t)$; we refer the interested reader to~\citet{Magnus1988} for details about such manipulations.
For convenience, let us define 
\begin{eqnarray}
\Cb(t) &\defin& \Diag(\cos(\vb_{j}t))\\
\Sb(t) &\defin& \Diag(\sin(\vb_{j}t))\\
\Vb    & \defin& \Diag(\vb_{j})\\
\RJb(t) &\defin& \Db_{\J}(t) \ThetaJb(t) [\nabla_t\Db(t)]_{\J}^\top.
\end{eqnarray}
and denote the symmetric part of a square matrix $\Mb$ by $\sym(\Mb) \defin \frac{1}{2}( \Mb + \Mb^\top)$.

\begin{lemma}%
\label{lem:derivative}
\begin{eqnarray}
\label{eq:derivative_0_t}
\nabla_{t} \Db(t) &=& (-\Dbo \Sb(t) + \Wb \Cb(t)) \Vb\\ % = \Qb(t)\Vb\\
\|[\nabla_{t} \Db(t)]_{\J}\|_{\fro} & = & \|\vb_{\J}\|_{2}\\
\label{eq:derivative_1_t}
\nabla_t \PJb(t) &=& 2\sym\big( \RJb(t) (\Ib - \PJb(t)) \big)\\
\label{eq:derivative_2_t}
\nabla_t [\ThetaJb(t) \Db_{\J}^\top(t)] &=& 
\ThetaJb(t) ( [\nabla_t\Db(t)]_{\J}^\top (\Ib - \PJb(t)) - [\Db(t)]_{\J}^\top [\RJb(t)]^\top)\\
\label{eq:derivative_3_t}
\nabla_t [\ThetaJb(t)] &=& 
-2\sym\Big(  \ThetaJb(t)      [\nabla_t\Db(t)]_{\J}^\top \Db_{\J}(t)  \ThetaJb(t)  \Big). 
\end{eqnarray}
\end{lemma}

\begin{lemma}
\label{lem:diff_norm_oper}
Assume $t < \sqrt{1-\delta_{k}(\Dbo)}$, then for any $\Wb \in \Wcal_{\Dbo}$, $\vb \in \Scal^{p}$ and $\J$ with $|\J| \leq k$ we have
\begin{eqnarray}
\label{eq:norm_deltaproj}
\triple \PJb(t)-\PJb(0) \triple_2
\leq 
\| \PJb(t)-\PJb(0) \|_\fro 
& \leq &
 2t \cdot C_{t} \cdot \|\vb_{\J}\|_{2},\\
\label{eq:norm_deltainvdict}
\triple \ThetaJb(t) [\Db]_{\J}^\top(t) - \ThetaJb(0) [\Dbo]_{\J}^\top\triple_2 
\leq
\| \ThetaJb(t) [\Db]_{\J}^\top(t) - \ThetaJb(0) [\Dbo]_{\J}^\top\|_{\fro}
&\leq& 2t \cdot C_{t}^{2} \cdot \|\vb_{J}\|_{2},\\
\label{eq:norm_deltainv}
\triple \ThetaJb(t)-\ThetaJb(0) \triple_2
\leq
\| \ThetaJb(t)-\ThetaJb(0) \|_2
& \leq &
2t \cdot C_{t}^{3} \cdot \|\vb_{\J}\|_{2}.
\end{eqnarray}
\end{lemma}
\begin{lemma}\label{lem:BoundDeltaPJThetaJ}
Assume that $t < \sqrt{1-\delta_{k}(\Dbo)}$. Denote $\Pb_{\J,1} = \PJb(\Wb_{1},\vb_{1},t)$ and $\Pb_{\J,2} = \PJb(\Wb_{2},\vb_{2},t)$ and similarly for the other considered quantities. 
For any $\Wb_{1},\Wb_{2} \in \Wcal_{\Dbo}$, $\vb_{1},\vb_{2} \in \Scal_{+}^{p}$, and $\J$ with $|\J| \leq k$ we have 
\begin{eqnarray}
\triple \Db_{\J,1}-\Db_{\J,2} \triple_{2} \leq \| \Db_{\J,1}-\Db_{\J,2} \|_{\fro}
&\leq& 2t \cdot C_t \cdot d\big((\Wb_{1},\vb_{1}),(\Wb_{2},\vb_{2})\big)\\
\triple \Pb_{\J,1}-\Pb_{\J,2} \triple_{2} \leq \| \Pb_{\J,1}-\Pb_{\J,2} \|_{\fro}
&\leq& 5t \cdot C_{t} \cdot d\big((\Wb_{1},\vb_{1}),(\Wb_{2},\vb_{2})\big)\\
\triple \Thetab_{\J,1}[\Db_{1}]_{\J}^{\top}-\Thetab_{\J,2}[\Db_{2}]_{\J}^{\top}\triple_{2} 
\leq \| \Thetab_{\J,1}[\Db_{1}]_{\J}^{\top}-\Thetab_{\J,2}[\Db_{2}]_{\J}^{\top}\|_{\fro}
&\leq& 5t \cdot C_{t} \cdot d\big((\Wb_{1},\vb_{1}),(\Wb_{2},\vb_{2})\big)\\
\triple \Thetab_{\J,1}-\Thetab_{\J,2}\triple_{2} \leq \|\Thetab_{\J,1}-\Thetab_{\J,2}\|_{\fro}
&\leq& 5t \cdot C_{t}^{3} \cdot d\big((\Wb_{1},\vb_{1}),(\Wb_{2},\vb_{2})\big).
\end{eqnarray}
\end{lemma}
\begin{proof}[Proof of Lemma~\ref{lem:diff_norm_oper}-Equation~\eqref{eq:norm_deltaproj}]
We apply a Taylor formula with remainder~\citep[e.g., Theorem 14.4 in][]{Dym2007} based on Lemma~\ref{lem:derivative} (Equation~\eqref{eq:derivative_1_t}):
for any $\Ub \in \Real^{m\times m}$ there exists $0 \leq t' = t'(\Ub) \leq t$ such that
\[
\trace\Big( \Ub \cdot (\PJb(t) - \PJb(0) ) \Big)
= 2t \cdot \trace \Big( \Ub \cdot \sym\big( \RJb(t') (\Ib - \PJb(t')) \big) \Big) 
\leq 2t \cdot \| \RJb(t') (\Ib - \PJb(t'))  \|_\fro  \cdot  \|\Ub\|_\fro.
\]
Given that $\| [\nabla \Db(t')]_{\J} \|_\fro = \|\vb_{\J}\|_{2}$, we have using the bound~\eqref{eq:BoundDThetaJ}
\begin{equation}
\label{eq:RJBound}
\|\RJb(t')\|_{\fro} 
\leq \triple \Db_{\J}(t')\ThetaJb(t') \triple_2 \cdot \| [\nabla \Db(t')]_{\J} \|_\fro \leq C_{t} \cdot \|\vb_{\J}\|_{2},
\end{equation}
hence the upper bound
\begin{eqnarray*}
\trace\Big( \Ub \cdot (\PJb(t) - \PJb(0) ) \Big)
 & \leq &
2t  \cdot \| \RJb(t')\|_\fro  \cdot  \|\Ub\|_\fro 
\leq 2t \cdot C_{t} \cdot \|\vb_{\J}\|_{2} \cdot  \|\Ub\|_\fro.
\end{eqnarray*}
We conclude using the fact that $\triple \PJb(t) - \PJb(0)  \triple_{2} \leq \| \PJb(t) - \PJb(0)  \|_\fro =\max_{ \|\Ub\|_\fro \leq 1} \trace( \Ub^\top (\PJb(t) - \PJb(0) ) )$,
\end{proof}

\begin{proof}[Proof of Lemma~\ref{lem:diff_norm_oper}-Equation~\eqref{eq:norm_deltainvdict}]
Again, we apply a Taylor formula with remainder and
Lemma~\ref{lem:derivative} (Equation~\eqref{eq:derivative_2_t}): for any $\Ub \in \Real^{m \times p}$, there exists some $0 \leq t' \leq t$ such that
\begin{eqnarray*}
\trace \left(\Ub (\ThetaJb(t) \Db_{\J}^\top(t) - \ThetaJb(0) [\Dbo]_{\J}^\top )\right) 
&=&
 t \cdot
\trace \left( \Ub
\Big[ \ThetaJb(t')
\big( [\nabla_t\Db(t')]_{\J}^\top (\Ib - \PJb(t')) - [\Db(t')]_{\J}^\top [\RJb(t')]^\top \big)  
\Big] \right)\\
&\leq&
t  \cdot \| \ThetaJb(t')
\big( [\nabla_t\Db(t')]_{\J}^\top (\Ib - \PJb(t')) - [\Db(t')]_{\J}^\top [\RJb(t')]^\top \big) 
\|_{\fro} \cdot \|\Ub\|_{\fro}
\end{eqnarray*}
Now, using the bounds~\eqref{eq:BoundThetaJ},~\eqref{eq:BoundDThetaJ} and~\eqref{eq:RJBound} we have
\begin{eqnarray*}
\| \ThetaJb(t')
\big( [\nabla_t\Db(t')]_{\J}^\top (\Ib - \PJb(t'))\|_{\fro}
&\leq &
\triple \ThetaJb(t')\triple_{2} \cdot %\triple (\Ib - \PJb(t'))\triple_{2} \cdot
\| [\nabla_t\Db(t')]_{\J}\|_{\fro} 
\leq C_{t}^{2} \cdot \|\vb_{\J}\|_{2},\\
\| \ThetaJb(t')[\Db(t')]_{\J}^\top [\RJb(t')]^\top\|_{\fro}
& \leq &
\triple\ThetaJb(t')[\Db(t')]_{\J}^\top\triple_{2} \cdot \| \RJb(t')\|_{\fro}
\leq
C_{t} \cdot (C_{t} \cdot \|\vb_{\J}\|_{2}) \leq C_{t}^{2 }\cdot \|\vb_{\J}\|_{2}
\end{eqnarray*}
and we can conclude.
\end{proof}

\begin{proof}[Proof of Lemma~\ref{lem:diff_norm_oper}-Equation~\eqref{eq:norm_deltainv}]
We follow the same line, using the intermediate result from Lemma~\ref{lem:derivative} (Equation~\eqref{eq:derivative_3_t}).  For any $\Ub \in \Real^{p\times p}$ there is some $0 \leq t' = t'(\Ub) \leq t$ such that
\begin{eqnarray*}
\left|\trace\left(\Ub \cdot (\ThetaJb(t) - \ThetaJb(0) )\right)\right|
&=& \left|2t \cdot \trace\left(\Ub \cdot \sym\big( \ThetaJb(t') [\nabla_t\Db(t')]_{\J}^\top \Db_{\J}(t')  \ThetaJb(t')\big)  \right)\right| \\
&\leq& 2t \cdot \| \ThetaJb(t') [\nabla_t\Db(t')]_{\J}^\top \Db_{\J}(t')  \ThetaJb(t') \|_{\fro} \cdot
\|\Ub\|_\fro.
\end{eqnarray*}
Since $\| [\nabla \Db(t')]_{\J} \|_\fro = \|\vb_{\J}\|_{2}$, using~\eqref{eq:BoundDThetaJ} and~\eqref{eq:BoundThetaJ} we obtain the upper bound
\begin{eqnarray*}
2t\cdot \|   \ThetaJb(t')      [\nabla_t\Db(t')]_{\J}^\top \Db_{\J}(t')  \ThetaJb(t')   \|_{\fro} 
&\leq & 
2t \cdot \triple \Db_{\J}(t')  \ThetaJb(t')   \triple_2
\cdot \triple \ThetaJb(t')  \triple_2
\cdot \|  [\nabla_t\Db(t')]_{\J}  \|_{\fro} \\
&\leq & 
2t  \cdot C_{t} \cdot C_{t}^{2} \cdot \|\vb_{\J}\|_{2}.
\end{eqnarray*}
\end{proof}
\begin{proof}[Proof of Lemma~\ref{lem:BoundDeltaPJThetaJ}]
Since
\(
d\left((\Wb_{1},\vb_{1}),(\Wb_{2},\vb_{2})\right) = \max [\max_{j} \|\wb_{1}^{j}-\wb_{2}^{j}\|_{2}, \|\vb_{1}-\vb_{2}\|_{2}] \leq \varepsilon,
\)
we can bound the difference between the columns of $\Db_{i} = \Db(\Dbo,\Wb_{i},\vb_{i},t)$, $i=1,2$:
\begin{eqnarray*}
\db_{1}^{j}-\db_{2}^{j}
&=&
(\cos (\vb_{1}^jt) -\cos (\vb_{2}^jt)) \cdot \dbo^{j} + \sin (\vb_{1}^jt) \cdot \wb_{1}^{j} - \sin (\vb_{2}^jt) \cdot \wb_{2}^{j}\\
\|\db_{1}^{j}-\db_{2}^{j}\|_{2}^{2}
&=&
(\cos (\vb_{1}^jt) -\cos (\vb_{2}^jt))^{2} + \|\sin (\vb_{1}^jt) \cdot \wb_{1}^{j} - \sin (\vb_{2}^jt) \cdot \wb_{2}^{j}\|_{2}^{2}\\
&=&
\cos^{2} (\vb_{1}^jt) + \cos^{2} (\vb_{2}^jt) -2 \cos (\vb_{1}^jt) \cos (\vb_{2}^jt) \\
&& + \sin^{2} (\vb_{1}^jt) + \sin^{2} (\vb_{2}^jt) -2 \sin (\vb_{1}^jt) \sin (\vb_{2}^jt) [\wb_{1}^{j}]^{\top}\wb_{2}^{j}\\
&=&
2-2 \cos (\vb_{1}^jt) \cos (\vb_{2}^jt)-[2- \|\wb_{1}^j-\wb_{2}^{j}\|_{2}^{2}] \sin (\vb_{1}^jt) \sin (\vb_{2}^jt)\\
&=&
\|\wb_{1}^j-\wb_{2}^{j}\|_{2}^{2} \cdot \sin (\vb_{1}^jt) \sin (\vb_{2}^jt) + 4\sin^{2}\frac{(\vb_{1}^j-\vb_{2}^j)t}{2} 
\leq 
\left(\varepsilon^{2} \vb_{1}^j \vb_{2}^j + (\vb_{1}^j-\vb_{2}^j)^{2}\right)t^{2}
\end{eqnarray*}
As a result we obtain
\[
\|\Db_{1}-\Db_{2}\|_{\fro}^{2} = \sum_{j=1}^{p} \|\db_{1}^{j}-\db_{2}^{j}\|_{2}^{2} \leq \left(\varepsilon^{2} \vb_{1}^{\top} \vb_{2} + \|\vb_{1}-\vb_{2}\|_{2}^{2}\right) t^{2} \leq 2 \varepsilon^{2} t^{2}.
\]
Exploiting Lemma~\ref{lem:paramonto}, we can write $\Db_{2} = \Db(\Db_{1},\Wb,\vb,t')$ with $t' \leq \frac{\pi}{2} \|\Db_{1}-\Db_{2}\|_{\fro} \leq \frac{\pi}{\sqrt{2}} \varepsilon t$. 

Now consider $\Db(\tau) \defin  \Db(\Db_{1},\Wb,\vb,\tau)$ with $0 \leq \tau \leq t'$ and $\db^{j}(\tau)$ its columns. Noticing that $\tau \mapsto \db^{j}(\tau)$ is a geodesic on the unit sphere that joins $\db^{j}(0) = \db^{j}_{1}$ to $\db^{j}(t') = \db^{j}_{2}$, 
we obtain
\[
\|\db^{j}(\tau) -\dbo^{j}\|_{2} \leq \max(\|\db^{j}_{1} -\dbo^{j}\|_{2},\|\db^{j}_{2} -\dbo^{j}\|_{2}) = 2 \sin \left(\frac{t\vb_{j}}{2}\right).
\]
Hence, exploiting Lemma~\ref{lem:paramonto} again, we can also write $\Db(\tau) = \Db(\Dbo,\Wb',\vb',\tau')$, with $\tau' \leq t$. This implies that for every dictionary on the curve $\tau \mapsto \Db(\tau)$, $0 \leq \tau \leq t'$, the bounds of Lemma~\ref{lem:RIPBounds} with the constant $C_{t}$ hold true. We can therefore repeat the Taylor argument of the proof of Lemma~\ref{lem:diff_norm_oper}, noticing that since the considered end point is at $t' \leq \frac{\pi}{\sqrt{2}}\varepsilon t$ instead of $t$, the factor $2t$ in the resulting bounds is replaced by $2t' \leq \pi\sqrt{2} \varepsilon t \leq 5 \varepsilon t$.
\end{proof}
\subsection{Control of  norms}

In this section, we first recall some known concentration results.
\begin{lemma}[From~\citet{Hsu2011}]\label{lem:onesided_quadratictail}
Let us consider $\zb \in \Real^m$ a random vector of independent sub-Gaussian variables with parameters
upper bounded by $\sigma>0$. Let $\Ab \in \Real^{m\times p}$ be a fixed matrix. For all $\tau>0$, it holds
\[
\Pr\Big(
\|\Ab\zb\|_2^2 > \sigma^2( \|\Ab\|_\fro^2 +2\sqrt{ \trace[(\Ab^\top\!\Ab)^2] \tau} + 2\triple \Ab^\top\!\Ab \triple_2 \tau )
\Big) \leq \exp(-\tau).
\]
In particular, for any $\tau \geq 1$, we have 
\[
\Pr\Big(
\|\Ab\zb\|_2^2 > 5\sigma^2 \|\Ab\|_\fro^2 \tau 
\Big) \leq \exp(-\tau).
\]
\end{lemma}
\begin{lemma}[Bernstein's Inequality]\label{lem:bernstein}
Let $\{z_j\}_{j\in\SET{n}}$ be a collection of independent, zero-mean random variables. 
If there exist $M, \varsigma \in \Real_+$ such that for any integer $k \geq 2$ and any $j\in\SET{n}$,
it holds
\[
\Exp[ |z_j|^k ] \leq \frac{k!}{2} M^{k-2} \varsigma^2, 
\]
then we have for any $\tau \geq 0$,
\[
\Pr\Big( \sum_{j=1}^n z_j > \tau \Big)\leq 
\exp\Big( -\frac{\tau^2}{2 n \varsigma^2 + 2 M \tau} \Big).
\]
In particular, for any $\tau \leq \frac{\varsigma \sqrt{n}}{2M}$, we have
\[
\Pr\Big( \frac{1}{n}\sum_{j=1}^n z_j > 2\varsigma\frac{\tau}{\sqrt{n}} \Big)\leq 
\exp\big( -\tau^2 \big).
\]
\end{lemma}
\begin{proof}
The displayed result is a straightforward adaptation of Lemma 4.1.9 in \citet{Del1999}, where
we use the term $\varsigma^2$ in lieu of the true variance.
\end{proof}
\begin{lemma}[Control of the $\ell_2$-norm of a signal and its coefficients]\label{lem:l2norm_signal}
Let $\xb$ be a signal following our generative model, and $\alphabo$ be its coefficients.
For any $\tau \geq 1$ and $\Db = \Db(\Wb,\vb,t)$, we have
\begin{eqnarray}
\Pr\left( \|\xb-\Db \alphabo \|_2^2+\|\varepsilonb\|_{2}^{2} > 5 (t^{2}  \sigma_{\alpha}^2 +2m  \sigma^{2})\tau \right) &\leq& \exp(-\tau)\\
\Pr\left( \|\xb-\Db \alphabo \|_2^2 > 5 (t^{2}  \sigma_{\alpha}^2 +m  \sigma^{2})\tau \right) &\leq& \exp(-\tau)\\
\Pr\left( \|(\Ib-\PJb(t))\xb \|_2^2 > 5 (t^{2}  \sigma_{\alpha}^2 +(m-k)  \sigma^{2})\tau \right) &\leq& \exp(-\tau)\\
\Pr\left( \|\alphabo\|_2^2 > 5 k\sigma_{\alpha}^2 \tau \right) &\leq& \exp(-\tau)\\
\Pr\left( \|\xb\|_2^2 > 5(k\sigma_{\alpha}^2 +m\sigma^2 ) \tau \right) &\leq& \exp(-\tau)
\end{eqnarray}
\end{lemma}
\begin{proof}
We prove the result for $\|\xb-\Db \alphabo \|_2^2+\|\varepsilonb\|_{2}^{2}$. The same technique applies to the other quantities. We recall that $\xb-\Db\alphabo=[\Dbo-\Db]_\J [\alphabo]_\J+\varepsilonb$, and that the considered norm can be expressed as follows 
\[
\|\xb-\Db \alphabo \|_2^2+\|\varepsilonb\|_{2}^{2}
=
\left\|
\left[
\begin{array}{cc}
\sigma_{\alpha} [\Dbo-\Db]_\J & \sigma\Ib\\
\mathbf{0} & \sigma\Ib
\end{array}
\right]  \binom{\frac{1}{\sigma_{\alpha}}[\alphabo]_\J}{\frac{1}{\sigma}\varepsilonb} 
\right\|_2.
\]
The result is a direct application of Lemma~\ref{lem:onesided_quadratictail} conditioned to the draw of $\J$, using Lemma~\ref{lem:orthoprojcomplement} to control
\[
\left\|
\left[
\begin{array}{cc}
\sigma_{\alpha} [\Dbo-\Db]_\J & \sigma\Ib\\
\mathbf{0} & \sigma\Ib
\end{array}
\right]  
\right\|_\fro^{2}
= \|[\Dbo-\Db]_\J\|_{\fro}^{2} \cdot \sigma_{\alpha}^{2}  + 2m  \sigma^{2} 
\leq t^{2}  \sigma_{\alpha}^{2} + 2m  \sigma^{2}.
\]
The bound being independent of $\J$, the result is also true without conditioning. Note that to control the behaviour of $\|(\Ib-\PJb(t))\xb \|_2^2$ we use the fact that since $\PJb(t)$ is an orthogonal projector on a subspace of dimension $k$, we have $\|\Ib-\PJb(t)\|_{\fro}^{2} = m-k$.
\end{proof}
\begin{lemma}\label{lem:suprxi}
Let $\xb$ and $\alphabo$ be drawn according to our signal model. Define
\begin{eqnarray*}
y  &=& \sup_{\Wb,\vb} \Lcal_{\xb}(\Db(\Wb,\vb,t),\alphabo)\\
y' &=& \sup_{\Wb,\vb} \left\{\Lcal_{\xb}(\Db(\Wb,\vb,t),\alphabo) + \Lcal_{\xb}(\Dbo,\alphabo)\right\}
\end{eqnarray*}
For any $\tau \geq 1$ we have
\begin{eqnarray}
\Pr(y \geq A_{\Lcal}(t) \cdot \tau) &\leq& e^{-\tau}\\
\Pr(y' \geq A_{r}(t) \cdot \tau) &\leq& e^{-\tau}
\end{eqnarray}
where
\begin{eqnarray*}
A_{\Lcal}(t) &\defin& \frac{5(1+\log 2)}{2} \cdot \left(t^{2}  \sigma_{\alpha}^{2} +m  \sigma^{2}+ \lambda k  \sigma_{\alpha}\right)\\
A_{r}(t) &\defin& \frac{5(1+\log 2)}{2} \cdot \left(t^{2}  \sigma_{\alpha}^{2} +2m  \sigma^{2}+ 2\lambda k  \sigma_{\alpha}\right).
\end{eqnarray*}
\end{lemma}
\begin{proof}
Using Lemma~\ref{lem:orthoprojcomplement} we have, for $\Db = \Db(\Wb,\vb,t)$, uniformly over $\Wb,\vb$:
\begin{eqnarray*}
\Lcal_{\xb}(\Db,\alphabo)
&=&
\frac{1}{2} \|\xb-\Db \alphabo\|_{2}^{2} + \lambda \|\alphabo\|_{1}
\leq
\frac{1}{2}  \|\xb-\Db \alphabo\|_{2}^{2} + \lambda \sqrt{k} \|\alphabo\|_{2}\\
\Lcal_{\xb}(\Db,\alphabo) + \Lcal_{\xb}(\Dbo,\alphabo)
&\leq&
\frac{1}{2}  \|\xb-\Db \alphabo\|_{2}^{2} + \frac{1}{2} \|\varepsilonb\|_{2}^{2} + 2\lambda \sqrt{k} \|\alphabo\|_{2}.
\end{eqnarray*}
Fix any $\tau \geq 1$ and define $\tau' = (1+\log 2)\tau \geq \tau + \log2$. If $\|\xb-\Db \alphabo\|_{2}^{2} \leq 5(t^{2}\sigma_{\alpha}^{2}+m\sigma^{2})\tau'$ and $\|\alphabo\|_{2}^{2} \leq 5k\sigma_{\alpha}^{2}\tau'$, then $\|\alphabo\|_{2} \leq \sqrt{5k}\sigma_{\alpha}\sqrt{\tau'} \leq \sqrt{5k}\sigma_{\alpha}\tau'$, hence we have
\[
y \leq 
\left(\frac{1}{2}
\left(
5 t^{2}  \sigma_{\alpha}^{2} + 5m  \sigma^{2}
\right)+ \lambda \sqrt{k}\sqrt{5k\sigma_{\alpha}^{2}}\right)\tau'
\leq 
\frac{5}{2}
\left(
t^{2}  \sigma_{\alpha}^{2} + m  \sigma^{2}
+ \lambda k\sigma_{\alpha}\right)\tau'
 = A_{\Lcal}(t) \cdot \tau.
\]
Lemma~\ref{lem:onesided_quadratictail} and a union bound yield
\(
\Pr(y\geq A_{\Lcal}(t) \cdot \tau) \leq
2 \exp(-\tau') \leq \exp(-\tau).
\)
The proof for $y'$ is similar.
\end{proof}
\begin{lemma}\label{lem:truncated_moments}
Let $y$ be a random variable satisfying for any $\tau \geq 1$
\begin{equation}
\label{eq:DefExponentialDecay}
\Pr\left( |y| > A \tau \right) \leq \exp(-\tau).
\end{equation}
for some positive constant $A>0$.
Consider an event $\Ecal$ defined on the same probability space as that of $y$.
For any $u \geq 1$, any integer $q \geq 1$, and $0<p\leq 1$, we have
\begin{eqnarray}
\label{eq:qExpectation}
\Exp\Big[ \indicator{\Ecal} |y|^{pq} \Big]
&\leq& q! \Big[ A^p u \Big]^q  
\Big[\Pr(\Ecal)+\exp(3-u) \Big]\\
\label{eq:qMomentDeviation}
\Exp\Big[ \Big| \indicator{\Ecal} |y|^p  - \Exp\big[\indicator{\Ecal} |y|^p\big]  \Big|^q  \Big]
&\leq& q! \Big[ 2A^p u \Big]^q  
\Big[\Pr(\Ecal)+\exp(3-u) \Big].
\end{eqnarray}
\end{lemma}
\begin{proof}
To begin with, let us notice that by invoking twice the triangle inequality, we have 
\[
\left(\Exp\Big\{ \Big| \indicator{\Ecal} |y|^p  - \Exp\big\{\indicator{\Ecal} |y|^p\big\}  \Big|^q  \Big\}\right)^{1/q}
\leq 
\left(\Exp\big\{\indicator{\Ecal} |y|^{pq}\big\}\right)^{1/q} 
+ 
\left(\Exp\big\{ (\Exp\big\{\indicator{\Ecal} |y|^p\big\})^q \big\}\right)^{1/q}, 
\]
so that by using Jensen's inequality, we obtain
\[
\Exp\Big\{ \Big| \indicator{\Ecal} |y|^p  - \Exp\big[\indicator{\Ecal} |y|^p\big]  \Big|^q  \Big\}
\leq
2^q \Exp\big[\indicator{\Ecal} |y|^{pq}\big],
\]
thus proving~\eqref{eq:qMomentDeviation} provided that~\eqref{eq:qExpectation} holds. We now focus on these raw moments.
Let fix some $u \geq 1$.
We introduce the event
\[
\mathcal{K}\defin
\Big\{ \omega;\ \frac{|y(\omega)|}{A} \leq  u \Big\},
\]
and define $l_u$ as the largest integer such that $u \in [l_u,l_u+1)$.
We can then ``discretize'' the event $\mathcal{K}^c$ as
\[
\mathcal{K}^c \subseteq \bigcup_{l=l_u}^\infty \mathcal{K}^c_l,\quad \text{with}\ 
\mathcal{K}^c_l = 
\Big\{ \omega;\ 
\frac{|y(\omega)|}{A} \in [l,l+1)\Big\}.
\]
We have
\begin{eqnarray*}
 \Exp\{ \indicator{\Ecal} |y|^{pq}\} &=& 
\Exp\{ \indicator{\Ecal \cap \Kcal} |y|^{pq}\} + \Exp\{ \indicator{\Ecal \cap \Kcal^c} |y|^{pq}\}
\leq  \big(Au\big)^{pq} \cdot \Pr(\Ecal)+ 
\sum_{l=l_u}^\infty\Exp\{ \indicator{\Ecal \cap \Kcal^c_l} |y|^{pq}\}\\
&\leq&
A^{pq}
\cdot \Big[  u^{pq} \cdot \Pr(\Ecal) + 
\sum_{l=l_u}^\infty (l+1)^{pq} \cdot \Exp\{ \indicator{\Ecal \cap \mathcal{K}^c_l}\} \Big]\\
&\leq& 
A^{pq}
\cdot \Big[u^{q} \cdot \Pr(\Ecal) + 
\sum_{l=l_u}^\infty (l+1)^{pq} \cdot \Exp[ 
\indicator{\{ \omega;\
|y(\omega)| \geq A l \}} ] \Big]
\end{eqnarray*}
where in the last line we used $u^{p} \leq u$ since $u\geq 1$ and $p \leq 1$. Using the hypothesis~\eqref{eq:DefExponentialDecay}, we continue
\[
 \Exp\{ \indicator{\Ecal} |y|^{pq}\} \leq 
A^{pq} 
\cdot \Big[ u^{q} \cdot \Pr(\Ecal) +
\sum_{l=l_u}^\infty (l+1)^{pq} \exp(-l) \Big].
\]
Upper bounding the discrete sum by a continuous integral, 
we recognize here the incomplete Gamma function~\citep{Gautschi1998},
\begin{eqnarray*}
\sum_{l=l_u}^\infty (l+1)^{pq} e^{-l}
&=&
\sum_{l=l_u}^\infty \int_{l}^{l+1}(l+1)^{pq} e^{-l} dt
\leq
\sum_{l=l_u}^\infty \int_{l}^{l+1}(t+1)^{pq} e^{-(t+1)+t+1-l} dt\\
&\leq&
e^{2}
\sum_{l=l_u}^\infty \int_{l}^{l+1}(t+1)^{pq} e^{-(t+1)} dt
=
e^{2}
\int_{l_{u}}^{\infty}(t+1)^{pq} e^{-(t+1)} dt\\
&= &
e^{2}
\int_{l_{u}+1}^{\infty}t^{pq} e^{-t} dt
\leq 
e^{2}
\int_{u}^{\infty}t^{q} e^{-t} dt
=e^{2} \Gamma\left(q+1,u\right)
\end{eqnarray*}
where again we used $t^{pq}\leq t^{q}$ for $t \geq 1$. A standard formula~\citep[see equation (1.3) in][]{Gautschi1998} leads to, for $u \geq 1$, 
\[
\Gamma(q+1,u)=q! \exp(-u) \sum_{j=0}^q \frac{u^j}{j!} \leq e\, q! \exp(-u) u^q. 
\]
Putting all the pieces together we thus reach the advertised conclusion.
\end{proof}
\begin{corollary}\label{cor:concentration_subgaussian}
Consider $n$ independent draws $\{ y^i \}_{ i\in\SET{n} }$ satisfying the hypothesis~\eqref{eq:DefExponentialDecay}. Consider also $n$ independent events $\{ \Ecal^i \}_{ i\in\SET{n} }$ defined on the same probability space, with $\max_{i\in\SET{n}}\Pr(\Ecal^i) \leq \kappa \leq 1$.
Then, for any $0<p\leq1$ and $0\leq \tau \leq \sqrt{n\kappa}$,  we have
\begin{eqnarray}
\label{eq:concentration_subgaussian_expectation}
\Exp\{ \indicator{ \Ecal^i }|y^{i}|^p\}
\leq 
2A^{p} \cdot (3-\log \kappa) \cdot \kappa&&\\
\label{eq:concentration_subgaussian_deviation}
\Pr\left(\Big|\frac{1}{n} \sum_{i=1}^{n} \left( \indicator{ \Ecal^i }|y^{i}|^{p} -\Exp\{ \indicator{ \Ecal^i }|y^{i}|^p\}\right)\Big|
\geq 
8 A^{p} \cdot (3-\log \kappa) \cdot \sqrt{\kappa} \cdot \frac{\tau}{\sqrt{n}}
\right) 
& \leq & \exp(-\tau^{2})
\end{eqnarray}
\end{corollary}
\begin{proof}
Applying Lemma~\ref{lem:truncated_moments}-Equation~\eqref{eq:qExpectation} with $u=3-\log \kappa$ for $q=1$ we obtain~\eqref{eq:concentration_subgaussian_expectation} where we used that $\Pr(\Ecal^{i})+e^{3-u} \leq \kappa + e^{3-u} = 2 \kappa$.
Similarly, applying Lemma~\ref{lem:truncated_moments}-Equation~\eqref{eq:qMomentDeviation} for $q \geq 2$, we can apply Lemma~\ref{lem:bernstein} with $z_{i} = \indicator{ \Ecal^i }|y^{i}|^{p} -\Exp\{ \indicator{ \Ecal^i }|y^{i}|^p\}$, $M = 2A^p u$ and $\varsigma = \sqrt{2}M\sqrt{\kappa+e^{3-u}} = 2M\sqrt{\kappa}= 4A^p(3-\log \kappa)\sqrt{\kappa}$. This shows that for $0 \leq \tau \leq \frac{\sqrt{n}}{2}\frac{\varsigma}{M} = \sqrt{n\kappa}$ we have~\eqref{eq:concentration_subgaussian_deviation}.
\end{proof}

\end{document}